\documentclass[11pt]{article} 
\usepackage[margin=1in]{geometry}
\usepackage{authblk}
\usepackage[hidelinks]{hyperref}
\usepackage{microtype}
\usepackage[american]{babel}
\usepackage{needspace}

\usepackage{natbib} 
\usepackage{mathtools} 
\usepackage{booktabs} 
\usepackage{tikz} 

\usepackage{tabularx}
\usepackage{array}

\usepackage{amsfonts}
\usepackage{amsthm}
\usepackage{multirow}
\usepackage[ruled,vlined]{algorithm2e}

\usepackage{graphicx}
\usepackage{subcaption}
\usepackage{placeins} 
\newtheorem{theorem}{Theorem}
\newtheorem{lemma}[theorem]{Lemma}
\newtheorem{proposition}[theorem]{Proposition}
\newtheorem{corollary}[theorem]{Corollary}
\newtheorem{assumption}{Assumption}
\theoremstyle{remark}

\newcommand{\E}{\mathbb{E}} 

\title{Length-Aware Adversarial Training for Variable-Length Trajectories: Digital Twins for Mall Shopper Paths}

\date{}

\author[1]{\textbf{He~Sun}\thanks{Corresponding author. Email: \href{mailto:he.sun@yale.edu}{he.sun@yale.edu}.}}

\author[1]{Jiwoong~Shin}
\author[1]{Ravi~Dhar}

\affil[1]{%
  Yale University\\
  New Haven, Connecticut, USA
}

\begin{document}\maketitle

\begin{abstract}
We study generative modeling of \emph{variable-length trajectories}---sequences of visited locations/items with associated timestamps---for downstream simulation and counterfactual analysis.
A recurring practical issue is that standard mini-batch training can be unstable when trajectory lengths are highly heterogeneous, which in turn degrades \emph{distribution matching} for trajectory-derived statistics.
We propose \textbf{length-aware sampling (LAS)}, a simple batching strategy that groups trajectories by length and samples batches from a single length bucket, reducing within-batch length heterogeneity (and making updates more consistent) without changing the model class.
We integrate LAS into a conditional trajectory GAN with auxiliary time-alignment losses and provide (i) a distribution-level guarantee for derived variables under mild boundedness assumptions, and (ii) an IPM/Wasserstein mechanism explaining why LAS improves distribution matching by removing length-only shortcut critics and targeting within-bucket discrepancies.
Empirically, LAS consistently improves matching of derived-variable distributions on a multi-mall dataset of shopper trajectories and on diverse public sequence datasets (GPS, education, e-commerce, and movies), outperforming random sampling across dataset-specific metrics.
\end{abstract}

\section{Introduction}
\label{sec:intro}
Learning realistic \emph{trajectory and sequence models}---and increasingly, \emph{trajectory generators} for simulation and counterfactual analysis---is important in domains such as mobility analytics~\citep{gonzalez2008mobility,feng2018deepmove,mohamed2020socialstgcnn}, recommender systems~\citep{kang2018sasrec,sun2019bert4rec,tagliabue2020multiverse}, and sequential decision logs in education~\citep{piech2015deep}.
A key difficulty shared across these settings is \textbf{variable trajectory length}: real sequences can range from a few steps to hundreds, and length is often strongly correlated with other characteristics (e.g., dwell time, inter-event timing, or item/category diversity).

In practice, we train deep generative models with stochastic mini-batches.
When trajectory lengths are highly heterogeneous, mini-batches mix very short and very long sequences, encouraging the discriminator/critic to exploit length-correlated signals rather than within-length behavioral structure.
This is especially damaging when the goal is \emph{distribution matching} for \emph{trajectory-derived variables}---statistics computed from an entire sequence (e.g., total duration, average per-step time, transition structure, or entropy-like measures).
As a result, the adversarial objective may improve while important derived-variable distributions remain mismatched, limiting fidelity for downstream simulation.

We address this with a \textbf{length-aware sampling (LAS)} scheme that (i) partitions trajectories into length buckets and (ii) draws each mini-batch from a single bucket.
LAS is a training-time intervention (no model changes) that controls within-batch length heterogeneity and makes discriminator/generator updates more consistent in practice.
We combine LAS with a conditional trajectory GAN and auxiliary time-alignment losses to build \emph{digital twins} for trajectory data---generators that can be conditioned on scenario variables to support counterfactual simulation.

\paragraph{Mall digital twin as a motivating case study.}
Shopping malls remain among the most data-rich yet under-optimized physical marketplaces~\citep{eppli1994mall,brueckner1993mall,seiler2017search}.
We study a proprietary dataset of anonymized foot-traffic trajectories collected from \emph{four} large malls, enabling counterfactual questions such as: How would closing an anchor store, changing the tenant mix, or re-routing flows affect dwell time and the distribution of visits?
While the mall application motivates the paper, our method and evaluation are \emph{domain-agnostic} and are validated on additional public sequence datasets.

\paragraph{Contributions.}
\begin{itemize}
\item We formalize trajectory generation with \emph{derived-variable distribution matching} as an evaluation target.
\item We propose \textbf{length-aware sampling (LAS)}, a simple length-bucket batching strategy, and show how to integrate it into GAN training.
\item We provide theory: (i) a Wasserstein bound for derived-variable distributions under boundedness and controlled training losses, and (ii) an IPM/Wasserstein mechanism explaining why LAS improves distribution matching by removing length-only shortcut critics and targeting within-bucket discrepancies.
\item We demonstrate empirical gains of LAS over random sampling on a multi-mall dataset and multiple public sequence datasets.
\end{itemize}

\section{Related Work}
\label{sec:related}
Our work connects to (i) modeling and generating sequential/trajectory data, (ii) digital twins and counterfactual simulation, and (iii) stabilizing adversarial/stochastic training under heterogeneous data.

\paragraph{Trajectory and sequence modeling.}
Trajectory data are central in mobility analytics~\citep{gonzalez2008mobility,feng2018deepmove,mohamed2020socialstgcnn}.
Beyond mobility, generative sequence modeling has been explored in settings such as pedestrian motion~\citep{gupta2018social} and in general-purpose sequence generators, including GAN-style methods for discrete sequences~\citep{yu2017seqgan} and synthetic time-series generation~\citep{yoon2019timegan}.
In recommender systems, sequential models are widely used to represent and generate user--item trajectories (e.g., recurrent or attention-based models)~\citep{hidasi2015gru4rec,kang2018sasrec,sun2019bert4rec,wu2018srgnn,tagliabue2020multiverse}.
Our focus differs: we optimize and evaluate \emph{distribution matching} of \emph{trajectory-derived statistics} and study how batching by length shapes this objective.

\paragraph{Digital twins and counterfactual simulation.}
Digital twins aim to create forward simulators for complex systems~\citep{grieves2016digitaltwin,fuller2020digitaltwin,kritzinger2018digitaltwin,attaran2023digital}.
In many operational settings (including retail), counterfactual analysis is often addressed with observational causal methods that are inherently backward-looking~\citep{athey2017beyond}.
We contribute a complementary generative angle: a learned simulator calibrated on observed trajectories that can be conditioned on scenario variables to support ``what-if'' analyses.

\paragraph{Mall retail analytics and shopper trajectories.}
Marketing and operations research have studied mall design, tenant mix, and shopper flows, traditionally using aggregate footfall, surveys, and structural models~\citep{eppli1994mall,brueckner1993mall}.
More recent work leverages fine-grained in-store/indoor mobility traces and path data to study store transitions, dwell-time distributions, and consumer search behavior~\citep{seiler2017search}.
Our setting aligns with this line of work but focuses on learning a \emph{generative} simulator whose distribution matches derived-variable statistics and supports counterfactual scenario testing.

\paragraph{Stability under heterogeneous mini-batches and adversarial training.}
Stochastic optimization and stability in non-convex settings have been widely studied~\citep{ghadimi2013sgd,bottou2018sgd}, and curriculum/ordering strategies are a classic tool for handling heterogeneous difficulty/structure~\citep{bengio2009curriculum}.
GAN training introduces additional instability due to the adversarial objective~\citep{arjovsky2017gan,mescheder2018gan_convergence}, and prior work proposes stabilization strategies such as Wasserstein/gradient-penalty critics~\citep{gulrajani2017improved}.
LAS is complementary to these lines: rather than changing the objective or architecture, it controls mini-batch composition to reduce length-only shortcuts and focus learning on within-length discrepancies that matter for distribution matching.

\paragraph{Positioning.}
Prior work has typically examined mall-level analytics, spatiotemporal modeling, or adversarial training stability in isolation.
Our contribution is to unify these strands within a single framework: we instantiate a mall digital-twin setting, introduce length-aware sampling as a simple training intervention, and provide theory and empirical evidence linking LAS to improved matching of length-dependent derived variables.

\section{Problem Setup}
\label{sec:problem}
We consider conditional generation of variable-length trajectories.
A trajectory is a sequence
\[
x=\{(j_t,\tau_t^{(\mathrm{intra})},\tau_t^{(\mathrm{inter})})\}_{t=1}^{T},
\]
where $j_t$ is a discrete location/item identifier, $\tau_t^{(\mathrm{intra})}$ is the time spent at step $t$, and $\tau_t^{(\mathrm{inter})}$ is the transition time to the next step.\footnote{For non-mall datasets, $(j_t,\tau_t)$ may represent different event attributes; the framework only requires variable-length sequences with optional continuous covariates.}
The length $T$ varies across trajectories.

\paragraph{Conditional generation.}
Each trajectory is associated with observed context $c$ (e.g., entry time, user segment, scenario variables).
Let $p_{\mathrm{data}}(x\mid c)$ denote the true conditional distribution and $p_G(x\mid c)$ the generator distribution.
Our goal is to learn $p_G$ so that generated trajectories match the real distribution both at the sequence level and in terms of \emph{trajectory-derived variables}.

\paragraph{Derived variables and evaluation.}
Let $f:\mathcal{X}\to\mathbb{R}$ be a scalar \emph{derived variable} computed from a full trajectory (e.g., total duration, average dwell time, number of visits, entropy of categories, or dataset-specific statistics).
Let $P_f$ and $Q_f$ be the distributions of $f(x)$ when $x\sim p_{\mathrm{data}}(\cdot\mid c)$ and $x\sim p_G(\cdot\mid c)$, respectively (marginalizing over $c$ when appropriate).
We quantify distribution mismatch using distances such as Wasserstein-1 for continuous variables and KL/JS divergence after discretization for discrete/histogram variables.
In the mall domain, we report a broad set of derived variables capturing dwell, transitions, and visit patterns; in the other domains we use a compact set of dataset-specific derived variables.

\section{Method}
\label{sec:method}

\subsection{Conditional trajectory GAN}
We instantiate $p_G(x\mid c)$ with a conditional generator $G_\theta$ and discriminator (critic) $D_\phi$.
We summarize the main architectural components below and provide full details in Appendix~\ref{app:arch_full}.

\paragraph{Architecture summary.}
We use a three-stage design: (1) store-feature embedding with attention-based neighborhood fusion, (2) an LSTM-based conditional generator that outputs the next store and timing heads, and (3) a bidirectional LSTM discriminator/critic over the full sequence.

\paragraph{Store and context encoding.}
We represent each mall as a graph $G=(V,E)$ with stores as nodes and spatial adjacencies as edges.
Each store $v_i$ is described by a feature vector $\mathbf{x}_i$ (identity, floor, category, traffic/open features, and neighborhood statistics; see Appendix~\ref{app:arch_full}).
A learned encoder maps $\mathbf{x}_i$ to an embedding $\mathbf{e}_i\in\mathbb{R}^{d_e}$ and fuses neighbor information via attention,
\[
\tilde{\mathbf{e}}_i 
= \mathbf{e}_i + \sum_{j\in\mathcal{N}(i)} \alpha_{ij}\,\mathbf{W}\mathbf{e}_j,
\qquad
\alpha_{ij}=\mathrm{softmax}_{j}\big(\mathbf{q}_i^\top \mathbf{k}_j\big),
\]
yielding a context-aware store representation $\tilde{\mathbf{e}}_i$.
Mall-level day context $c$ (calendar/campaign/weather indicators) is embedded and concatenated to the generator inputs at every step.

\paragraph{Generator and discriminator heads.}
At step $t$, the generator conditions on the previous hidden state, the previous visited store embedding, and the context $c$ to produce (i) a categorical distribution over the next store (implemented with a Gumbel-Softmax relaxation for differentiability), and (ii) nonnegative intra- and inter-store times using separate regression heads.
The discriminator processes the full sequence with a bidirectional LSTM and outputs a sequence-level realism score.

\subsection{Training objective}
We use a non-saturating GAN objective:
\[
\mathcal{L}_D(\phi)
=
-\mathbb{E}_{x\sim p_{\mathrm{data}}}\big[\log D_\phi(x)\big]
-\mathbb{E}_{\hat{x}\sim p_G}\big[\log(1-D_\phi(\hat{x}))\big],
\]
\[
\mathcal{L}_{\mathrm{adv}}(\theta)
=
-\mathbb{E}_{\hat{x}\sim p_G}\big[\log D_\phi(\hat{x})\big].
\]
To better align timing statistics, we add auxiliary time losses (detailed in Appendix~\ref{app:loss_full}):
\[
\mathcal{L}_G(\theta)
=
\mathcal{L}_{\mathrm{adv}}(\theta)
+
\lambda_{\mathrm{time}}\Big(\mathcal{L}_{\mathrm{intra}}+\mathcal{L}_{\mathrm{inter}}\Big),
\]
where, for a real trajectory of length $T$ and a generated trajectory of length $\hat{T}$,
\[
\mathcal{L}_{\mathrm{intra}}
=
\frac{1}{\min(T,\hat{T})}\sum_{t=1}^{\min(T,\hat{T})}
\left|\hat{\tau}_t^{(\mathrm{intra})}-\tau_t^{(\mathrm{intra})}\right|,
\]
\[
\mathcal{L}_{\mathrm{inter}}
=
\frac{1}{\min(T,\hat{T})}\sum_{t=1}^{\min(T,\hat{T})}
\left|\hat{\tau}_t^{(\mathrm{inter})}-\tau_t^{(\mathrm{inter})}\right|.
\]
We alternate gradient updates for $\phi$ and $\theta$ (Appendix~\ref{app:training_full}).

\paragraph{Dataset-specific objectives.}
In the mall domain, we train with the adversarial loss together with the auxiliary intra/inter time alignment terms above.
For the public sequential datasets, we do not use the mall-specific time losses and instead use a dataset-appropriate adversarial objective:
\textbf{Education} and \textbf{GPS} use the standard adversarial loss (treating each example as a sequence and relying on the discriminator to learn timing/structure implicitly);
\textbf{Movie} uses the adversarial loss augmented with a feature matching regularizer (\texttt{feature\_matching\_loss});
and \textbf{Amazon} uses a Wasserstein (WGAN-style) objective for improved training stability.
Full loss definitions are in Appendix~\ref{app:loss_full}.

\paragraph{Training procedure and complexity.}
Each iteration samples a minibatch of real trajectories using RS or LAS (Section~\ref{subsec:las}), generates a matched minibatch from $G_\theta$, and performs alternating updates of $D_\phi$ and $G_\theta$.
The dominant cost is the forward/backward pass over $B$ sequences of length at most $T_{\max}$, i.e., $O(BT_{\max})$ per update up to architecture-dependent constants; LAS adds only a small bookkeeping overhead for bucket sampling.

\subsection{Length-aware sampling (LAS)}
\label{subsec:las}
Let $\ell(x)=T$ denote trajectory length.
We partition the training set into $K$ length buckets $\{\mathcal{D}_k\}_{k=1}^K$ using length quantiles.
LAS draws each mini-batch from a \emph{single} bucket:
first sample a bucket index $K_s\sim w$ (with weights $w_k$), then sample all $m$ examples uniformly from $\mathcal{D}_{K_s}$.
In our experiments, we use the empirical bucket mixture $w_k \propto p_k$, where $p_k:=|\mathcal{D}_k|/|\mathcal{D}|$ is the empirical bucket mass.\footnote{Uniform bucket sampling is a straightforward alternative; we do not vary this choice in our experiments.}
This removes within-batch length heterogeneity and can make discriminator/generator updates more consistent for length-correlated objectives, while still exposing the model to all lengths over training.

\begin{algorithm}[t]
\caption{Length-aware sampling (LAS) for variable-length trajectories.}
\label{alg:las}
\KwIn{Dataset $\mathcal{D}$; length buckets $\{\mathcal{D}_k\}_{k=1}^K$; bucket weights $w$; batch size $m$.}
\KwOut{Mini-batch $\mathcal{B}$ of size $m$.}
Sample bucket index $k \sim \text{Categorical}(w_1,\dots,w_K)$\;
Sample $x_1,\dots,x_m \sim \text{Unif}(\mathcal{D}_k)$\;
\Return $\mathcal{B}=\{x_i\}_{i=1}^m$\;
\end{algorithm}

\section{Theory}
\label{sec:theory}
We state two types of results: (i) \emph{distribution-level} bounds for derived variables, and (ii) \emph{optimization-level} an IPM/Wasserstein mechanism explaining why LAS improves distribution matching by removing length-only shortcut critics and targeting within-bucket discrepancies.

\subsection{Assumptions}
\begin{assumption}[Boundedness and controlled training losses]
\label{ass:bounds}
(i) Trajectory length is bounded: $T\le T_{\max}$ almost surely.\\
(ii) Per-step time contributions are bounded: for all $t$, $0\le \tau_t^{(\mathrm{intra})}+\tau_t^{(\mathrm{inter})}\le B$.\\
(iii) After training, the sequence-level divergence and auxiliary losses are controlled:
\[
\begin{aligned}
\mathrm{JS}(p_{\mathrm{data}}\Vert p_G) &\le \delta,\\
\mathcal{L}_{\mathrm{intra}} &\le \epsilon_{\mathrm{intra}},\\
\mathcal{L}_{\mathrm{inter}} &\le \epsilon_{\mathrm{inter}}.
\end{aligned}
\]
\end{assumption}
Let $C_{\mathrm{JS}}$ denote a universal constant such that $\mathrm{TV}(P,Q)\le C_{\mathrm{JS}}\sqrt{\mathrm{JS}(P\Vert Q)}$.

\subsection{Derived-variable distribution bounds}
For derived variables we use in the mall domain (Appendix~\ref{app:theory_full}),
\[
\begin{aligned}
\mathrm{Tot}(x)
&=\sum_{t=1}^{T}\tau_t^{(\mathrm{intra})}
  +\sum_{t=1}^{T-1}\tau_t^{(\mathrm{inter})},\\
\mathrm{Avg}(x)
&=\frac{1}{T}\sum_{t=1}^{T}\tau_t^{(\mathrm{intra})},\\
\mathrm{Vis}(x)
&=T.
\end{aligned}
\]
and more generally for any scalar $f(x)$ that is Lipschitz under an appropriate trajectory semi-metric (Appendix~\ref{app:theory_full}).
Let $P_f$ and $Q_f$ be the distributions of $f(x)$ under $p_{\mathrm{data}}$ and $p_G$.

We measure distributional closeness via the 1-Wasserstein distance (Kantorovich--Rubinstein duality):
\[
W_1(P_f, Q_f)
=
\sup_{\|g\|_{\mathrm{Lip}}\le 1}
\left|
\E_{x\sim p_{\mathrm{data}}}\!\big[g(f(x))\big]
-
\E_{\hat{x}\sim p_G}\!\big[g(f(\hat{x}))\big]
\right|.
\]

\begin{theorem}[Distributional closeness for derived variables]
\label{thm:w1_main}
Under Assumption~\ref{ass:bounds}, for each 
$f \in \{\mathrm{Tot}, \mathrm{Avg}, \mathrm{Vis}\}$,
\[
W_1(P_f, Q_f) \;\le\;
\begin{cases}
\begin{aligned}
& T_{\max}\big(\epsilon_{\mathrm{intra}}+\epsilon_{\mathrm{inter}}\big)\\
&\quad + B\,T_{\max}\,C_{\mathrm{JS}}\sqrt{\delta},
\end{aligned} & f=\mathrm{Tot},\\[0.35em]
\begin{aligned}
& \epsilon_{\mathrm{intra}}\\
&\quad + B\,T_{\max}\,C_{\mathrm{JS}}\sqrt{\delta},
\end{aligned} & f=\mathrm{Avg},\\[0.35em]
2T_{\max}\,\mathrm{TV}\!\big(p_{\mathrm{data}}(T),p_G(T)\big), & f=\mathrm{Vis}.
\end{cases}
\]
\end{theorem}
\paragraph{Proof sketch.}
For any 1-Lipschitz test function $g$, the gap
$\big|\E[g(f(x))]-\E[g(f(\hat{x}))]\big|$ decomposes into (i) mismatch between real and generated sequences, controlled by the sequence-level divergence via $\mathrm{TV}\!\le C_{\mathrm{JS}}\sqrt{\mathrm{JS}}$, and (ii) per-step timing mismatch, controlled by $\mathcal{L}_{\mathrm{intra}}$ and $\mathcal{L}_{\mathrm{inter}}$.
For $\mathrm{Tot}$, summing per-step errors yields the $T_{\max}(\epsilon_{\mathrm{intra}}+\epsilon_{\mathrm{inter}})$ term; for $\mathrm{Avg}$, normalization removes the factor $T_{\max}$ for the intra contribution; and for $\mathrm{Vis}$, the derived variable depends only on the length marginal, giving a bound in terms of $\mathrm{TV}(p_{\mathrm{data}}(T),p_G(T))$.
See Appendix~\ref{app:theory_full} for full statements and proofs.

\paragraph{Additional implications.}
We summarize additional consequences for distribution matching; full proofs are in Appendix~\ref{app:theory_full}.
Theorem~\ref{thm:w1_main} isolates two drivers of derived-variable mismatch: within-sequence timing errors and mismatch in the length marginal.
The results below make precise why length-aware batching targets these terms and reduces shortcut signals for the discriminator.

\begin{corollary}[From $W_1$ control to CDF control (informal)]
For any derived variable $f$ supported on a bounded interval, small $W_1(P_f,Q_f)$ implies a small Kolmogorov--Smirnov distance between the corresponding CDFs, up to constants depending on the support radius.
\end{corollary}

\begin{lemma}[Bucket-only (length-only) critics are a null space]
Let $K(x)$ denote the length bucket index used by LAS. Any critic that depends only on $K(x)$ has identical expectation under the data and generator when evaluated \emph{within a fixed bucket}, and therefore cannot provide an ``easy'' within-batch shortcut signal for the discriminator under LAS.
\end{lemma}

\begin{lemma}[Global Wasserstein dominated by length mismatch + within-bucket discrepancy]
Let $p_{\mathrm{data}}=\sum_k w_k p_k$ and $p_G=\sum_k \hat{w}_k q_k$ be mixtures over length buckets.
Then the global $W_1$ distance decomposes into (i) a term proportional to $\mathrm{TV}(w,\hat w)$ capturing length-marginal mismatch and (ii) a weighted sum of within-bucket discrepancies $W_1(p_k,q_k)$.
\end{lemma}

\noindent Proofs are provided in Appendix~\ref{app:theory_full}.

\subsection{Why LAS improves distribution matching in practice}
\label{subsec:theory_las}
LAS changes mini-batch construction so that each discriminator/generator update is computed on a \emph{single} length regime.
This reduces within-batch length heterogeneity and prevents length from becoming an easy within-batch shortcut feature for the discriminator.
Crucially for our empirical goal (derived-variable \emph{distribution matching}), LAS also controls \emph{exposure} to different lengths via the bucket weights $w$:
over $S$ updates, each length bucket is selected about $w_k S$ times, ensuring all length regimes contribute training signal.
Since many derived variables are length-dependent (Theorem~\ref{thm:w1_main}), improved coverage and length-level supervision translate into better distribution matching in our evaluations.
To make this mechanism explicit, we provide a simple structural statement:
under LAS, any \emph{length-only} (bucket-only) discriminator feature becomes uninformative within an update, forcing the critic to focus on within-bucket structure.
A more detailed IPM/Wasserstein view is given in Appendix~\ref{subsec:las_app}.

\begin{proposition}[LAS removes length-only shortcut critics]
\label{prop:las_projection_main}
Let \(K(x)\in\{1,\dots,K\}\) denote the length bucket of a trajectory \(x\).
For any function \(a:\{1,\dots,K\}\to\mathbb{R}\), define the bucket-only term \(\phi(x):=a(K(x))\).
In a LAS update conditioned on bucket \(k\), we have \(\phi(x)=a(k)\) almost surely under both the real and generated bucket-conditional distributions, and thus
\[
\E\!\big[\phi(X)\mid K(X)=k\big]-\E\!\big[\phi(\hat X)\mid K(\hat X)=k\big] = 0.
\]
Therefore bucket-only (length-only) components lie in a null space of the LAS discriminator objective within a bucket and cannot provide an ``easy'' within-batch shortcut signal.
\end{proposition}

\section{Experiments}
\label{sec:experiments}

We evaluate \textbf{random sampling (RS)} versus \textbf{length-aware sampling (LAS)} in adversarial training for sequential trajectory data.
Across all experiments, we keep the \emph{model architecture} and \emph{optimization hyperparameters} fixed. For each dataset, RS and LAS also share the same training objective; only the mini-batch construction rule changes. Note that the objective can be \emph{dataset-specific} for public benchmarks (e.g., Wasserstein for Amazon for stability); see Appendix~\ref{app:loss_full}. 
Our primary goal is \emph{distributional fidelity} of \textbf{derived variables} (e.g., trajectory length, total time, diversity) that are used downstream for planning, simulation, and analytics.

\subsection{Datasets and derived variables}
\label{sec:datasets}

Table~\ref{tab:dataset_overview} summarizes the evaluation datasets and the derived variables we compare between ground-truth and generated samples.
For the mall datasets, each trajectory is a sequence of store visits with associated \emph{intra-store} and \emph{inter-store} durations; for the public datasets, each user trajectory is a variable-length sequence (e.g., ratings, GPS points, or question attempts), optionally with continuous attributes.

\begin{table}[t]
\caption{Evaluation datasets and derived variables. Mall identifiers are anonymized.}
\label{tab:dataset_overview}
\centering
\small
\setlength{\tabcolsep}{3pt}
\renewcommand{\arraystretch}{0.98}
\begin{tabularx}{\columnwidth}{@{}l l X@{}}
\toprule
Dataset & Domain & Derived variables (distributional evaluation) \\
\midrule
Mall A--D & Indoor mobility &
Total time in mall; number of store visits (trajectory length); avg/total intra-store time; avg/total inter-store time; store-type mix; time spent per category; floor distribution; store diversity. \\
Amazon & E-commerce ratings &
Sequence length; item diversity; mean inter-event days; duration (days); mean rating. \\
Movie & Movie ratings &
Trajectory length; inter-rating time (minutes); mean rating; rating std. \\
Education & Student learning &
Trajectory length (\#questions); mean correctness; std correctness. \\
GPS & GPS mobility &
Trajectory length; total distance (km); average speed (km/h). \\
\bottomrule
\end{tabularx}
\end{table}
\paragraph{Derived variables and evaluation metric.}
For each real or generated trajectory \\
$\pi=\{(j_t,\tau_t^{(\text{intra})},\tau_t^{(\text{inter})})\}_{t=1}^{T}$,
we compute a set of scalar summaries (``derived variables'') and compare their \emph{empirical distributions}
between real and synthetic data on the held-out test set.
Our primary distributional metric is the Kolmogorov--Smirnov (KS) distance:
\[
\mathrm{KS}(P,Q)=\sup_{x}\left|F_{P}(x)-F_{Q}(x)\right|,
\]
where $F_{P}$ and $F_{Q}$ are the empirical CDFs of the derived variable under real and generated trajectories,
respectively (for discrete variables we apply KS to the cumulative mass function under a fixed ordering).

For the mall datasets, we report KS for the following derived variables:
\begin{itemize}
    \item Total intra-store time: $M_{\text{intra}}^{\text{tot}}=\sum_{t=1}^{T}\tau_t^{(\text{intra})}$.
    \item Total inter-store time: $M_{\text{inter}}^{\text{tot}}=\sum_{t=1}^{T}\tau_t^{(\text{inter})}$.
    \item Avg.\ intra-store time: $M_{\text{avg-intra}}=\frac{1}{T}\sum_{t=1}^{T}\tau_t^{(\text{intra})}$.
    \item Avg.\ inter-store time: $M_{\text{avg-inter}}=\frac{1}{\max(T-1,1)}\sum_{t=1}^{T}\tau_t^{(\text{inter})}$.
    \item Total time in mall: $M_{\text{tot}}=M_{\text{intra}}^{\text{tot}}+M_{\text{inter}}^{\text{tot}}$.
    \item Trajectory length (\#visits): $M_{\text{len}}=T$.
    \item Store diversity: $M_{\text{div-store}}=\big|\{j_t\}_{t=1}^{T}\big|$.
\end{itemize}
For category/floor summaries, with $c(j_t)$ the store category and $f(j_t)$ the floor, we form
per-trajectory histograms such as visit counts $N_{c}=\sum_{t=1}^{T}\mathbf{1}[c(j_t)=c]$ and
floor counts $N_{f}=\sum_{t=1}^{T}\mathbf{1}[f(j_t)=f]$, as well as time-by-category
$T^{(\text{intra})}_{c}=\sum_{t=1}^{T}\tau_t^{(\text{intra})}\mathbf{1}[c(j_t)=c]$,
and compare their induced marginals across trajectories.
Analogous trajectory-level summaries are used for the public datasets (Table~\ref{tab:dataset_overview}).

\subsection{Implementation details}
\label{sec:impl_details}
\paragraph{Model configuration (notation $\rightarrow$ value).}
For the mall experiments, dataset-specific constants are set from the data (e.g., number of stores/floors/categories),
while embedding sizes and network widths are shared across experiments.
For one representative mall, we use:
\begin{center}
\small
\setlength{\tabcolsep}{4pt}
\renewcommand{\arraystretch}{0.95}
\begin{tabular}{l l l}
\toprule
Symbol & Description & Value \\
\midrule
$|\mathcal{S}|$ & number of stores & $202$ \\
$F$ & number of floors & $3$ \\
$C$ & number of store categories & $19$ \\
\midrule
$d_e$ & store embedding dimension & $32$ \\
$h$ & LSTM hidden size & $128$ \\
$z$ & latent dimension (generator) & $16$ \\
$d_{\text{type}}$ & store--type embedding dimension & $16$ \\
$d_{\text{floor}}$ & floor embedding dimension & $8$ \\
\bottomrule
\end{tabular}
\end{center}

\paragraph{Training protocol.}
Training follows the procedure described in the algorithmic section, with the same loss notation and objectives:
the adversarial loss for realism and $\ell_1$ losses for time heads (intra/inter) weighted as in the loss section.
We use Adam optimizers ($\beta_1{=}0.5,\ \beta_2{=}0.999$) with learning rate $10^{-4}$ for both generator and discriminator,
batch size $128$, spectral normalization on linear layers, and Gumbel–Softmax sampling for store selection with an annealed temperature
from $1.5$ down to $0.1$.
Training runs for up to $18$ epochs with early stopping (patience $=3$) based on generator loss.

\subsection{Evaluation protocol}
\label{sec:eval_protocol}

For each dataset, we train two models with identical architectures and hyperparameters: one using \textbf{random sampling (RS)} and one using \textbf{length-aware sampling (LAS)}; the only difference is how mini-batches are constructed during training.
We evaluate on held-out test data.
For the mall domain, we split by \emph{unique days} (80\%/20\%) to prevent temporal leakage and generate trajectories under the same day-level context as the test set (Appendix~\ref{app:exp_full}).
For public datasets, we use a held-out split as described in Appendix~\ref{app:exp_full}.

Our goal is \emph{distributional fidelity} of the derived variables in Table~\ref{tab:dataset_overview}.
On the test set, we compare empirical distributions between real and generated samples and report the \textbf{Kolmogorov--Smirnov statistic (KS)} for each derived variable (lower is better).
We also visualize distribution overlays for representative variables; additional diagnostics (e.g., $t$-tests, KL divergence) and full plots are provided in Appendix~\ref{app:exp_full}.

\subsection{Mall digital-twin results}
\label{sec:mall_results}

Table~\ref{tab:mall_4_results} reports KS statistics on six key derived variables across four proprietary mall datasets.
LAS consistently improves the \emph{length-related} and \emph{time-related} marginals (e.g., \#visits and total time), and reduces the overall mean KS across these metrics from \textbf{0.737} (RS) to \textbf{0.253} (LAS), a \textbf{65.7\%} relative reduction.
Figure~\ref{fig:mall_main} visualizes representative distributions on Mall~D.

\begin{table*}[!t]
\caption{Mall datasets: goodness-of-fit for derived variables. We report KS statistics between ground-truth and generated distributions (lower is better).}
\label{tab:mall_4_results}
\centering
\scriptsize
\setlength{\tabcolsep}{3pt}
\renewcommand{\arraystretch}{0.95}
\begin{tabular}{lcccccccc}
\toprule
 & \multicolumn{2}{c}{Mall A} & \multicolumn{2}{c}{Mall B} & \multicolumn{2}{c}{Mall C} & \multicolumn{2}{c}{Mall D} \\
\cmidrule(lr){2-3}\cmidrule(lr){4-5}\cmidrule(lr){6-7}\cmidrule(lr){8-9}
Derived variable & RS & LAS & RS & LAS & RS & LAS & RS & LAS \\
\midrule
Total time in mall & 0.528 & 0.056 & 0.538 & 0.152 & 0.630 & 0.269 & 0.661 & 0.072 \\
Trajectory length / \#visits & 0.955 & 0.047 & 0.947 & 0.048 & 0.953 & 0.048 & 0.951 & 0.044 \\
Avg intra-store time & 0.975 & 0.005 & 0.978 & 0.066 & 0.975 & 0.382 & 0.959 & 0.034 \\
Avg inter-store time & 0.622 & 0.289 & 0.645 & 0.380 & 0.684 & 0.404 & 0.767 & 0.456 \\
Store category mix & 0.278 & 0.333 & 0.506 & 0.287 & 0.467 & 0.333 & 0.477 & 0.303 \\
Floor distribution & 1.000 & 0.667 & 1.000 & 0.333 & 1.000 & 0.667 & 0.200 & 0.400 \\
\midrule
Mean across metrics & 0.726 & 0.233 & 0.769 & 0.211 & 0.785 & 0.350 & 0.669 & 0.218 \\
\bottomrule
\end{tabular}
\end{table*}

\begin{figure}[t]
\centering
\begin{subfigure}{0.48\columnwidth}
\centering
\includegraphics[width=\linewidth]{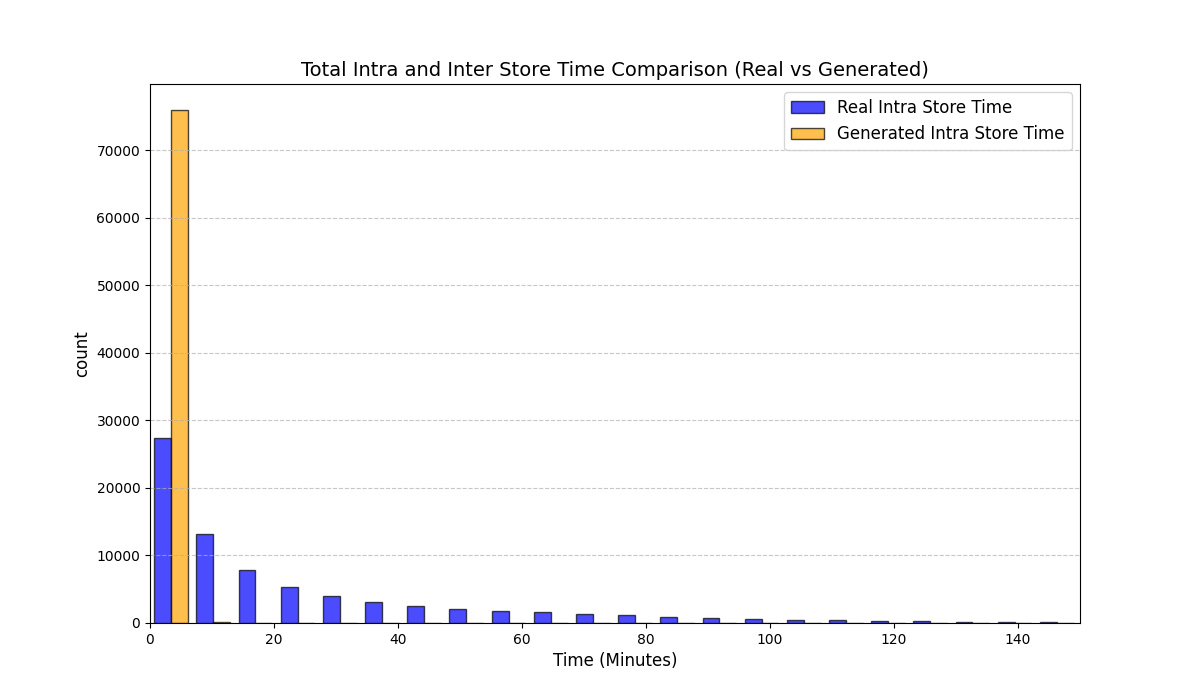}
\caption{RS: total time in mall (Mall D).}
\end{subfigure}\hfill
\begin{subfigure}{0.48\columnwidth}
\centering
\includegraphics[width=\linewidth]{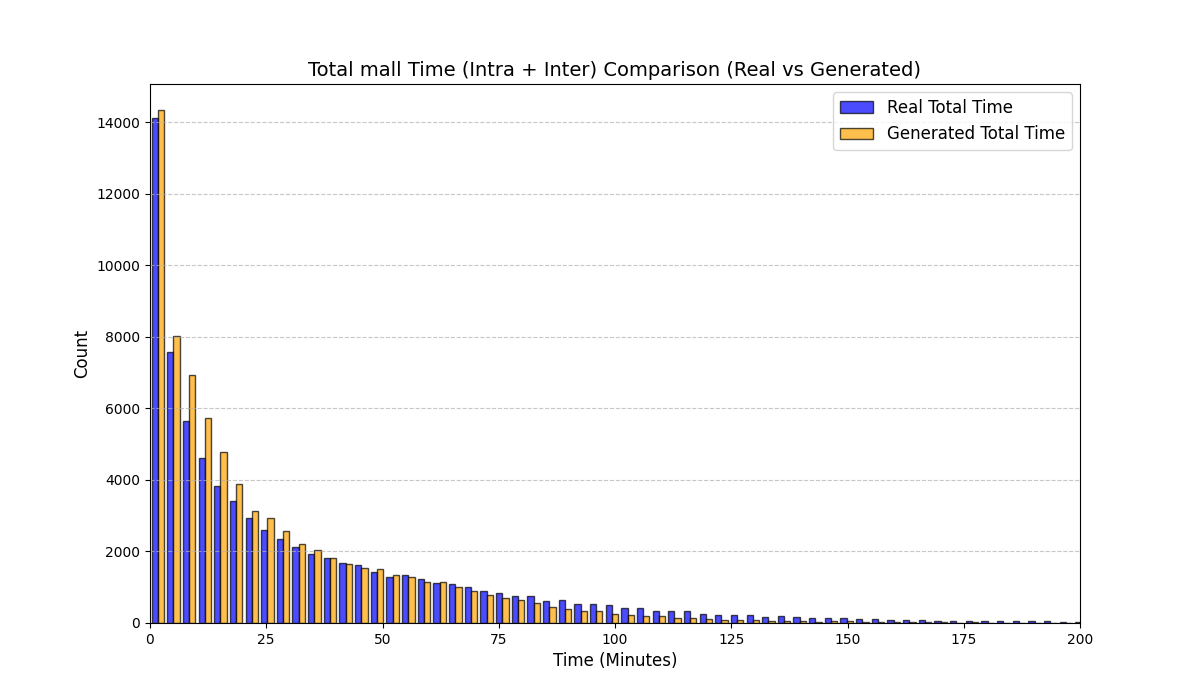}
\caption{LAS: total time in mall (Mall D).}
\end{subfigure}

\vspace{0.2em}

\begin{subfigure}{0.48\columnwidth}
\centering
\includegraphics[width=\linewidth]{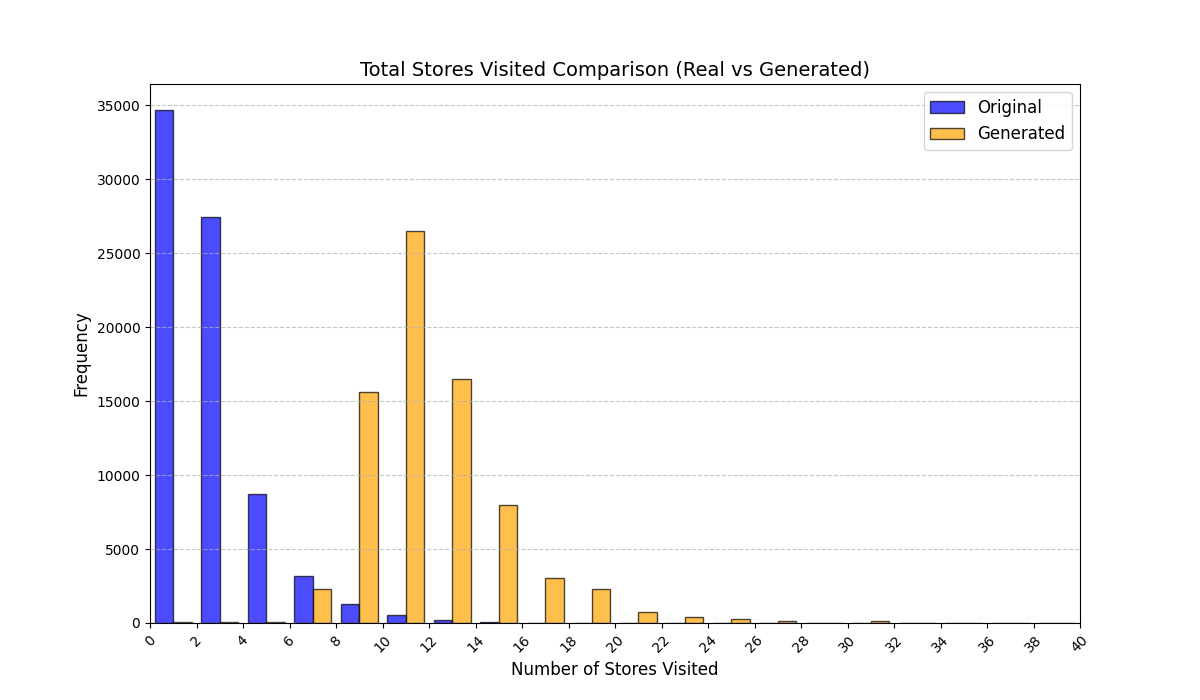}
\caption{RS: \#store visits / trajectory length (Mall D).}
\end{subfigure}\hfill
\begin{subfigure}{0.48\columnwidth}
\centering
\includegraphics[width=\linewidth]{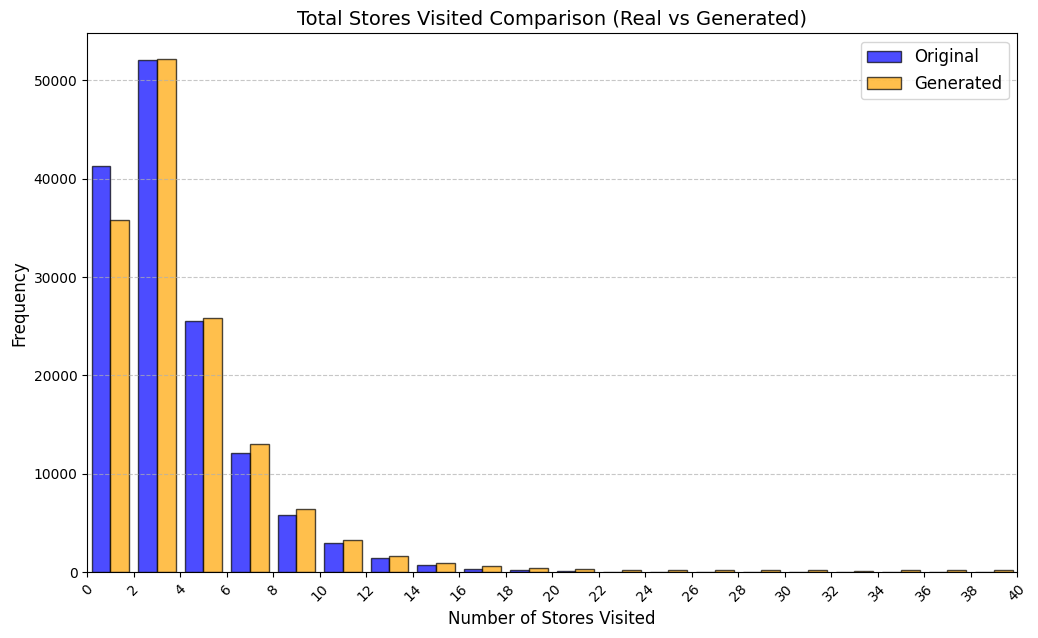}
\caption{LAS: \#store visits / trajectory length (Mall D).}
\end{subfigure}
\caption{Representative mall distributions. LAS improves agreement with the ground-truth marginals without changing the GAN objective.}
\label{fig:mall_main}
\end{figure}

To summarize performance across a broader set of mall-derived variables, Table~\ref{tab:mall_metric_avgs} reports the \emph{average} KS across \textbf{all ten} mall metrics (timing, diversity, and categorical marginals). 
On average, LAS reduces mean KS from \textbf{0.697} (RS) to \textbf{0.247} (LAS), a \textbf{64.5\%} reduction.
Figure~\ref{fig:mall_length_all} further shows trajectory-length distributions across all four malls.

\begin{table}[t]
\caption{Mall datasets: mean KS across malls for each derived variable. The final row is the mean across all metrics.}
\label{tab:mall_metric_avgs}
\centering
\footnotesize
\setlength{\tabcolsep}{3pt}
\renewcommand{\arraystretch}{0.95}
\begin{tabularx}{\columnwidth}{@{}>{\raggedright\arraybackslash}X c c c@{}}
\toprule
Derived variable & \shortstack{Mean KS\\(RS)} & \shortstack{Mean KS\\(LAS)} & \shortstack{Relative\\reduction} \\
\midrule
Avg intra-store time & 0.972 & 0.122 & 87.5\% \\
Number of visits & 0.951 & 0.047 & 95.1\% \\
Floor distribution & 0.800 & 0.517 & 35.4\% \\
Total intra-store time & 0.796 & 0.119 & 85.1\% \\
Total inter-store time & 0.763 & 0.380 & 50.3\% \\
Avg inter-store time & 0.679 & 0.382 & 43.8\% \\
Total time in mall & 0.589 & 0.137 & 76.7\% \\
Store diversity & 0.556 & 0.066 & 88.1\% \\
Store type distribution & 0.432 & 0.314 & 27.3\% \\
Time spent per category & 0.426 & 0.390 & 8.4\% \\
\midrule
All metrics (mean) & 0.697 & 0.247 & 64.5\% \\
\bottomrule
\end{tabularx}
\end{table}

See Appendix~\ref{app:extra_mall_plots} for per-mall trajectory-length (\#visits) distribution plots under RS and LAS.

\subsection{Public sequential datasets}
\label{sec:public_results}

Table~\ref{tab:public_results} reports KS statistics on four public datasets: Amazon (e-commerce ratings), Movie (movie ratings), Education (student learning sequences), and GPS (mobility trajectories).
We observe the largest gains on \emph{duration} and \emph{diversity} related metrics on Amazon, consistent improvements on Movie inter-event timing, and clear reductions on GPS and Education derived-variable mismatches (especially length-related marginals).
Figures~\ref{fig:amazon_main}--\ref{fig:gps_main} visualize representative marginals.

\begin{table}[t]
\caption{Public datasets: KS statistics (lower is better). For each dataset, RS and LAS share the same model and objective; only the batching strategy differs.}
\label{tab:public_results}
\centering
\small
\begin{tabular}{llcc}
\toprule
Dataset & Derived variable & RS & LAS \\
\midrule
Amazon & Sequence length & 0.002 & 0.002 \\
Amazon & Item diversity & 0.338 & 0.020 \\
Amazon & Inter-event days & 0.456 & 0.170 \\
Amazon & Duration (days) & 0.413 & 0.046 \\
Amazon & Mean rating & 0.632 & 0.590 \\
\addlinespace
Movie & Trajectory length & 0.120 & 0.067 \\
Movie & Inter-rating time (min) & 0.466 & 0.294 \\
Movie & Mean rating & 0.155 & 0.106 \\
Movie & Rating std & 0.754 & 0.669 \\
\addlinespace
Education & Trajectory length & 0.411 & 0.164 \\
Education & Mean correctness & 0.9997 & 0.529 \\
Education & Std correctness & 0.9994 & 0.350 \\
\addlinespace
GPS & Trajectory length & 0.243 & 0.0287 \\
GPS & Total distance (km) & 0.284 & 0.142 \\
GPS & Average speed (km/h) & 0.312 & 0.108 \\
\midrule
Amazon & Mean across metrics & 0.368 & 0.166 \\
Movie & Mean across metrics & 0.373 & 0.284 \\
Education & Mean across metrics & 0.803 & 0.348 \\
GPS & Mean across metrics & 0.280 & 0.093 \\
\bottomrule
\end{tabular}
\end{table}

\begin{figure}[t]
\centering
\begin{subfigure}{0.48\columnwidth}
\centering
\includegraphics[width=\linewidth]{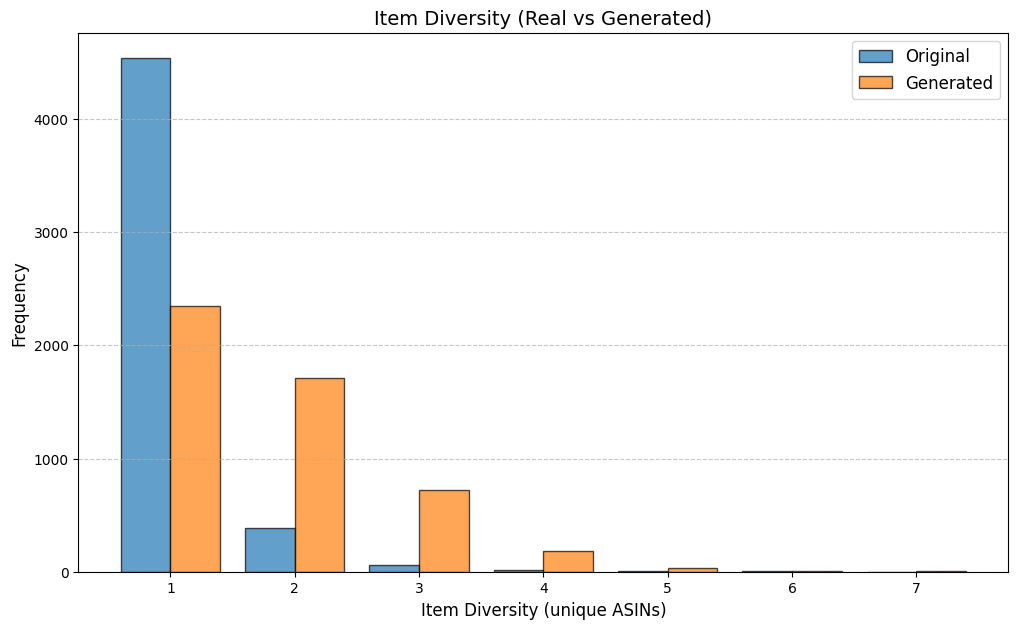}
\caption{RS: item diversity.}
\end{subfigure}\hfill
\begin{subfigure}{0.48\columnwidth}
\centering
\includegraphics[width=\linewidth]{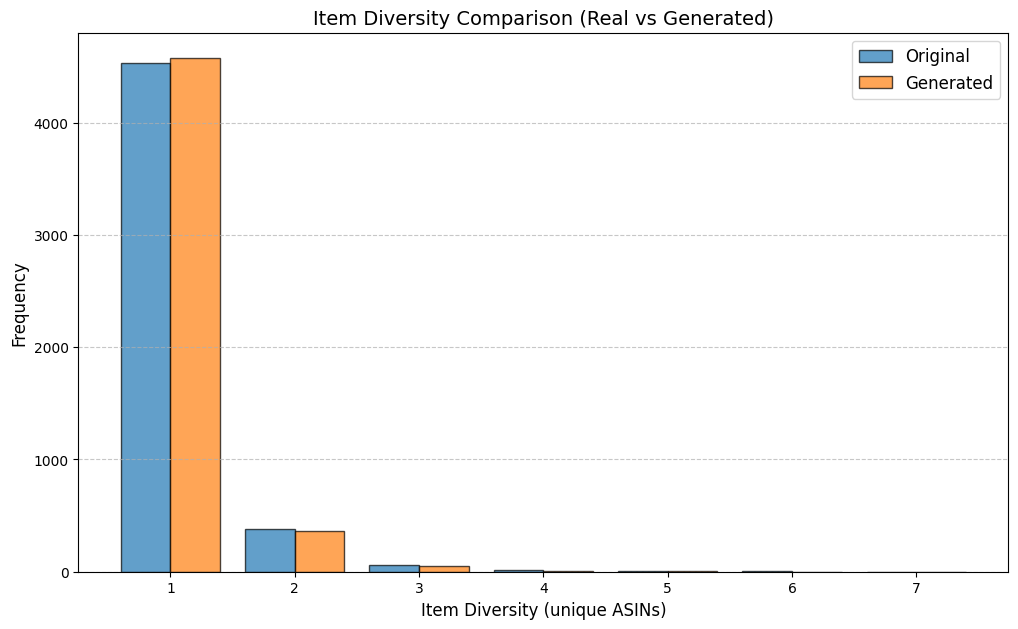}
\caption{LAS: item diversity.}
\end{subfigure}

\vspace{0.5em}

\begin{subfigure}{0.48\columnwidth}
\centering
\includegraphics[width=\linewidth]{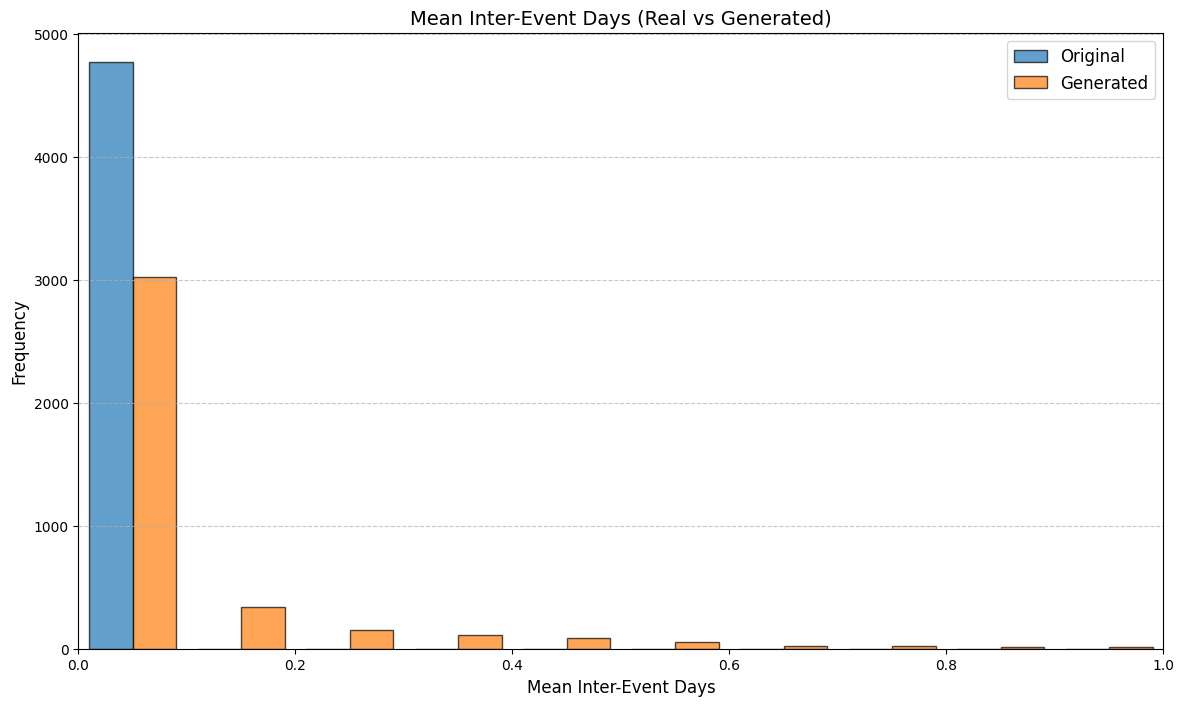}
\caption{RS: inter-event days.}
\end{subfigure}\hfill
\begin{subfigure}{0.48\columnwidth}
\centering
\includegraphics[width=\linewidth]{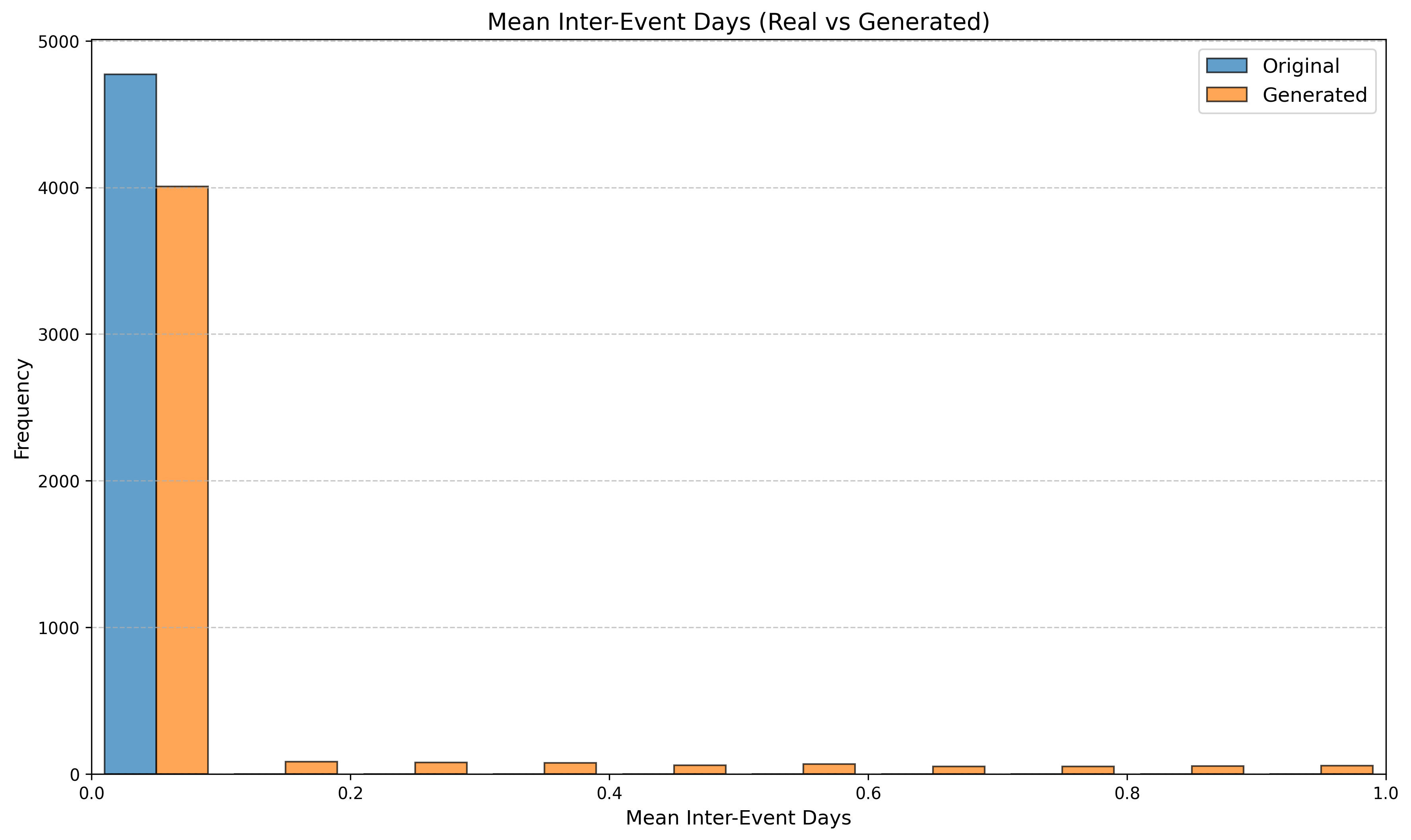}
\caption{LAS: inter-event days.}
\end{subfigure}
\caption{Amazon marginals. LAS improves agreement on diversity and timing-related derived variables.}
\label{fig:amazon_main}
\end{figure}

\begin{figure}[t]
\centering
\begin{subfigure}{0.48\columnwidth}
\centering
\includegraphics[width=\linewidth]{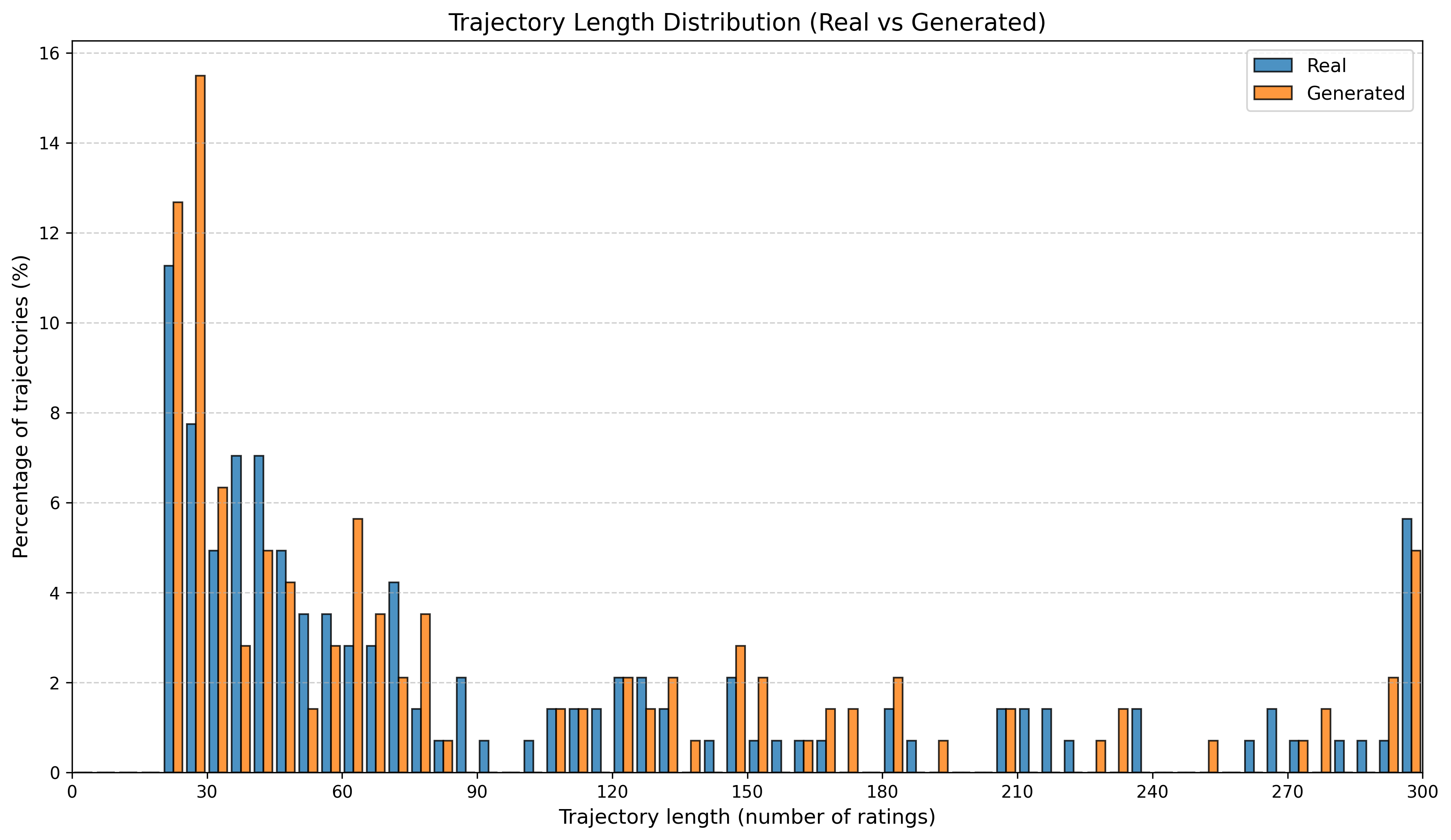}
\caption{RS: trajectory length.}
\end{subfigure}\hfill
\begin{subfigure}{0.48\columnwidth}
\centering
\includegraphics[width=\linewidth]{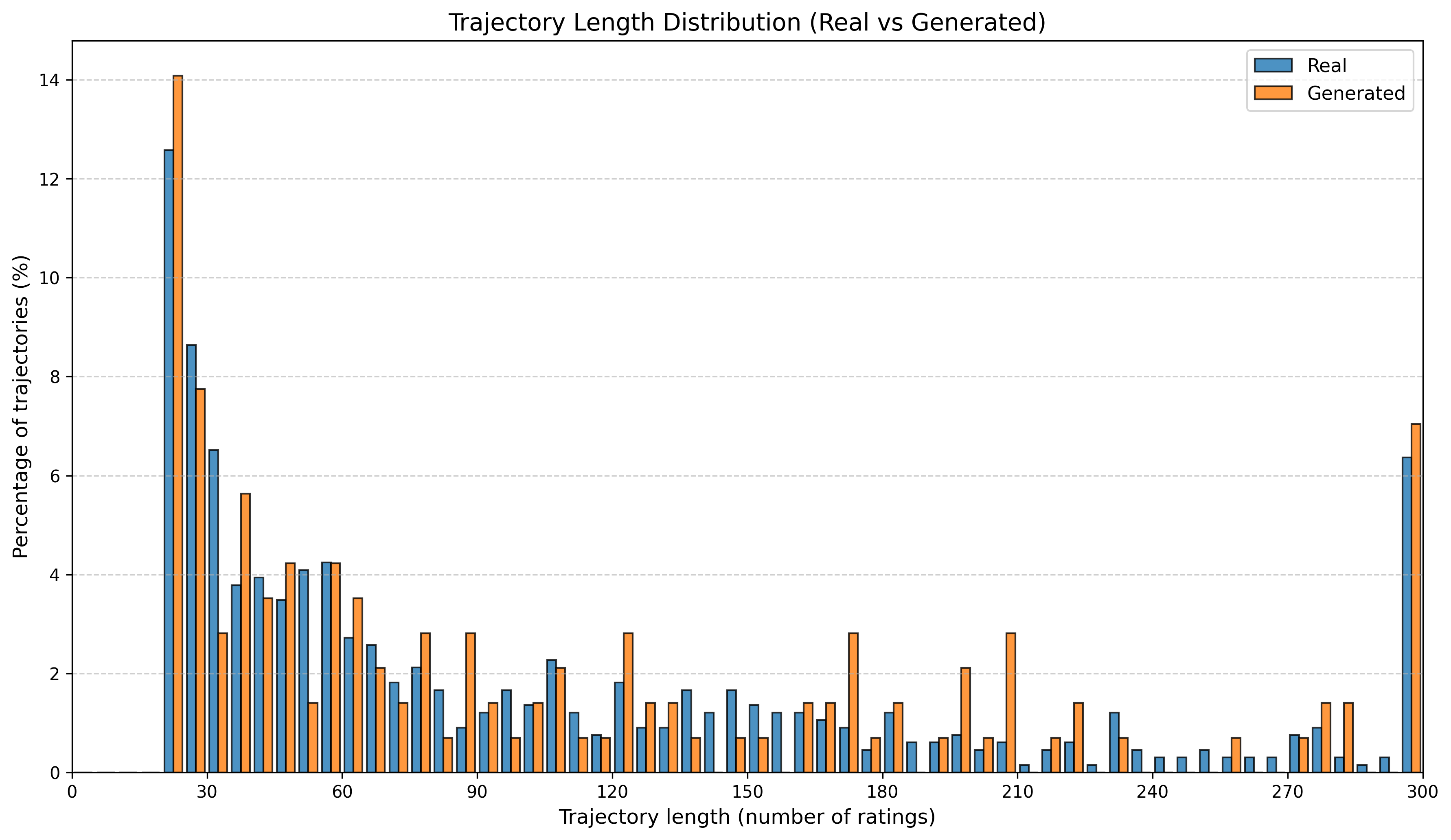}
\caption{LAS: trajectory length.}
\end{subfigure}

\vspace{0.5em}

\begin{subfigure}{0.48\columnwidth}
\centering
\includegraphics[width=\linewidth]{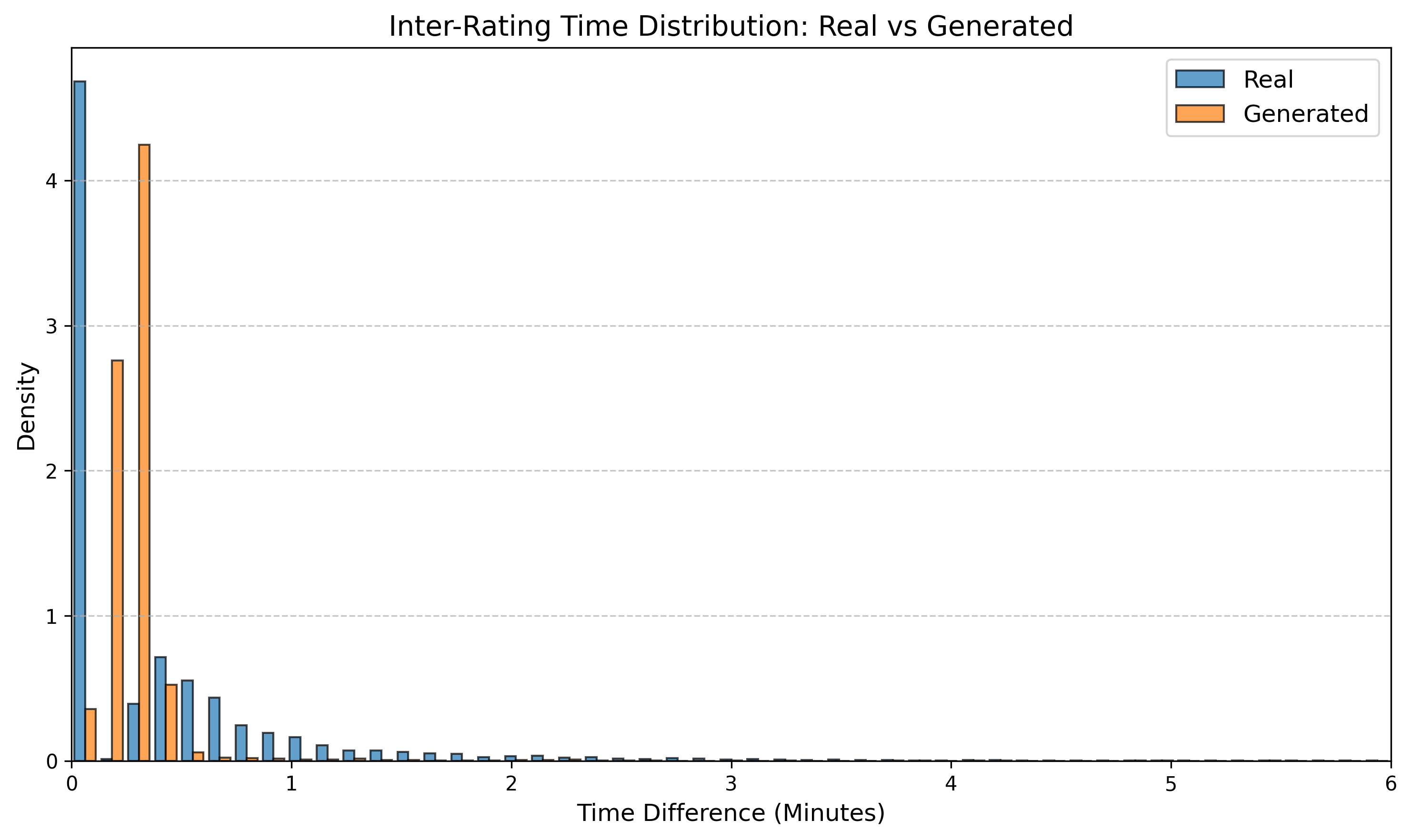}
\caption{RS: inter-rating time.}
\end{subfigure}\hfill
\begin{subfigure}{0.48\columnwidth}
\centering
\includegraphics[width=\linewidth]{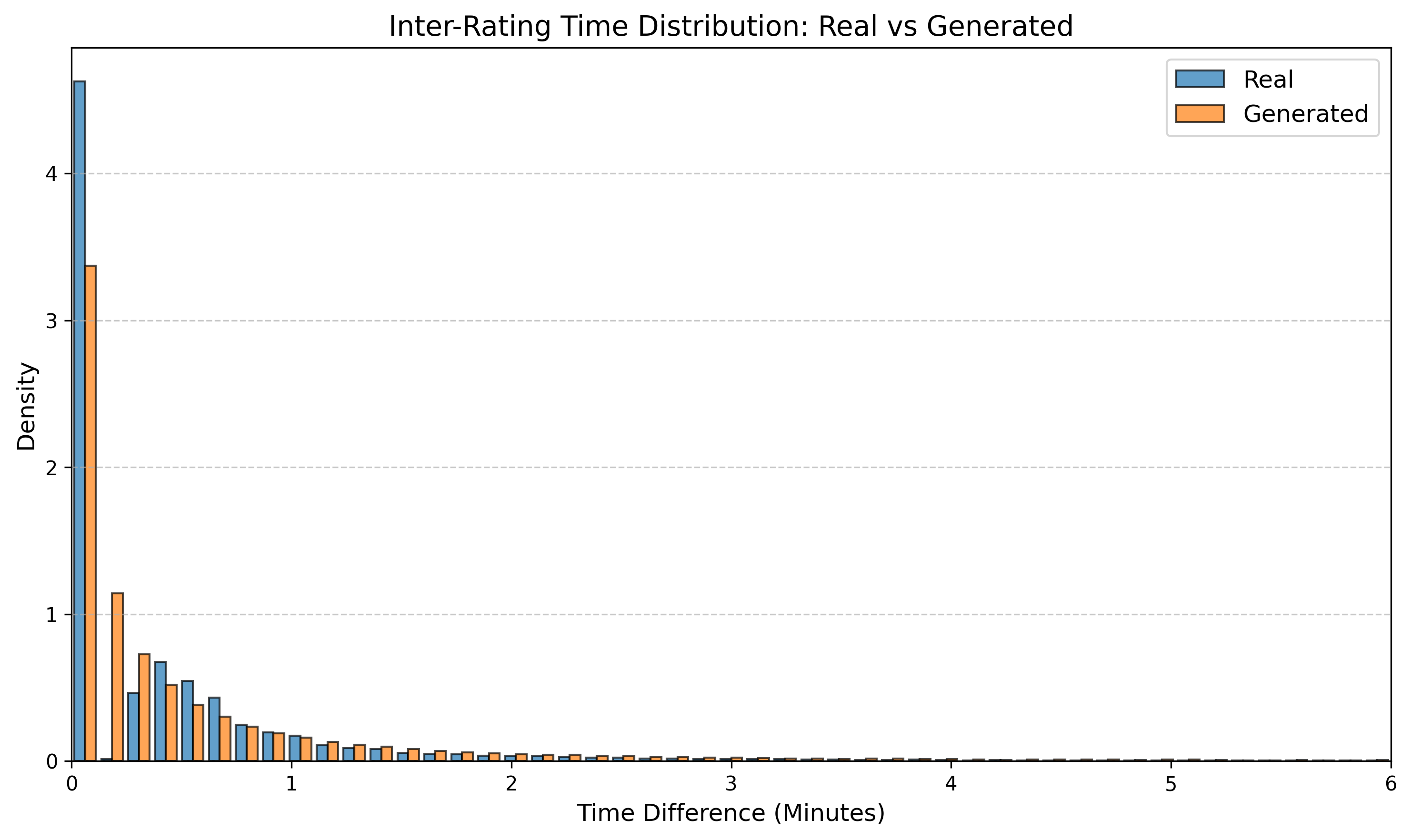}
\caption{LAS: inter-rating time.}
\end{subfigure}
\caption{Movie marginals. LAS improves both trajectory-length and inter-event timing distributions.}
\label{fig:movie_main}
\end{figure}
\begin{figure}[t]
\centering
\begin{subfigure}[b]{0.48\linewidth}
\centering
\includegraphics[width=\linewidth]{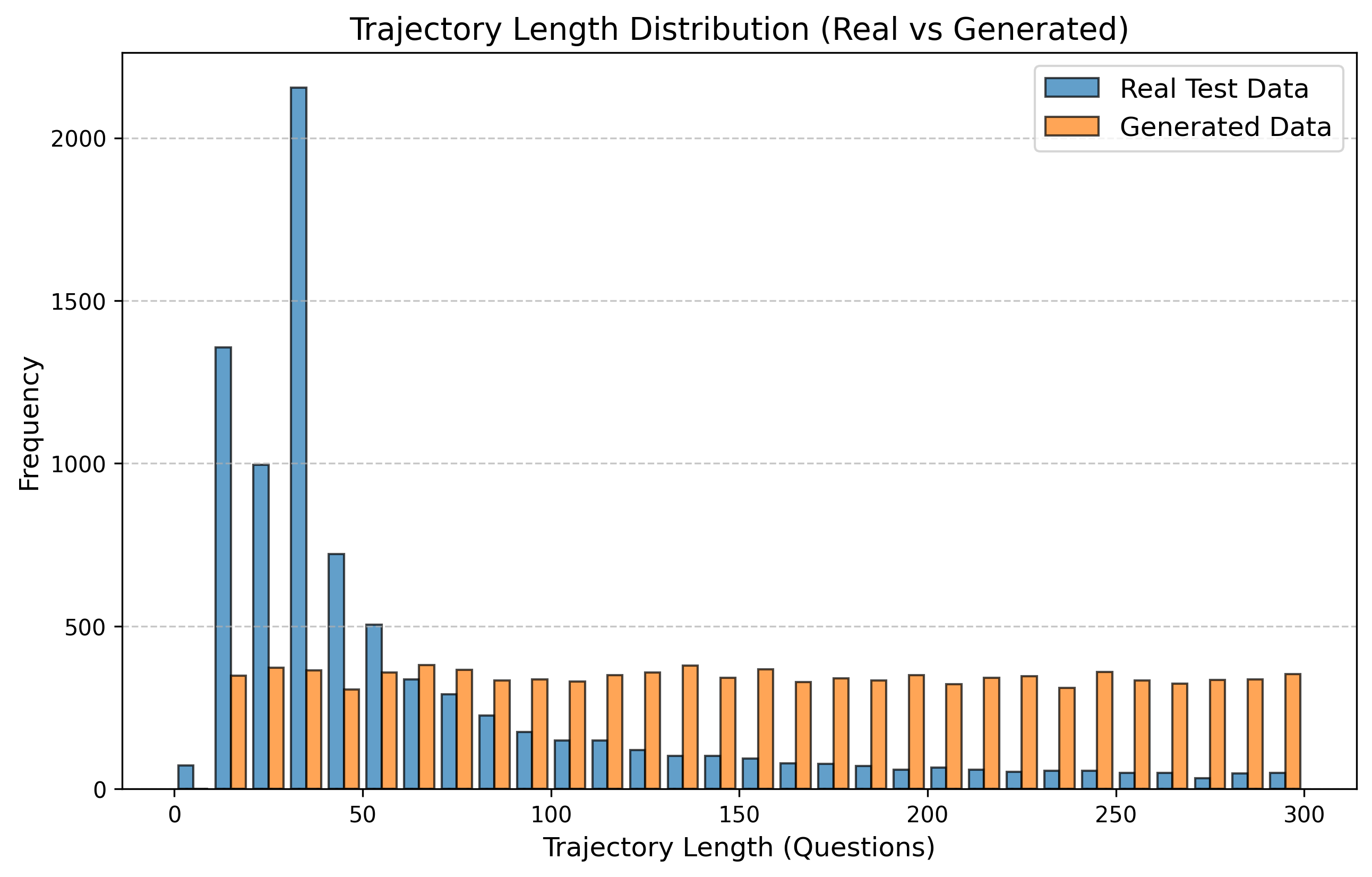}
\caption{RS: trajectory length}
\end{subfigure}
\hfill
\begin{subfigure}[b]{0.48\linewidth}
\centering
\includegraphics[width=\linewidth]{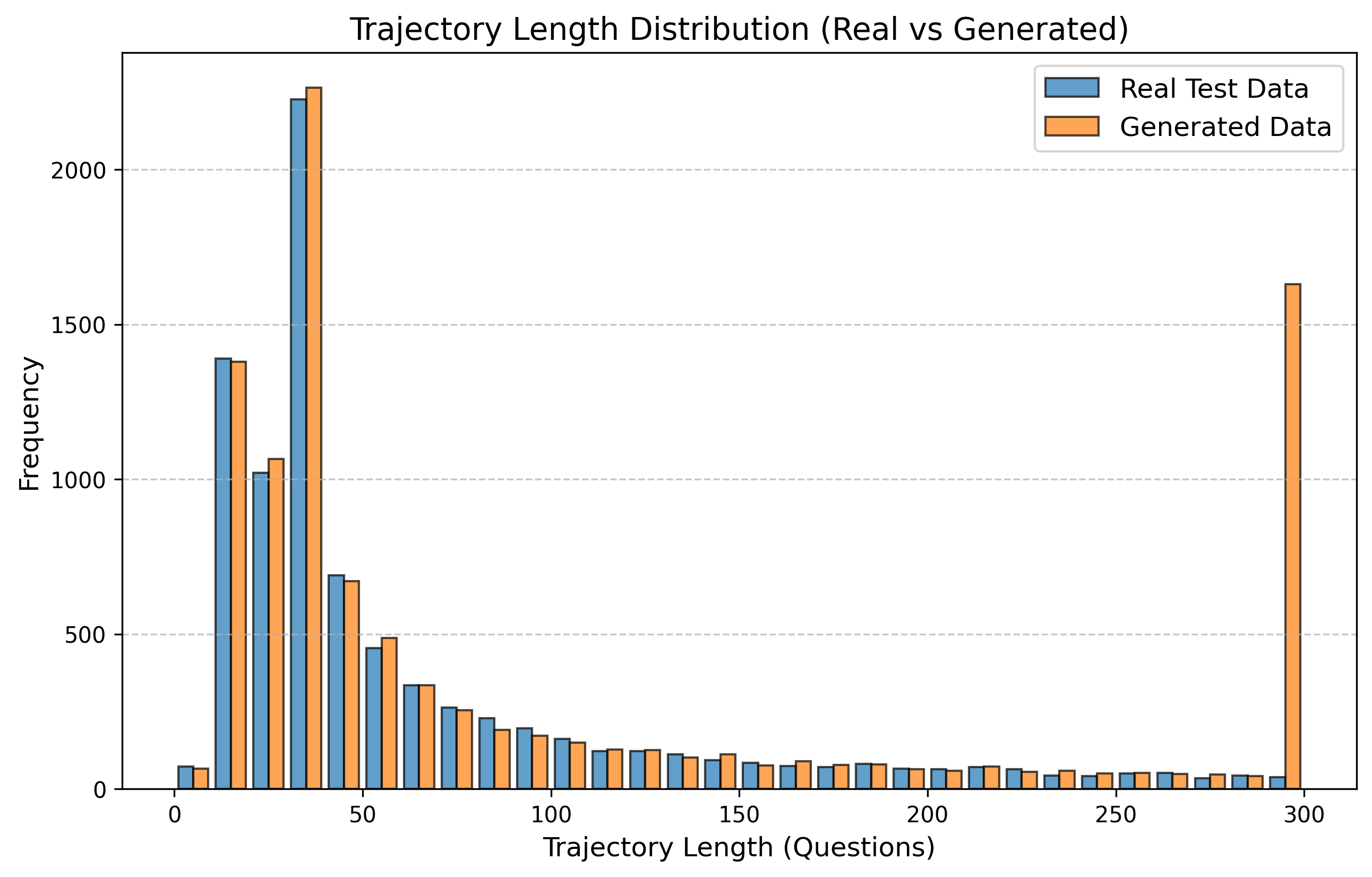}
\caption{LAS: trajectory length}
\end{subfigure}

\begin{subfigure}[b]{0.48\linewidth}
\centering
\includegraphics[width=\linewidth]{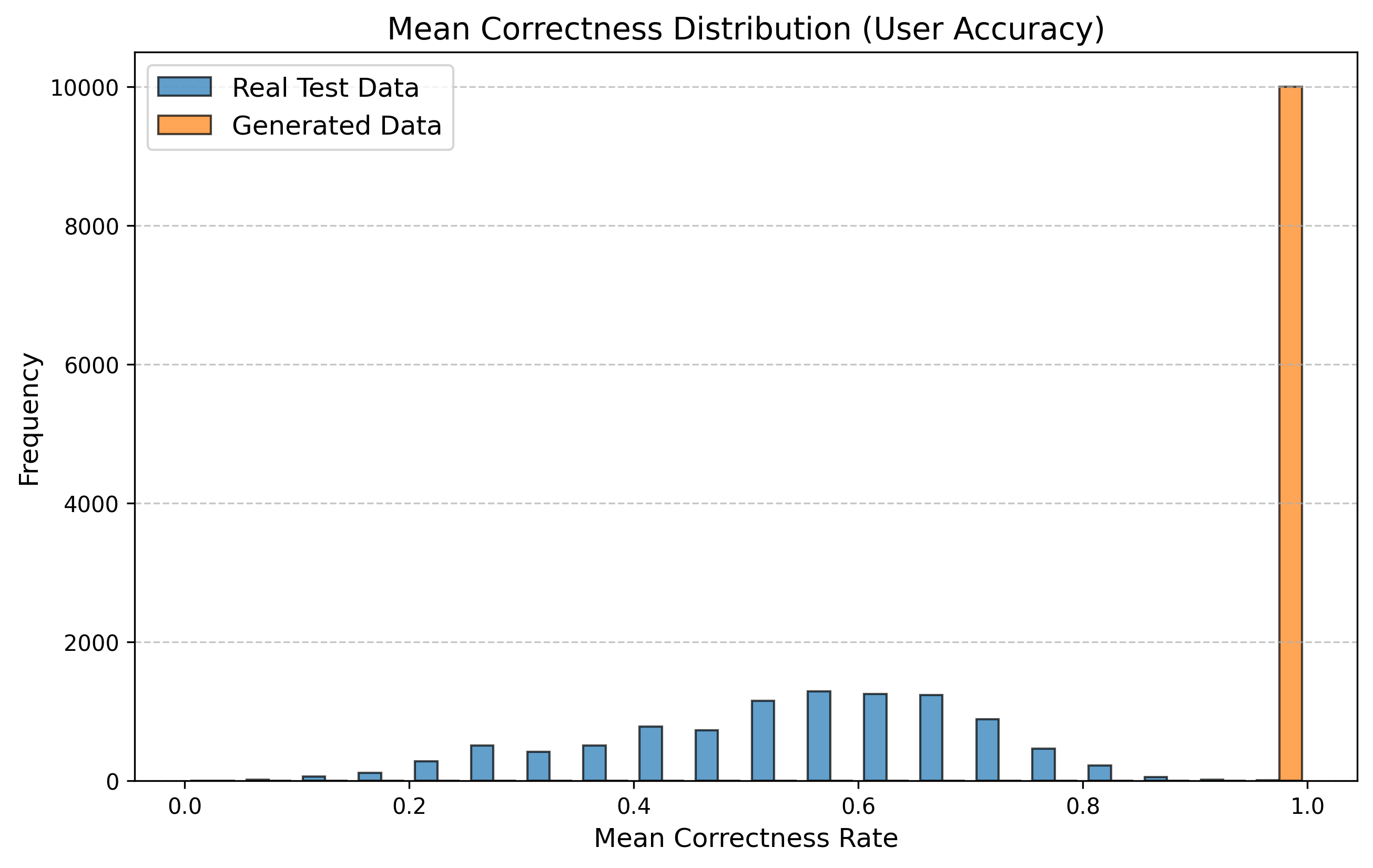}
\caption{RS: mean correctness}
\end{subfigure}
\hfill
\begin{subfigure}[b]{0.48\linewidth}
\centering
\includegraphics[width=\linewidth]{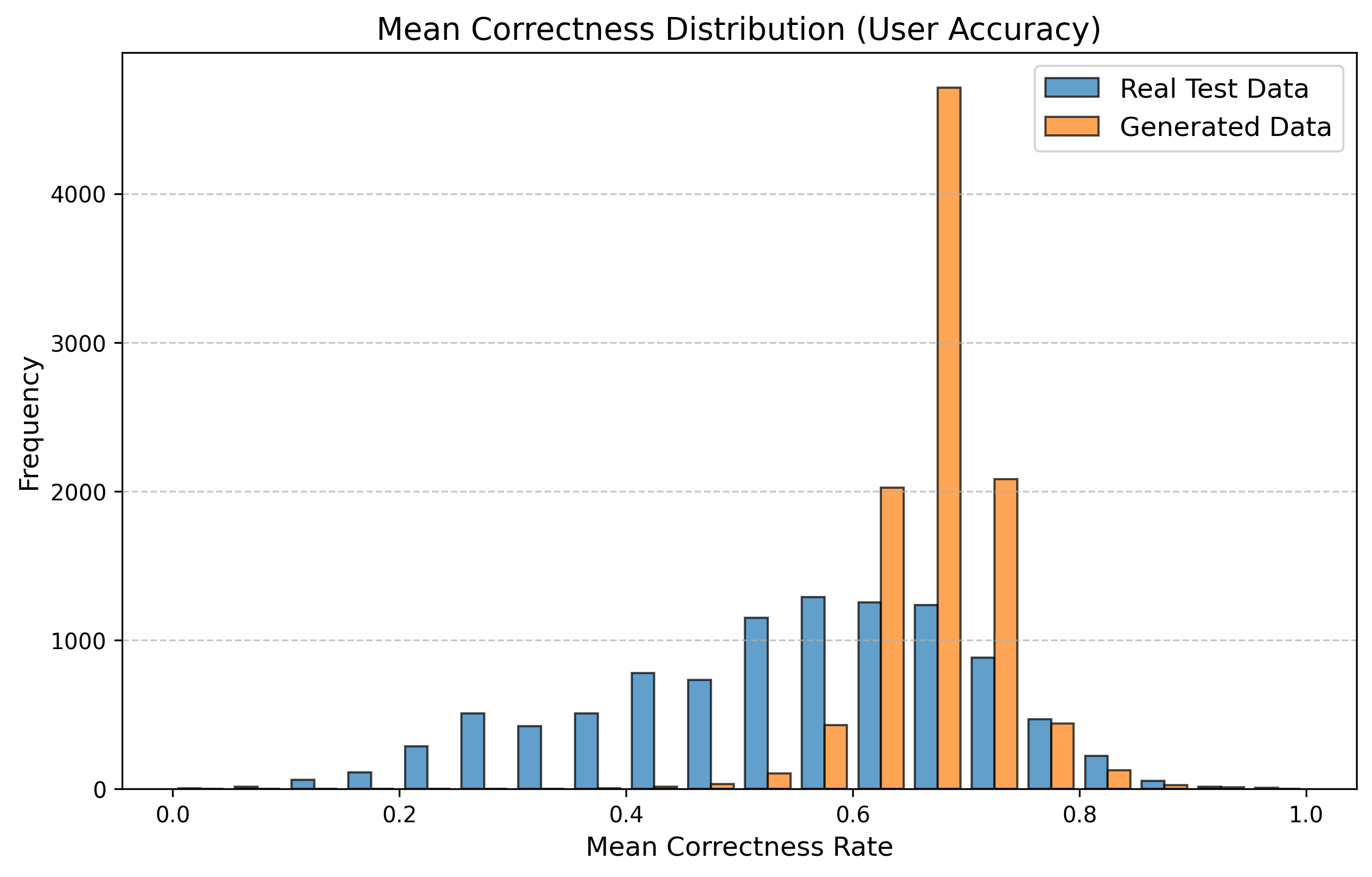}
\caption{LAS: mean correctness}
\end{subfigure}
\caption{Education dataset: representative marginals under RS and LAS.}
\label{fig:edu_main}
\end{figure}

\begin{figure}[t]
\centering
\begin{subfigure}[b]{0.48\linewidth}
\centering
\includegraphics[width=\linewidth]{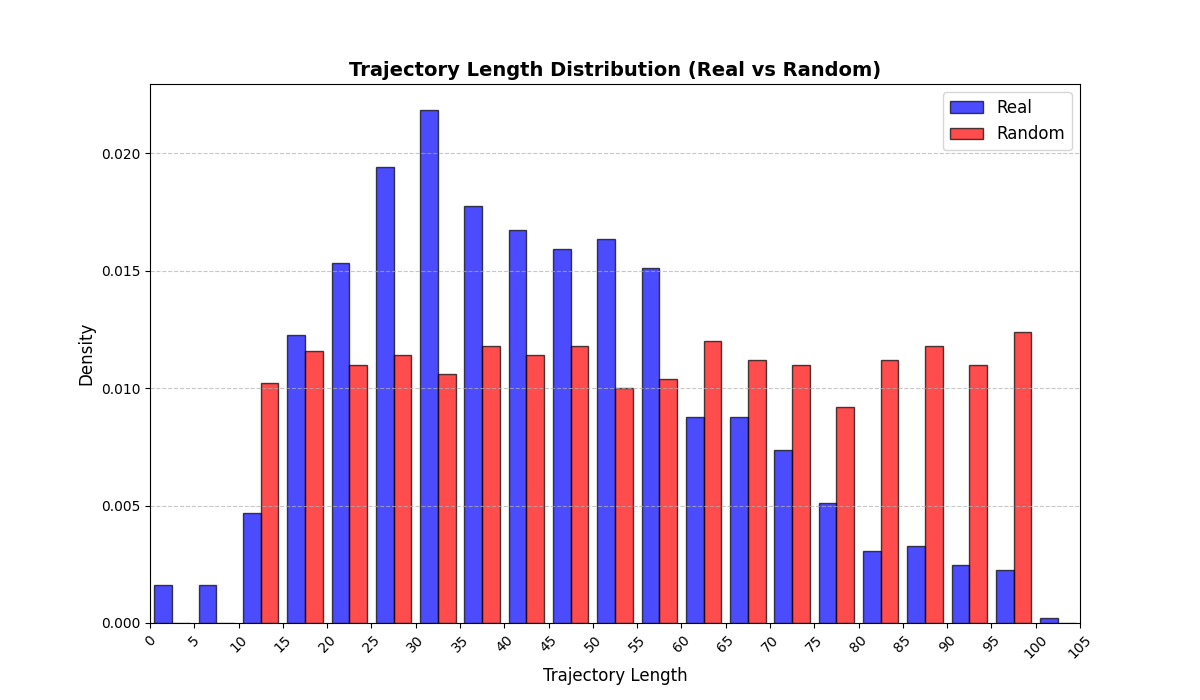}
\caption{RS: trajectory length}
\end{subfigure}
\hfill
\begin{subfigure}[b]{0.48\linewidth}
\centering
\includegraphics[width=\linewidth]{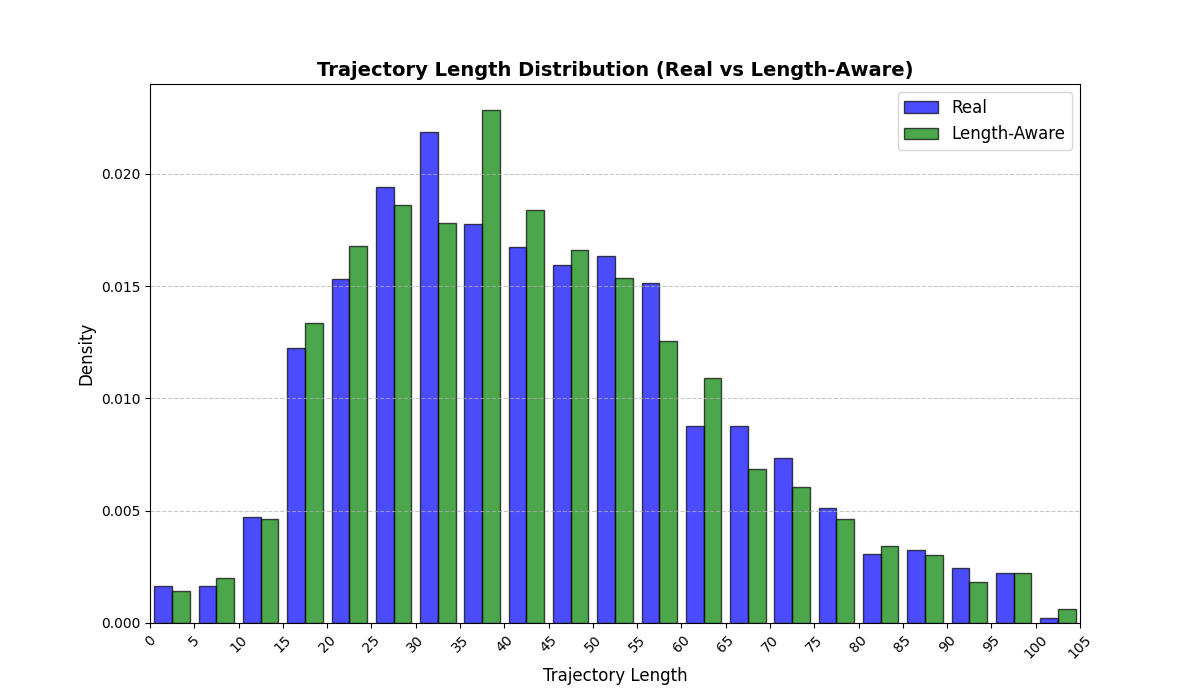}
\caption{LAS: trajectory length}
\end{subfigure}

\begin{subfigure}[b]{0.48\linewidth}
\centering
\includegraphics[width=\linewidth]{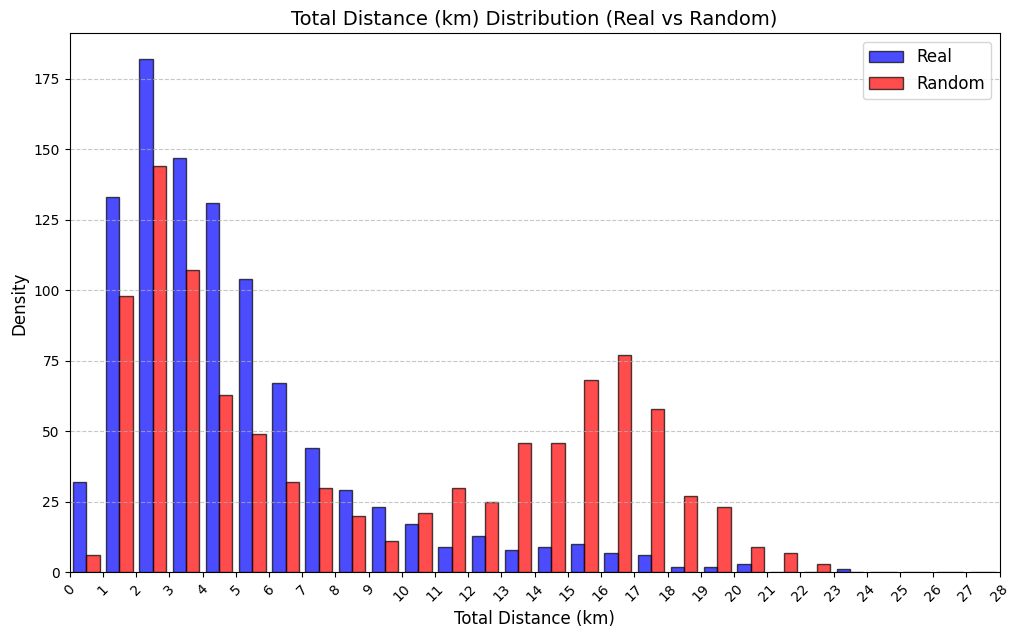}
\caption{RS: total distance}
\end{subfigure}
\hfill
\begin{subfigure}[b]{0.48\linewidth}
\centering
\includegraphics[width=\linewidth]{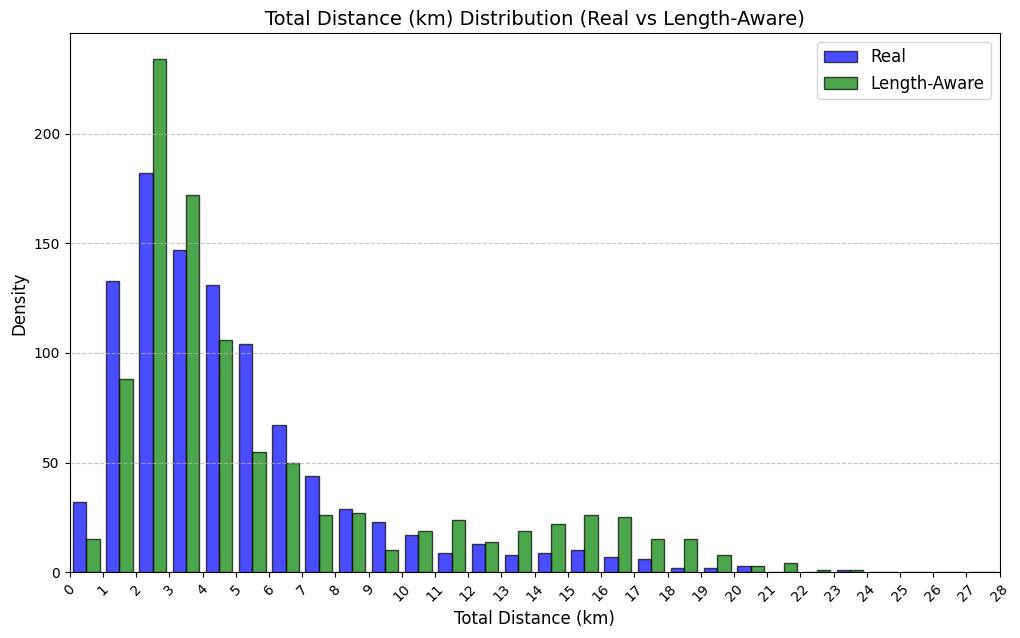}
\caption{LAS: total distance}
\end{subfigure}
\caption{GPS dataset: representative marginals under RS and LAS.}
\label{fig:gps_main}
\end{figure}

\subsection{Discussion}
LAS is most effective when the dataset exhibits \emph{substantial length heterogeneity}, where RS mixes short and long trajectories within a mini-batch and can make length an easy shortcut feature for the discriminator.
In such settings, we find LAS often yields more consistent adversarial updates and better distribution matching for length- and time-related derived variables.
We occasionally observe smaller gains (or mild regressions) on certain categorical marginals (e.g., store-type or floor distributions in some malls), suggesting that controlling length alone may not fully resolve capacity or representation limits of the underlying generator/discriminator.
Overall, LAS provides a strong ``drop-in'' improvement that improves distributional fidelity of key derived variables, and in many cases leads to more stable training behavior in practice.

\subsection{Controllability and what-if analyses}
\label{sec:controllability_main}
Beyond unconditional distributional fidelity, we also test whether the trained generator responds in intuitive directions to \emph{context} and \emph{spatial} perturbations.
These checks support downstream ``what-if'' analyses while remaining lightweight (full figures and additional details are in Appendix~\ref{app:full_eval}).

\paragraph{Conditional store influence (ON/OFF).}
We condition on whether a focal store $s^\ast$ is open and compare the generated distributions of visitation and dwell-time variables on ON days versus OFF days (Appendix~\ref{sec:store-influence}).
The model shifts visitation and time-allocation in the expected direction: when $s^\ast$ is open, trajectories exhibit increased propensity to include $s^\ast$ and reallocate time budget accordingly, whereas OFF days behave more like shorter, targeted trips.
This indicates the generator can reflect exogenous availability constraints rather than merely matching unconditional marginals.

\paragraph{Swapping experiments with gate distance.}
To probe sensitivity to mall layout, we ``swap'' (relocate) a target brand across stores at varying distances to the nearest gate and re-generate trajectories under the modified mapping (Appendix~\ref{sec:swap-gate}).
We observe smooth, distance-dependent changes in visitation and time-related metrics, consistent with learned spatial attractiveness rather than brittle length-only artifacts.
Together, these analyses suggest LAS-stabilized training yields a model that is not only accurate on marginal distributions but also meaningfully controllable under structured perturbations.

\FloatBarrier
\section{Conclusion}

We introduced length-aware sampling (LAS) for stabilizing variable-length trajectory generation and evaluated distributional fidelity on dataset-specific derived variables across proprietary mall data and additional trajectory datasets. Full model/training details, theory proofs, and additional plots are provided in the appendix.

\clearpage
\bibliography{uai2026-template}

\newpage
\appendix

\section{Model Architecture (Full Details)}
\label{app:arch_full}
\noindent Section~\ref{sec:method} provides a compact architecture summary; this appendix gives full details for reproducibility.

We model the mall environment and shopper behavior using a three-stage architecture:  
1) Store Feature Embedding with Attention Fusion,  
2) LSTM-based Conditional Trajectory Generator,  
3) Bidirectional LSTM-based Discriminator.

\subsection{Store Feature Embedding with Attention Fusion}

We represent the mall as a graph \( G = (V, E) \), where:
\begin{itemize}
    \item \( V = \{ v_1, \ldots, v_N \} \) is the set of stores;
    \item \( E \) is the adjacency set representing spatial connections between stores.
\end{itemize}

\paragraph{Store features and preprocessing.}  
Each store \( v_i \) is associated with a feature vector \( \mathbf{x}_i \in \mathbb{R}^{F_{\text{store}}} \).  
All store vectors form the matrix:
\[
\mathbf{X} = 
\begin{bmatrix}
\mathbf{x}_1^\top \\
\vdots \\
\mathbf{x}_N^\top
\end{bmatrix}
\in \mathbb{R}^{N \times F_{\text{store}}}
\]

The feature vector for store \(v_i\) is constructed as:
\[
\mathbf{x}_i =
\Big[
\begin{aligned}[t]
&\mathbf{x}_i^{(\text{id})};\ \mathbf{x}_i^{(\text{floor})};\ \mathbf{x}_i^{(\text{traffic})};\ \mathbf{x}_i^{(\text{open})};\\
&\mathbf{x}_i^{(\text{degree})};\ \mathbf{x}_i^{(\text{neighbor-counts})};\ \mathbf{x}_i^{(\text{neighbor-pct})};\\
&\mathbf{x}_i^{(\text{hop})};\ \mathbf{x}_i^{(\text{scope})}
\end{aligned}
\Big].
\]

Preprocessing steps:
\begin{itemize}
    \item \(\mathbf{x}_i^{(\text{id})}\): one-hot encoding of store identity (\(N\)-dim).
    \item \(\mathbf{x}_i^{(\text{floor})}\): one-hot encoding of floor identity.
    \item \(\mathbf{x}_i^{(\text{traffic})}\): daily visitor count, clipped at the 95th percentile, log-transformed, and standardized:
    \[
    x_{\text{traffic}} = 
    \frac{\log(1 + \min(\text{count}, p_{95})) - \mu}{\sigma}
    \]
    where \(p_{95}\) is the 95th percentile, \(\mu, \sigma\) are mean and std.
    \item \(\mathbf{x}_i^{(\text{open})}\): binary open/closed indicator.
    \item \(\mathbf{x}_i^{(\text{degree})}\): graph degree of store \(v_i\) (the number of directly connected neighboring stores), min-max scaled to \([0,1]\).
    \item \(\mathbf{x}_i^{(\text{neighbor-counts})}\): log-scaled raw counts of neighboring categories, aggregated by category type rather than the total number of neighbors.
    \item \(\mathbf{x}_i^{(\text{neighbor-pct})}\): normalized percentages of neighboring categories, where the proportions across all neighboring categories sum to 1.
    \item \(\mathbf{x}_i^{(\text{hop})}\): shortest hop distance to key facilities (elevators, escalators, gates), normalized to \([0,1]\).
    \item \(\mathbf{x}_i^{(\text{scope})}\): binary nationwide vs. regional scope indicator.
\end{itemize}

\paragraph{Mall-level context features.}  
At day \(\delta\), the mall-level context vector is defined as:
\[
\begin{aligned}
\mathbf{m}^{(\delta)} &=
\left[
\begin{aligned}[t]
&\mathbf{m}^{(\text{theme},\delta)};\ \mathbf{m}^{(\text{campaign},\delta)};\ \mathbf{m}^{(\text{temp},\delta)};\\
&\mathbf{m}^{(\text{precip},\delta)};\ \mathbf{m}^{(\text{sunshine},\delta)};\ \mathbf{m}^{(\text{wind},\delta)};\ 
\mathbf{m}^{(\text{weather},\delta)}
\end{aligned}
\right],\\
\mathbf{m}^{(\delta)} &\in \mathbb{R}^{F_{\text{mall}}}.
\end{aligned}
\]

\begin{table}[t]
\centering
\footnotesize
\setlength{\tabcolsep}{4pt}
\begin{tabularx}{\columnwidth}{l|>{\raggedright\arraybackslash}X|c}
\toprule
\textbf{Feature group} & \textbf{Description} & \textbf{Dim.}\\
\midrule
Store ID & One-hot identity & \(N\) \\
Floor & One-hot floor encoding & \#floors \\
Traffic & Daily visitor count (log-scaled and standardized) & 1 \\
Open status & Binary open/closed & 1 \\
Degree & Graph degree (number of neighboring stores) & 1 \\
Neighbor counts & Counts of neighboring categories (log-scaled) & \#categories \\
Neighbor percentages & Normalized proportions of neighboring categories & \#categories \\
Hop distance & Shortest path to key facilities (normalized) & \#facilities \\
Scope & Nationwide vs.\ regional indicator & 1 \\
\midrule
Mall context & Theme, campaigns, weather, and related mall-level factors & \(F_{\text{mall}}\) \\
\bottomrule
\end{tabularx}
\caption{Summary of feature categories and their dimensions.}
\end{table}

\paragraph{Attention-based Feature Fusion.}
We project the three feature groups—store, neighbor, and mall—into a shared embedding space using linear transformations followed by ReLU activations:

\[
\begin{aligned}
\mathbf{s}_i^{(\text{store})}
&= \operatorname{ReLU}\!\left(\mathbf{W}_{\text{store-emb}}\,\mathbf{x}_i^{(\text{store})}\right),\\
\mathbf{W}_{\text{store-emb}} &\in \mathbb{R}^{d_{\text{embed}}\times d_{\text{store}}},\quad
\mathbf{s}_i^{(\text{store})}\in \mathbb{R}^{d_{\text{embed}}},\\[0.4em]
\mathbf{s}_i^{(\text{neighbor})}
&= \operatorname{ReLU}\!\left(\mathbf{W}_{\text{neighbor-emb}}\,\mathbf{x}_i^{(\text{neighbor})}\right),\\
\mathbf{W}_{\text{neighbor-emb}} &\in \mathbb{R}^{d_{\text{embed}}\times d_{\text{neighbor}}},\quad
\mathbf{s}_i^{(\text{neighbor})}\in \mathbb{R}^{d_{\text{embed}}},\\[0.4em]
\mathbf{s}_i^{(\text{mall})}
&= \operatorname{ReLU}\!\left(\mathbf{W}_{\text{mall-emb}}\,\mathbf{m}^{(\delta)}\right),\\
\mathbf{W}_{\text{mall-emb}} &\in \mathbb{R}^{d_{\text{embed}}\times d_{\text{mall}}},\quad
\mathbf{s}_i^{(\text{mall})}\in \mathbb{R}^{d_{\text{embed}}}.
\end{aligned}
\]

These projected embeddings are stacked to form a matrix:
\[
\mathbf{S}_i = 
\begin{bmatrix}
\mathbf{s}_i^{(\text{store})} \\
\mathbf{s}_i^{(\text{neighbor})} \\
\mathbf{s}_i^{(\text{mall})}
\end{bmatrix} 
\in \mathbb{R}^{3 \times d_{\text{embed}}}
\]

We then compute attention weights using a shared linear transformation followed by LeakyReLU and Softmax:
\[
\begin{aligned}
\boldsymbol{\alpha}_i^{(\text{attn})}
&= \operatorname{Softmax}\!\left(
\operatorname{LeakyReLU}\!\left(\mathbf{S}_i\,\mathbf{w}_{\text{attn}}\right)\right),\\
\mathbf{w}_{\text{attn}} &\in \mathbb{R}^{d_{\text{embed}}},\quad
\boldsymbol{\alpha}_i^{(\text{attn})}\in \mathbb{R}^{3}.
\end{aligned}
\]

The final attention-fused embedding is the weighted sum:
\[
\mathbf{s}_i^{\text{embed}} = \sum_{j \in \{\text{store}, \text{neighbor}, \text{mall}\}} \alpha_i^{(\text{attn}, j)} \cdot \mathbf{s}_i^{(j)} 
\in \mathbb{R}^{d_{\text{embed}}}
\]

\subsection{Conditional Sequence Generator (LSTM)}

To avoid notation clashes, we explicitly denote the store index at timestep \(t\) as \(j_{t}\).  
The generator is modeled as a conditional LSTM that recursively produces the next store in the trajectory until it outputs a special end-of-trajectory token. The generated trajectory can have variable length \(\hat{T}\).

At each timestep \( t \):
\[
\mathbf{h}_t, \mathbf{c}_t = \text{LSTMCell}(\mathbf{u}_t, \mathbf{h}_{t-1}, \mathbf{c}_{t-1})
\]

where \(\mathbf{u}_t\) is the input vector:
\[
\begin{aligned}
\mathbf{u}_t &=
\left[
\mathbf{s}_{j_{t-1}}^{\text{embed}};\ \mathbf{z}_{\text{latent}};\ \mathbf{v};\ 
\tau_t^{(\text{intra})};\ \tau_t^{(\text{inter})}
\right],\\
\mathbf{u}_t &\in \mathbb{R}^{d_{\text{embed}} + d_{\text{latent}} + d_{\text{visitor}} + 2}.
\end{aligned}
\]

\paragraph{Input components.}  
\begin{itemize}
    \item \( \mathbf{s}_{j_{t-1}}^{\text{embed}} \in \mathbb{R}^{d_{\text{embed}}} \) — attention-fused embedding of the previously visited store;
    \item \( \mathbf{z}_{\text{latent}} \in \mathbb{R}^{d_{\text{latent}}} \) — latent noise vector sampled from a prior distribution (e.g., \(\mathcal{N}(0, I)\)) to introduce diversity in the generated sequences;
    \item \( \mathbf{v} \in \mathbb{R}^{d_{\text{visitor}}} \) — visitor demographic or context embedding;
    \item \( \tau_t^{(\text{intra})},\ \tau_t^{(\text{inter})} \in \mathbb{R}^1 \) — intra- and inter-visit time intervals.
\end{itemize}

\paragraph{LSTM cell details.}
We use an LSTM with hidden state dimension \(H\). Its gate equations are:
\begin{align*}
\mathbf{i}_t &= \sigma \!\left( \mathbf{W}_i \mathbf{u}_t + \mathbf{U}_i \mathbf{h}_{t-1} + \mathbf{b}_i \right), \\
\mathbf{f}_t &= \sigma \!\left( \mathbf{W}_f \mathbf{u}_t + \mathbf{U}_f \mathbf{h}_{t-1} + \mathbf{b}_f \right), \\
\mathbf{o}_t^{(\text{gate})} &= \sigma \!\left( \mathbf{W}_o \mathbf{u}_t + \mathbf{U}_o \mathbf{h}_{t-1} + \mathbf{b}_o \right), \\
\tilde{\mathbf{c}}_t &= \tanh \!\left( \mathbf{W}_c \mathbf{u}_t + \mathbf{U}_c \mathbf{h}_{t-1} + \mathbf{b}_c \right), \\
\mathbf{c}_t &= \mathbf{f}_t \odot \mathbf{c}_{t-1} + \mathbf{i}_t \odot \tilde{\mathbf{c}}_t, \\
\mathbf{h}_t &= \mathbf{o}_t^{(\text{gate})} \odot \tanh(\mathbf{c}_t).
\end{align*}
Here \(\sigma(\cdot)\) is the sigmoid function, \(\odot\) denotes element-wise multiplication, and
\(\mathbf{h}_t,\mathbf{c}_t \in \mathbb{R}^H\) are the hidden and cell states.

Because the generator stops when the end-of-trajectory token is sampled, the length \(\hat{T}\) of the generated sequence may vary from sample to sample.

where:
\begin{itemize}
    \item \( \sigma(\cdot) \) is the sigmoid activation function,
    \item \( \odot \) denotes element-wise multiplication,
    \item \( \mathbf{h}_t, \mathbf{c}_t \in \mathbb{R}^H \) are the hidden and cell states.
\end{itemize}

Because the generator stops when the end-of-trajectory token is sampled, the length \(\hat{T}\) of the generated sequence may vary from sample to sample.  

\paragraph{Output layer.}  
At each timestep \( t \), the generator produces three outputs from the hidden state \(\mathbf{h}_t \in \mathbb{R}^{H}\):  

1. Store index prediction (plus end-of-trajectory token):  
\[
\begin{aligned}
\mathbf{o}_t^{\text{store}} &= \mathbf{W}_{\text{store-out}} \mathbf{h}_t + \mathbf{b}_{\text{store}},\\
\mathbf{W}_{\text{store-out}} &\in \mathbb{R}^{(N+1)\times H},\quad
\mathbf{b}_{\text{store}} \in \mathbb{R}^{N+1},\quad
\mathbf{o}_t^{\text{store}} \in \mathbb{R}^{N+1}.
\end{aligned}
\]
A softmax is applied:
\[
\mathbf{p}_t^{\text{store}} = \text{Softmax}(\mathbf{o}_t^{\text{store}}),
\quad \mathbf{p}_t^{\text{store}} \in [0,1]^{N+1},
\quad \sum_{k=1}^{N+1} p_{t,k}^{\text{store}} = 1
\]
where the \((N+1)\)-th entry is a special token indicating the end of the trajectory.  

2. Intra-visit time prediction:  
\[
\begin{aligned}
\hat{\tau}_t^{(\text{intra})} &= \mathbf{w}_{\text{intra}}^\top \mathbf{h}_t + b_{\text{intra}},\\
\mathbf{w}_{\text{intra}} &\in \mathbb{R}^{H},\quad
b_{\text{intra}} \in \mathbb{R}.
\end{aligned}
\]

3. Inter-visit time prediction:  
\[
\begin{aligned}
\hat{\tau}_t^{(\text{inter})} &= \mathbf{w}_{\text{inter}}^\top \mathbf{h}_t + b_{\text{inter}},\\
\mathbf{w}_{\text{inter}} &\in \mathbb{R}^{H},\quad
b_{\text{inter}} \in \mathbb{R}.
\end{aligned}
\]

Thus, the generator outputs a predicted next-store distribution, along with intra- and inter-visit time estimates, for as many timesteps as needed until the special end-of-trajectory token is predicted, resulting in a variable-length generated sequence.

\subsection{Discriminator Architecture}
At each timestep \(t\), the discriminator receives the store embedding and the corresponding time intervals:
\[
\mathbf{x}_t^{\text{disc}}
=\big[\,\mathbf{s}_{j_t}^{\text{embed}};\ \tau_t^{(\text{intra})};\ \tau_t^{(\text{inter})}\,\big]
\in \mathbb{R}^{d_{\text{embed}} + 2}.
\]
\vspace{-0.35em}

The input sequence is variable-length,
\[
x = \big(\mathbf{x}_1^{\text{disc}}, \ldots, \mathbf{x}_L^{\text{disc}}\big),
\]
where \(L=T\) for real trajectories and \(L=\hat{T}\) for generated trajectories.
This sequence is processed by a bidirectional LSTM with hidden size \(H_D\):
\vspace{-0.25em}
\begin{align*}
\mathbf{h}_t^{\rightarrow} &= \operatorname{LSTM}_{\text{fwd}}\!\left(\mathbf{x}_t^{\text{disc}},\ \mathbf{h}_{t-1}^{\rightarrow}\right),\\
\mathbf{h}_t^{\leftarrow} &= \operatorname{LSTM}_{\text{bwd}}\!\left(\mathbf{x}_t^{\text{disc}},\ \mathbf{h}_{t+1}^{\leftarrow}\right).
\end{align*}

\paragraph{Intuition of bidirectionality and variable-length handling.}
Unlike the generator, the discriminator is not constrained to operate sequentially in one direction. 
Using a bidirectional LSTM allows it to incorporate both past and future context when evaluating each timestep. 
For example, whether a visit to a particular store is realistic may depend not only on the previous visits but also on the subsequent visits. 
This holistic view of the entire sequence enables the discriminator to more effectively detect unrealistic transitions or inconsistencies that might otherwise be missed if it only processed the sequence forward in time.

Because LSTMs process sequences one timestep at a time, they naturally support variable-length inputs: they unroll for as many timesteps as are available in the input trajectory and then stop.  
For real trajectories of length \(T\) and generated trajectories of length \(\hat{T}\), the bidirectional LSTM runs until the end of each sequence without requiring padding or truncation.  
At the sequence level, the discriminator uses the forward hidden state at the last valid timestep \(\mathbf{h}_L^{\rightarrow}\) and the backward hidden state at the first timestep \(\mathbf{h}_1^{\leftarrow}\), ensuring that the entire trajectory is fully represented regardless of its length.

The final feature vector concatenates the last forward and backward states with the visitor context vector \(\mathbf{v}\):
\[
\mathbf{f}^{\text{disc}} = 
\left[ \mathbf{h}_L^{\rightarrow};\ \mathbf{h}_1^{\leftarrow};\ \mathbf{v} \right]
\in \mathbb{R}^{2H_D + d_{\text{visitor}}}
\]

Finally, the discriminator outputs the real/fake probability:
\[
\begin{aligned}
\hat{y} &= \sigma\!\left(\mathbf{W}_{d\text{-out}}\,\mathbf{f}^{\text{disc}} + b_{d\text{-out}}\right),\\
\mathbf{W}_{d\text{-out}} &\in \mathbb{R}^{1 \times (2H_D + d_{\text{visitor}})},\quad
b_{d\text{-out}} \in \mathbb{R},\quad
\hat{y}\in(0,1).
\end{aligned}
\]

\section{Loss Functions (Full Details)}
\label{app:loss_full}
\noindent The main text states the training objective at a high level; we collect full loss definitions and dataset-specific variants here.
The training objective consists of separate losses for the generator and the discriminator.\par
\smallskip\noindent
The generator loss \(\mathcal{L}_G\) combines an adversarial term with optional auxiliary terms, while the discriminator loss \(\mathcal{L}_D\) is purely adversarial.

\paragraph{Mall objective (timing-aligned adversarial training).}
For the mall experiments, we use an adversarial objective augmented with explicit intra-/inter-event timing alignment:
\[
\mathcal{L}_G 
\;=\; \mathcal{L}_{\mathrm{adv}}
\;+\; \lambda_{\mathrm{time}} \big( \mathcal{L}_{\mathrm{intra}} + \mathcal{L}_{\mathrm{inter}} \big),
\]
where \(\lambda_{\mathrm{time}}>0\) controls the relative weight of the temporal alignment terms.
Here \(\mathcal{L}_{\mathrm{intra}}\) penalizes mismatches in intra-store (within-stop) timing, and \(\mathcal{L}_{\mathrm{inter}}\) penalizes mismatches in inter-store transition timing.

\paragraph{Public dataset objectives.}
For the public datasets, we use dataset-specific adversarial objectives (and keep the objective fixed when comparing RS vs.\ LAS within each dataset; only the batching strategy differs):
\textbf{Education}: adversarial loss only (treating the data as a generic sequence; the discriminator implicitly learns timing/structure);
\textbf{GPS}: adversarial loss only;
\textbf{Movie}: adversarial loss with an additional feature-matching term \(\mathcal{L}_{\mathrm{fm}}\);
\textbf{Amazon}: Wasserstein (WGAN-style) loss for improved training stability.

\noindent We explored additional auxiliary terms in preliminary experiments; unless stated otherwise, the results in the paper use the objectives specified above.

\subsection{Adversarial Loss}

The adversarial loss encourages the generator to produce realistic trajectories that fool the discriminator:
\[
\mathcal{L}_{\text{adv}} = 
- \mathbb{E}_{\hat{x} \sim G} 
\left[ \log D(\hat{x}) \right]
\]

Here, \(\hat{x}\) represents a generated trajectory, defined as the sequence:
\[
\hat{x} = \left\{ 
\mathbf{s}_{j_t}^{\text{embed}}, 
\hat{\tau}_t^{(\text{intra})}, 
\hat{\tau}_t^{(\text{inter})}
\right\}_{t=1}^{\hat{T}}
\]
where \(\mathbf{s}_{j_t}^{\text{embed}}\) is the embedding of the predicted store index \(j_t\), 
\(\hat{\tau}_t^{(\text{intra})}, \hat{\tau}_t^{(\text{inter})}\) are the generator-predicted time intervals, 
and \(\hat{T}\) is the (possibly variable) generated sequence length.  
This matches the discriminator input described in the previous section: the discriminator never sees the raw store index \(j_t\) directly but instead receives the corresponding embeddings and predicted time intervals.

\subsection{Intra-Store Time Prediction Loss}

This term enforces accurate prediction of intra-store visit durations:
\[
\mathcal{L}_{\text{intra}} = 
\frac{1}{\min(T, \hat{T})}
\sum_{t=1}^{\min(T, \hat{T})} 
\left| \hat{\tau}_t^{(\text{intra})} - \tau_t^{(\text{intra})} \right|
\]
where \(T\) is the length of the real trajectory and \(\hat{T}\) is the length of the generated trajectory.

\subsection{Inter-Store Time Prediction Loss}

Similarly, the inter-store travel time loss is:
\[
\mathcal{L}_{\text{inter}} = 
\frac{1}{\min(T, \hat{T})}
\sum_{t=1}^{\min(T, \hat{T})} 
\left| \hat{\tau}_t^{(\text{inter})} - \tau_t^{(\text{inter})} \right|
\]

\paragraph{Note.}  
For \(\mathcal{L}_{\text{intra}}\) and \(\mathcal{L}_{\text{inter}}\), if the generated sequence length \(\hat{T}\) does not match the real sequence length \(T\), the losses are only computed up to \(\min(T, \hat{T})\).  
This avoids penalizing valid early stopping (when the generator predicts the end-of-trajectory token earlier) and ensures that sequence misalignment does not dominate the loss.

\subsection{Discriminator Loss}

The discriminator is trained with the standard binary cross-entropy loss:
\[
\mathcal{L}_D = 
- \mathbb{E}_{x \sim \text{real}} \left[ \log D(x) \right]
- \mathbb{E}_{\hat{x} \sim G} \left[ \log \left( 1 - D(\hat{x}) \right) \right]
\]
where:
\[
\begin{aligned}
x &=
\left\{
\mathbf{s}_{j_t}^{\text{embed}},
\tau_t^{(\text{intra})},
\tau_t^{(\text{inter})}
\right\}_{t=1}^T,\\
\hat{x} &=
\left\{
\mathbf{s}_{j_t}^{\text{embed}},
\hat{\tau}_t^{(\text{intra})},
\hat{\tau}_t^{(\text{inter})}
\right\}_{t=1}^{\hat{T}}.
\end{aligned}
\]

Importantly, the discriminator loss does \textbf{not} include the intra and inter-time; those are used exclusively for the generator.

\section{Training Algorithm (Full Details)}
\label{app:training_full}
\noindent This appendix provides pseudocode details complementing the main-text protocol summary.


\begin{algorithm}[t]
\caption{Adversarial Training with Time Loss (Length-Aware Sampling)}
\KwIn{Real trajectories $\mathcal{D}_{\text{real}}$, batch size $B$, learning rates $\eta_G, \eta_D$, Gumbel-Softmax parameters $(\tau_{\text{init}}, \tau_{\min}, \alpha_{\text{anneal}})$}
\KwOut{Trained generator $G_\theta$ and discriminator $D_\phi$}
Initialize $\theta, \phi$, temperature $\tau \leftarrow \tau_{\text{init}}$ \;
\While{not converged}{
    \tcp{--- Length-aware sampling of real trajectories ---}
    Sample $\{\mathbf{x}^i_{\text{real}}\}_{i=1}^B \sim p_{\text{len}}(\mathcal{D}_{\text{real}})$, 
    where $p_{\text{len}}$ is a length-aware distribution weighting trajectories by their length (see the convergence analysis in Section~\ref{sec:theory}) \;

    \tcp{--- Discriminator update ---}
    Sample latent vectors $\{\mathbf{z}^i\}_{i=1}^B \sim p(\mathbf{z})$ \;
    Generate fake trajectories $\hat{\mathbf{x}}^i \sim G_\theta(\mathbf{z}^i)$ using $\text{GumbelSoftmax}(\tau)$ \;
    \[
    \mathcal{L}_D = -\frac{1}{B} \sum_{i=1}^B 
    \left[ 
    \log D_\phi(\mathbf{x}^i_{\text{real}}) + 
    \log (1 - D_\phi(\hat{\mathbf{x}}^i))
    \right]
    \]
    $\phi \leftarrow \phi - \eta_D \nabla_\phi \mathcal{L}_D$ \;

    \tcp{--- Generator update ---}
    Generate new fake trajectories $\hat{\mathbf{x}}^i \sim G_\theta(\mathbf{z}^i)$ \;
    \[
    \mathcal{L}_{\text{adv}} = -\frac{1}{B} \sum_{i=1}^B \log D_\phi(\hat{\mathbf{x}}^i)
    \]
    
{\footnotesize
\[
\mathcal{L}_{\text{time}}
=\frac{1}{B}\sum_{i=1}^B \frac{1}{T_{\min}^i}\sum_{t=1}^{T_{\min}^i}
\Big(
\big|\hat{\tau}_{t}^{(\text{intra})}-\tau_{t}^{(\text{intra})}\big|
+\big|\hat{\tau}_{t}^{(\text{inter})}-\tau_{t}^{(\text{inter})}\big|
\Big).
\]
}

    \[
    \mathcal{L}_G = \mathcal{L}_{\text{adv}} + \lambda_{\text{time}} \mathcal{L}_{\text{time}}
    \]
    $\theta \leftarrow \theta - \eta_G \nabla_\theta \mathcal{L}_G$ \;

    \tcp{--- Temperature annealing ---}
    $\tau \leftarrow \max(\tau_{\min}, \alpha_{\text{anneal}} \cdot \tau)$ \;
}
\end{algorithm}

\section{Theory (Full Statements and Proofs)}
\label{app:theory_full}
\label{sec:theory_app}
\noindent The main text presents the key bound and intuition; this appendix contains full statements and proofs.

We give \emph{distribution-level} guarantees for the \textbf{derived variables} we ultimately report:
\noindent
\(\mathrm{Tot}(x)=\sum_{t=1}^{T}\tau_t^{(\mathrm{intra})}+\sum_{t=1}^{T-1}\tau_t^{(\mathrm{inter})}\),
\(\mathrm{Avg}(x)=\frac{1}{T}\sum_{t=1}^{T}\tau_t^{(\mathrm{intra})}\),
\(\mathrm{Vis}(x)=T\).

Each \(f\in\{\mathrm{Tot},\mathrm{Avg},\mathrm{Vis}\}\) is a deterministic, scalar-valued \emph{post-processing map} from a full trajectory to a summary statistic; it is \emph{not} the model architecture. We will compare the distributions of these derived variables under the data and generator.\par
\smallskip\noindent
\clearpage
A (random) customer trajectory is
\[
x=\{(j_t,\tau_t^{(\mathrm{intra})},\tau_t^{(\mathrm{inter})})\}_{t=1}^{T},
\]
where \(j_t\) is the store at step \(t\), \(\tau_t^{(\mathrm{intra})}\ge 0\) is in–store time, \(\tau_t^{(\mathrm{inter})}\ge 0\) is inter–store time, and \(T\) is the (random) visit length. We let \(T_{\max} \in \mathbb{N}\) denote a fixed upper bound on possible visit lengths. Let \(p_{\mathrm{data}}\) denote the data distribution over trajectories and \(p_G\) the generator distribution.

\paragraph{Training objective (matches implementation).}
The generator loss is
\[
\mathcal{L}_G\;=\;\mathcal{L}_{\mathrm{adv}}\;+\;\lambda_{\mathrm{time}}\!\left(\mathcal{L}_{\mathrm{intra}}+\mathcal{L}_{\mathrm{inter}}\right),
\]
with
\[
\begin{aligned}
\mathcal{L}_{\mathrm{intra}}
&=\mathbb{E}\!\left[\frac{1}{T_{\min}}
\sum_{t=1}^{T_{\min}}
\big|\hat{\tau}^{(\mathrm{intra})}_t-\tau^{(\mathrm{intra})}_t\big|\right],\\
\mathcal{L}_{\mathrm{inter}}
&=\mathbb{E}\!\left[\frac{1}{T_{\min}}
\sum_{t=1}^{T_{\min}}
\big|\hat{\tau}^{(\mathrm{inter})}_t-\tau^{(\mathrm{inter})}_t\big|\right].
\end{aligned}
\]

where \(T_{\min}=\min(T,\hat{T})\). The adversarial term is the standard generator BCE against the discriminator. In practice, we also use \emph{length-aware sampling} (LAS), which buckets sequences by length (defined precisely below).

\paragraph{Standing assumptions.}
(i) \(T\le T_{\max}\) almost surely;\ \ (ii) per-step contribution is bounded: \(0\le \tau_t^{(\mathrm{intra})}+\tau_t^{(\mathrm{inter})}\le B\);\ \ 
(iii) after training, the losses are controlled:
\[
\mathrm{JS}(p_{\mathrm{data}}\Vert p_G)\le \delta,\qquad
\mathcal{L}_{\mathrm{intra}}\le \epsilon_{\mathrm{intra}},\qquad
\mathcal{L}_{\mathrm{inter}}\le \epsilon_{\mathrm{inter}}.
\]
We will use a generic constant \(C_{\mathrm{JS}}\) for the inequality \(\mathrm{TV}(P,Q)\le C_{\mathrm{JS}}\sqrt{\mathrm{JS}(P\Vert Q)}\) (Pinsker-type control).

\subsection{Wasserstein Setup and the Derived-Variable Distributions}

For a given derived variable \(f:\mathcal{X}\to\mathbb{R}\), define the \emph{induced  distributions}
\[
\begin{aligned}
P_f &:= \text{law of } f(x)\ \text{for } x\sim p_{\mathrm{data}},\\
Q_f &:= \text{law of } f(\hat{x})\ \text{for } \hat{x}\sim p_G.
\end{aligned}
\]

We measure distributional closeness via the 1-Wasserstein distance
\[
W_1(P_f,Q_f)
=\sup_{\|g\|_{\mathrm{Lip}}\le 1}
\Big|
\mathbb{E}_{x\sim p_{\mathrm{data}}}\!\big[g(f(x))\big]
-\mathbb{E}_{\hat{x}\sim p_G}\!\big[g(f(\hat{x}))\big]
\Big|.
\]

where the supremum is over 1–Lipschitz \(g:\mathbb{R}\to\mathbb{R}\) (Kantorovich–Rubinstein duality) and
\(\|g\|_{\mathrm{Lip}}:=\sup_{u\neq v}\frac{|g(u)-g(v)|}{|u-v|}\).

\subsection{A Trajectory Semi-Metric and Lipschitz Transfers}
Let \(B>0\) denote a uniform per-step bound on the sum of intra- and inter-trajectory quantities, i.e.,
\[
\tau^{(\mathrm{intra})}_t + \tau^{(\mathrm{inter})}_t \;\le\; B,
\quad
\hat{\tau}^{(\mathrm{intra})}_t + \hat{\tau}^{(\mathrm{inter})}_t \;\le\; B,
\]
for all steps \(t\). This bound represents the maximum possible per-step contribution to the derived variables considered below.

Define the trajectory semi-metric
\[
d_{\mathrm{traj}}(x,\hat{x})
:=\sum_{t=1}^{T_{\min}}\!\Big(
\big|\tau^{(\mathrm{intra})}_t-\hat{\tau}^{(\mathrm{intra})}_t\big|
+\big|\tau^{(\mathrm{inter})}_t-\hat{\tau}^{(\mathrm{inter})}_t\big|
\Big)
+ B\,|T-\hat{T}|.
\]

\begin{lemma}[Lipschitz control of derived variables]
\label{lem:lipschitz-w1}
For any trajectories \(x,\hat{x}\),
\[
|\mathrm{Tot}(x)-\mathrm{Tot}(\hat{x})| \;\le\; d_{\mathrm{traj}}(x,\hat{x}),
\]
\noindent Let \(M:=\max(T,\hat{T},1)\).
\[
\big|\mathrm{Avg}(x)-\mathrm{Avg}(\hat{x})\big|
\le \frac{1}{M}\sum_{t=1}^{T_{\min}}
\big|\tau^{(\mathrm{intra})}_t-\hat{\tau}^{(\mathrm{intra})}_t\big|
+ \frac{B}{M}\,|T-\hat{T}|.
\]

\end{lemma}

\begin{proof}
Write $T_{\min}:=\min\{T,\hat{T}\}$ and denote the stepwise differences
$\Delta^{(\mathrm{intra})}_t:=\tau^{(\mathrm{intra})}_t-\hat{\tau}^{(\mathrm{intra})}_t$ and
$\Delta^{(\mathrm{inter})}_t:=\tau^{(\mathrm{inter})}_t-\hat{\tau}^{(\mathrm{inter})}_t$ for $t\le T_{\min}$.
We also use the shorthand $(a)_+:=\max\{a,0\}$ so that
$T-T_{\min}=(T-\hat{T})_+$ and $\hat{T}-T_{\min}=(\hat{T}-T)_+$.

\medskip
\noindent\textbf{(i) The case $f=\mathrm{Tot}$.}
By definition,
\[
\begin{aligned}
\mathrm{Tot}(x)
&=\sum_{t=1}^{T}\big(\tau^{(\mathrm{intra})}_t+\tau^{(\mathrm{inter})}_t\big),\\
\mathrm{Tot}(\hat{x})
&=\sum_{t=1}^{\hat{T}}\big(\hat{\tau}^{(\mathrm{intra})}_t+\hat{\tau}^{(\mathrm{inter})}_t\big).
\end{aligned}
\]

Then
\[
\begin{aligned}
\mathrm{Tot}(x)-\mathrm{Tot}(\hat{x})
&=\sum_{t=1}^{T_{\min}}\!\big(\Delta^{(\mathrm{intra})}_t+\Delta^{(\mathrm{inter})}_t\big)
 +\sum_{t=T_{\min}+1}^{T}\!\big(\tau^{(\mathrm{intra})}_t+\tau^{(\mathrm{inter})}_t\big) \\
&\quad -\sum_{t=T_{\min}+1}^{\hat{T}}\!\big(\hat{\tau}^{(\mathrm{intra})}_t+\hat{\tau}^{(\mathrm{inter})}_t\big).
\end{aligned}
\]

Taking absolute values and applying the triangle inequality gives
\[
\begin{aligned}
\big|\mathrm{Tot}(x)-\mathrm{Tot}(\hat{x})\big|
&\le \sum_{t=1}^{T_{\min}}\!\Big(|\Delta^{(\mathrm{intra})}_t|+|\Delta^{(\mathrm{inter})}_t|\Big) \\
&\quad + \sum_{t=T_{\min}+1}^{T}\!\big(\tau^{(\mathrm{intra})}_t+\tau^{(\mathrm{inter})}_t\big)
 + \sum_{t=T_{\min}+1}^{\hat{T}}\!\big(\hat{\tau}^{(\mathrm{intra})}_t+\hat{\tau}^{(\mathrm{inter})}_t\big).
\end{aligned}
\]

By the standing per-step bound, each tail term is at most $B$. Therefore,
\[
\sum_{t=T_{\min}+1}^{T}\big(\tau^{(\mathrm{intra})}_t+\tau^{(\mathrm{inter})}_t\big)
\le B\,(T-T_{\min})=B\,(T-\hat{T})_+,
\]
\[
\sum_{t=T_{\min}+1}^{\hat{T}}
\big(\hat{\tau}^{(\mathrm{intra})}_t+\hat{\tau}^{(\mathrm{inter})}_t\big)
\le B\,(\hat{T}-T_{\min})=B\,(\hat{T}-T)_+.
\]

Adding the two tails yields $B\,(T-\hat{T})_+ + B\,(\hat{T}-T)_+ = B\,|T-\hat{T}|$. Thus
\[
\big|\mathrm{Tot}(x)-\mathrm{Tot}(\hat{x})\big|
\le \sum_{t=1}^{T_{\min}}\!\Big(
\big|\tau^{(\mathrm{intra})}_t-\hat{\tau}^{(\mathrm{intra})}_t\big|
+\big|\tau^{(\mathrm{inter})}_t-\hat{\tau}^{(\mathrm{inter})}_t\big|
\Big)+B|T-\hat{T}|
= d_{\mathrm{traj}}(x,\hat{x}).
\]

\medskip
\noindent\textbf{(ii) The case $f=\mathrm{Avg}$.}
Recall
\[
\mathrm{Avg}(x)=\frac{1}{T}\sum_{t=1}^{T}\tau^{(\mathrm{intra})}_t,\qquad
\mathrm{Avg}(\hat{x})=\frac{1}{\hat{T}}\sum_{t=1}^{\hat{T}}\hat{\tau}^{(\mathrm{intra})}_t.
\]
Let $\overline{T}:=\max\{T,\hat{T},1\}$. Add and subtract the same ``matched-length'' terms to align denominators:\par
\smallskip\noindent
\begin{align*}
\mathrm{Avg}(x)-\mathrm{Avg}(\hat{x})
&= \frac{1}{T}\sum_{t=1}^{T}\tau_t^{(\mathrm{intra})}
   \;-\; \frac{1}{\hat{T}}\sum_{t=1}^{\hat{T}}\hat{\tau}_t^{(\mathrm{intra})} \\[2pt]
&= \frac{1}{\overline{T}}\Bigg(\sum_{t=1}^{T}\tau_t^{(\mathrm{intra})}-\sum_{t=1}^{\hat{T}}\hat{\tau}_t^{(\mathrm{intra})}\Bigg)
 + \Big(\tfrac{1}{T}-\tfrac{1}{\overline{T}}\Big)\sum_{t=1}^{T}\tau_t^{(\mathrm{intra})}
 - \Big(\tfrac{1}{\hat{T}}-\tfrac{1}{\overline{T}}\Big)\sum_{t=1}^{\hat{T}}\hat{\tau}_t^{(\mathrm{intra})} \\[2pt]
&= \Big(\tfrac{1}{T}-\tfrac{1}{\overline{T}}\Big)\sum_{t=1}^{T}\tau_t^{(\mathrm{intra})}
 + \frac{1}{\overline{T}}\sum_{t=1}^{T_{\min}}\!\big(\tau_t^{(\mathrm{intra})}-\hat{\tau}_t^{(\mathrm{intra})}\big)
 + \frac{1}{\overline{T}}\sum_{t=T_{\min}+1}^{T}\tau_t^{(\mathrm{intra})} \\
&\quad - \Big(\tfrac{1}{\hat{T}}-\tfrac{1}{\overline{T}}\Big)\sum_{t=1}^{\hat{T}}\hat{\tau}_t^{(\mathrm{intra})}
 - \frac{1}{\overline{T}}\sum_{t=T_{\min}+1}^{\hat{T}}\hat{\tau}_t^{(\mathrm{intra})} \\[2pt]
&= \underbrace{\Big(\tfrac{1}{T}-\tfrac{1}{\overline{T}}\Big)\sum_{t=1}^{T}\tau_t^{(\mathrm{intra})}}_{\text{(A)}}
 + \underbrace{\tfrac{1}{\overline{T}}\sum_{t=1}^{T_{\min}}\!\big(\tau_t^{(\mathrm{intra})}-\hat{\tau}_t^{(\mathrm{intra})}\big)}_{\text{(B)}} \\
&\quad + \underbrace{\tfrac{1}{\overline{T}}\sum_{t=T_{\min}+1}^{T}\tau_t^{(\mathrm{intra})}}_{\text{(C)}}
 - \underbrace{\Big(\tfrac{1}{\hat{T}}-\tfrac{1}{\overline{T}}\Big)\sum_{t=1}^{\hat{T}}\hat{\tau}_t^{(\mathrm{intra})}}_{\text{(D)}}
 - \underbrace{\tfrac{1}{\overline{T}}\sum_{t=T_{\min}+1}^{\hat{T}}\hat{\tau}_t^{(\mathrm{intra})}}_{\text{(E)}}.
\end{align*}

\clearpage

We bound each term:
\begin{itemize}
\item[(B)] Matched steps:
\[
\big|\text{(B)}\big|
\;\le\;
\frac{1}{\overline{T}}\sum_{t=1}^{T_{\min}}\!\big|\tau^{(\mathrm{intra})}_t-\hat{\tau}^{(\mathrm{intra})}_t\big|.
\]

\item[(C)+(E)] Tails: by nonnegativity and the per-step bound,
\[
\begin{aligned}
\big|\text{(C)}\big|+\big|\text{(E)}\big|
&\le \frac{1}{\overline{T}}\!\left(
\sum_{t=T_{\min}+1}^{T}\tau^{(\mathrm{intra})}_t
+\sum_{t=T_{\min}+1}^{\hat{T}}\hat{\tau}^{(\mathrm{intra})}_t\right) \\
&\le \frac{B}{\overline{T}}\big((T-\hat{T})_+ + (\hat{T}-T)_+\big)
= \frac{B}{\overline{T}}\,|T-\hat{T}|.
\end{aligned}
\]

\item[(A)+(D)] 
We treat (A) and (D) symmetrically and work with explicit algebra.  
For $T\ge 1$,
\[
\big|\text{(A)}\big|
=\left|\Big(\tfrac{1}{T}-\tfrac{1}{\overline{T}}\Big)\sum_{t=1}^{T}\tau^{(\mathrm{intra})}_t\right|
\le \Big|\tfrac{1}{T}-\tfrac{1}{\overline{T}}\Big|\,BT
= \frac{B}{\overline{T}}\,|\,\overline{T}-T\,|.
\]

Since $\overline{T}=\max\{T,\hat{T},1\}$ and $T\ge 1$, either $\overline{T}=T$ or $\overline{T}=\hat{T}$. Hence
\[
\begin{aligned}
|\,\overline{T}-T\,|
&=(\hat{T}-T)_+,\\
\text{and therefore}\qquad
\big|\text{(A)}\big|
&\le\frac{B}{\overline{T}}\,(\hat{T}-T)_+.
\end{aligned}
\]

Similarly, for $\hat{T}\ge 1$,
\[
\big|\text{(D)}\big|
=\left|\Big(\tfrac{1}{\hat{T}}-\tfrac{1}{\overline{T}}\Big)\sum_{t=1}^{\hat{T}}\hat{\tau}^{(\mathrm{intra})}_t\right|
\le \frac{B}{\overline{T}}\,|\,\overline{T}-\hat{T}\,|
= \frac{B}{\overline{T}}\,(T-\hat{T})_+.
\]

Adding the two gives
\[
\big|\text{(A)}\big|+\big|\text{(D)}\big|
\le \frac{B}{\overline{T}}\big((\hat{T}-T)_+ + (T-\hat{T})_+\big)
= \frac{B}{\overline{T}}\,|T-\hat{T}|.
\]

\end{itemize}

Combining the three parts (B), (C)+(E), and (A)+(D) yields

\[
\big|\mathrm{Avg}(x)-\mathrm{Avg}(\hat{x})\big|
\le \frac{1}{\overline{T}}\sum_{t=1}^{T_{\min}}
\big|\tau^{(\mathrm{intra})}_t-\hat{\tau}^{(\mathrm{intra})}_t\big|
+\frac{2B}{\overline{T}}\,|T-\hat{T}|.
\]

Finally, absorbing constants into $B$ if desired and recalling $\overline{T}=\max\{T,\hat{T},1\}$, we get
\[
\text{Let } M:=\max(T,\hat{T},1).\qquad
\big|\mathrm{Avg}(x)-\mathrm{Avg}(\hat{x})\big|
\le \frac{1}{M}\sum_{t=1}^{T_{\min}}
\big|\tau^{(\mathrm{intra})}_t-\hat{\tau}^{(\mathrm{intra})}_t\big|
+\frac{B}{M}\,|T-\hat{T}|.
\]

\end{proof}

\subsection{From Training Losses to Expected Trajectory Discrepancy}

\begin{lemma}[Matched-step control via L1 losses]
\label{lem:matched}
With \(\mathcal{L}_{\mathrm{intra}}\le \epsilon_{\mathrm{intra}}\) and \(\mathcal{L}_{\mathrm{inter}}\le \epsilon_{\mathrm{inter}}\),
\[
\begin{aligned}
\mathbb{E}\!\left[\sum_{t=1}^{T_{\min}}
|\tau^{(\mathrm{intra})}_t-\hat{\tau}^{(\mathrm{intra})}_t|\right]
&\le T_{\max}\,\epsilon_{\mathrm{intra}},\\
\mathbb{E}\!\left[\sum_{t=1}^{T_{\min}}
|\tau^{(\mathrm{inter})}_t-\hat{\tau}^{(\mathrm{inter})}_t|\right]
&\le T_{\max}\,\epsilon_{\mathrm{inter}}.
\end{aligned}
\]

\end{lemma}

\begin{proof}
Let
\[
S_{\mathrm{intra}} := \sum_{t=1}^{T_{\min}}
|\tau^{(\mathrm{intra})}_t-\hat{\tau}^{(\mathrm{intra})}_t|.
\]
By definition,
\[
\mathcal{L}_{\mathrm{intra}}
=
\mathbb{E}\!\left[\frac{1}{T_{\min}} S_{\mathrm{intra}}\right]
\le \epsilon_{\mathrm{intra}}.
\]
Since $T_{\min}\le T_{\max}$, we have $\frac{1}{T_{\min}}\ge \frac{1}{T_{\max}}$, hence
\[
\frac{1}{T_{\max}}\,S_{\mathrm{intra}}
\le
\frac{1}{T_{\min}}\,S_{\mathrm{intra}}.
\]
Taking expectations gives
\[
\frac{1}{T_{\max}}\mathbb{E}[S_{\mathrm{intra}}]
\le
\mathcal{L}_{\mathrm{intra}}
\le
\epsilon_{\mathrm{intra}},
\]
so $\mathbb{E}[S_{\mathrm{intra}}]\le T_{\max}\epsilon_{\mathrm{intra}}$.
The inter-time bound is identical.
\end{proof}

\begin{lemma}[Length tail controlled by divergence]
\label{lem:tail}
Let $\pi^\star$ be a maximal coupling of $p_{\mathrm{data}}$ and $p_G$. 
Here $p_{\mathrm{data}}(T)$ and $p_G(T)$ denote the marginal distributions over sequence length $T$ under $p_{\mathrm{data}}$ and $p_G$, respectively.
Then
\[
\mathbb{E}_{\pi^\star}\!\left[B|T-\hat{T}|\right]
\le BT_{\max}\,\mathbb{P}_{\pi^\star}(T\neq \hat{T})
= BT_{\max}\,\mathrm{TV}\!\big(p_{\mathrm{data}}(T),p_G(T)\big)
\le BT_{\max}C_{\mathrm{JS}}\sqrt{\delta}.
\]

\end{lemma}

\begin{proof}
Since $0\le T,\hat{T}\le T_{\max}$, we have the pointwise bound
\[
|T-\hat{T}|
\;\le\; T_{\max}\,\mathbf{1}_{\{T\neq \hat{T}\}}.
\]
Multiplying by $B$ and taking expectations under $\pi^\star$ yields
\[
\mathbb{E}_{\pi^\star}[B|T-\hat{T}|]
\le BT_{\max}\,\mathbb{E}_{\pi^\star}\!\big[\mathbb{I}\{T\neq \hat{T}\}\big]
= BT_{\max}\,\mathbb{P}_{\pi^\star}(T\neq \hat{T}).
\]

By the defining property of a maximal coupling,
\[
\mathbb{P}_{\pi^\star}(T\neq \hat{T}) \;=\; \mathrm{TV}\!\big(p_{\mathrm{data}}(T),p_G(T)\big).
\]
Finally, by the Pinsker-type control we assume (with constant $C_{\mathrm{JS}}$),
\[
\mathrm{TV}\!\big(p_{\mathrm{data}}(T),p_G(T)\big)
\le C_{\mathrm{JS}}\sqrt{\mathrm{JS}\!\left(p_{\mathrm{data}}(T)\Vert p_G(T)\right)}
\le C_{\mathrm{JS}}\sqrt{\delta}.
\]

Combining the displays gives the stated bound.
\end{proof}

\subsection{Wasserstein-1 Bounds for the Derived Variables}

\begin{theorem}[Distributional closeness for derived variables]
\label{thm:w1}
Under the standing assumptions in Section~\ref{sec:theory}, for each 
$f \in \{\mathrm{Tot}, \mathrm{Avg}, \mathrm{Vis}\}$ let
$P_f$ and $Q_f$ denote the distributions of $f(x)$ when $x$ is drawn from
$p_{\mathrm{data}}$ and $p_G$, respectively (as in the previous subsection).
Then
\[
W_1(P_f,Q_f)\le
\begin{cases}
T_{\max}\big(\epsilon_{\mathrm{intra}}+\epsilon_{\mathrm{inter}}\big)
+BT_{\max}C_{\mathrm{JS}}\sqrt{\delta}, & f=\mathrm{Tot},\\
\epsilon_{\mathrm{intra}}+BT_{\max}C_{\mathrm{JS}}\sqrt{\delta}, & f=\mathrm{Avg},\\
2T_{\max}\,\mathrm{TV}\!\big(p_{\mathrm{data}}(T),p_G(T)\big), & f=\mathrm{Vis}.
\end{cases}
\]

\end{theorem}

\begin{proof}

\textbf{Case $f=\mathrm{Tot}$.}
We start from the definition of $W_1$ via Kantorovich–Rubinstein duality for 
$(\mathbb{R},|\cdot|)$:
\[
W_1(P_{\mathrm{Tot}},Q_{\mathrm{Tot}})
=\sup_{\|g\|_{\mathrm{Lip}}\le 1}
\Big|\mathbb{E}_{x\sim p_{\mathrm{data}}}\!\big[g(\mathrm{Tot}(x))\big]
-\mathbb{E}_{\hat{x}\sim p_G}\!\big[g(\mathrm{Tot}(\hat{x}))\big]\Big|.
\]

Let $\pi$ be any coupling of $p_{\mathrm{data}}$ and $p_G$.
We can rewrite the difference inside the supremum as
\[
\mathbb{E}_{(x,\hat{x})\sim \pi}
\!\left[g(\mathrm{Tot}(x))-g(\mathrm{Tot}(\hat{x}))\right].
\]
Since $g$ is $1$–Lipschitz on $\mathbb{R}$ and $\mathrm{Tot}$ is $1$–Lipschitz 
with respect to $d_{\mathrm{traj}}$ (Lemma~\ref{lem:lipschitz-w1}), 
the composition $g\circ\mathrm{Tot}$ is also $1$–Lipschitz on the trajectory space.  
Therefore
\[
|g(\mathrm{Tot}(x))-g(\mathrm{Tot}(\hat{x}))| \le d_{\mathrm{traj}}(x,\hat{x}),
\]
and taking expectations gives
\[
W_1(P_{\mathrm{Tot}},Q_{\mathrm{Tot}}) \le \mathbb{E}_{\pi}[d_{\mathrm{traj}}(x,\hat{x})].
\]

We now choose $\pi = \pi^\star$, the matched+tail coupling from 
Lemmas~\ref{lem:matched} and \ref{lem:tail}, and bound the right-hand side directly.
By definition of $d_{\mathrm{traj}}$,
\[
\Delta_{\mathrm{time}}(x,\hat{x})
:=\sum_{t=1}^{T_{\min}}\!\Big(
\big|\tau^{(\mathrm{intra})}_t-\hat{\tau}^{(\mathrm{intra})}_t\big|
+\big|\tau^{(\mathrm{inter})}_t-\hat{\tau}^{(\mathrm{inter})}_t\big|
\Big),
\]
\[
\mathbb{E}_{\pi^\star}[d_{\mathrm{traj}}(x,\hat{x})]
=\mathbb{E}_{\pi^\star}\!\left[\Delta_{\mathrm{time}}(x,\hat{x})\right]
+\mathbb{E}_{\pi^\star}\!\left[B|T-\hat{T}|\right].
\]

For the matched-step terms, Lemma~\ref{lem:matched} ensures that the expected 
per-step intra-store and inter-store differences are bounded by 
$\epsilon_{\mathrm{intra}}$ and $\epsilon_{\mathrm{inter}}$, respectively,
and there are at most $T_{\max}$ matched steps.  
Thus
\[
\begin{aligned}
\mathbb{E}_{\pi^\star}\!\Big[\sum_{t=1}^{T_{\min}}
|\tau^{(\mathrm{intra})}_t-\hat{\tau}^{(\mathrm{intra})}_t|\Big]
&\le T_{\max}\,\epsilon_{\mathrm{intra}},\\
\mathbb{E}_{\pi^\star}\!\Big[\sum_{t=1}^{T_{\min}}
|\tau^{(\mathrm{inter})}_t-\hat{\tau}^{(\mathrm{inter})}_t|\Big]
&\le T_{\max}\,\epsilon_{\mathrm{inter}}.
\end{aligned}
\]

For the tail term, Lemma~\ref{lem:tail} bounds the expected length difference as
\[
\mathbb{E}_{\pi^\star}[\,|T-\hat{T}|\,] \le T_{\max}\,C_{\mathrm{JS}}\sqrt{\delta},
\]
so multiplying by $B$ gives
\[
\mathbb{E}_{\pi^\star}[B\,|T-\hat{T}|] \le B\,T_{\max}\,C_{\mathrm{JS}}\sqrt{\delta}.
\]

Combining these three bounds, we obtain
\[
\mathbb{E}_{\pi^\star}[d_{\mathrm{traj}}(x,\hat{x})]
\le T_{\max}\,(\epsilon_{\mathrm{intra}}+\epsilon_{\mathrm{inter}})
+ B\,T_{\max}\,C_{\mathrm{JS}}\sqrt{\delta}.
\]
Substituting back into the Wasserstein bound yields
\[
W_1(P_{\mathrm{Tot}},Q_{\mathrm{Tot}})
\le T_{\max}\big(\epsilon_{\mathrm{intra}}+\epsilon_{\mathrm{inter}}\big)
+ B\,T_{\max}\,C_{\mathrm{JS}}\sqrt{\delta},
\]
as claimed.

\textbf{Case $f=\mathrm{Avg}$.} 
We now bound $W_1(P_{\mathrm{Avg}},Q_{\mathrm{Avg}})$.  
By Kantorovich–Rubinstein duality for $(\mathbb{R},|\cdot|)$, we can write
\[
W_1(P_{\mathrm{Avg}},Q_{\mathrm{Avg}})
=\sup_{\|g\|_{\mathrm{Lip}}\le 1}\Phi(g)
\]
\[
\Phi(g):=\Big|\mathbb{E}_{x\sim p_{\mathrm{data}}}\!\big[g(\mathrm{Avg}(x))\big]
-\mathbb{E}_{\hat{x}\sim p_G}\!\big[g(\mathrm{Avg}(\hat{x}))\big]\Big|.
\]

For any coupling $\pi$ of $(x,\hat{x})$ with those marginals, the difference inside the supremum becomes
\[
\mathbb{E}_{(x,\hat{x})\sim\pi}\!\left[g(\mathrm{Avg}(x))-g(\mathrm{Avg}(\hat{x}))\right].
\]
Since $g$ is $1$–Lipschitz on $\mathbb{R}$, we have $|g(u)-g(v)| \le |u-v|$.  
Taking absolute values and the supremum over $g$ yields the bound
\[
W_1(P_{\mathrm{Avg}},Q_{\mathrm{Avg}}) \le
\mathbb{E}_{(x,\hat{x})\sim\pi}\!\left[\,|\mathrm{Avg}(x)-\mathrm{Avg}(\hat{x})|\,\right],
\]
valid for any coupling $\pi$.

Next, we use the pointwise Lipschitz bound for $\mathrm{Avg}$ from Lemma~\ref{lem:lipschitz-w1}: for any trajectories 
\[
W_1(P_{\mathrm{Avg}},Q_{\mathrm{Avg}})
=\sup_{\|g\|_{\mathrm{Lip}}\le 1}
\Big|\mathbb{E}_{x\sim p_{\mathrm{data}}}\!\big[g(\mathrm{Avg}(x))\big]
-\mathbb{E}_{\hat{x}\sim p_G}\!\big[g(\mathrm{Avg}(\hat{x}))\big]\Big|.
\]

Choosing the “matched+tail” coupling $\pi^\star$ from Lemmas~\ref{lem:matched} and \ref{lem:tail}, we take expectations under $\pi^\star$ to obtain
\begin{align*}
M &:= \max(T,\hat{T},1),\\
\Delta_{\mathrm{intra}}
&:= \sum_{t=1}^{T_{\min}}
\big|\tau^{(\mathrm{intra})}_t-\hat{\tau}^{(\mathrm{intra})}_t\big|,\\[2pt]
W_1(P_{\mathrm{Avg}},Q_{\mathrm{Avg}})
&\le
\mathbb{E}_{\pi^\star}\!\left[\frac{1}{M}\,\Delta_{\mathrm{intra}}\right]
+\mathbb{E}_{\pi^\star}\!\left[\frac{B}{M}\,|T-\hat{T}|\right].
\end{align*}

Under $\pi^\star$, the steps $t=1,\dots,T_{\min}$ are perfectly matched.  
By the definition of $\epsilon_{\mathrm{intra}}$ and Lemma~\ref{lem:matched}, the first expectation is at most $\epsilon_{\mathrm{intra}}$:
\[
\mathbb{E}_{\pi^\star}\!\left[
\frac{1}{\max(T,\hat{T},1)}
\sum_{t=1}^{T_{\min}}\!\big|\tau^{(\mathrm{intra})}_t-\hat{\tau}^{(\mathrm{intra})}_t\big|
\right] \le \epsilon_{\mathrm{intra}}.
\]
For the second term, since $\max(T,\hat{T},1) \ge 1$, we have
\[
\mathbb{E}_{\pi^\star}\!\left[
\frac{B}{\max(T,\hat{T},1)}\,|T-\hat{T}|
\right] \le B\,\mathbb{E}_{\pi^\star}[\,|T-\hat{T}|\,].
\]
By Lemma~\ref{lem:tail} (length tail controlled by divergence),
\[
\mathbb{E}_{\pi^\star}\!\left[
\frac{B}{\max(T,\hat{T},1)}\,|T-\hat{T}|
\right]
\;\le\; B\,T_{\max}\,C_{\mathrm{JS}}\,\sqrt{\delta}.
\]

Combining the two contributions, we conclude that
\[
W_1(P_{\mathrm{Avg}},Q_{\mathrm{Avg}}) \le
\epsilon_{\mathrm{intra}} + B\,T_{\max}\,C_{\mathrm{JS}}\,\sqrt{\delta}.
\]
In words, the $1$–Wasserstein distance between the $\mathrm{Avg}$ distributions is controlled by the average intra-store discrepancy plus a tail-length mismatch term at scale $B\,T_{\max}\sqrt{\delta}$.

\textbf{Case $f=\mathrm{Vis}$.}
Here $f(x)=T$ takes values in the finite set $\{0,1,\dots,T_{\max}\}$.
Let $P := p_{\mathrm{data}}(T)$ and $Q := p_G(T)$ be the two discrete
distributions on $\{0,\dots,T_{\max}\}$ with pmfs $p(j),q(j)$, and define the
\emph{tail CDFs}
\[
\begin{aligned}
\Delta_{\mathrm{intra}}(x,\hat{x})
&:=\sum_{t=1}^{T_{\min}}
\big|\tau^{(\mathrm{intra})}_t-\hat{\tau}^{(\mathrm{intra})}_t\big|,\\
|\mathrm{Avg}(x)-\mathrm{Avg}(\hat{x})|
&\le \frac{1}{M}\,\Delta_{\mathrm{intra}}(x,\hat{x})
+\frac{B}{M}\,|T-\hat{T}|.
\end{aligned}
\]

On the integer line with ground metric $|i-j|$, Kantorovich–Rubinstein duality gives
\[
W_1(P,Q) \;=\; \sup_{\|g\|_{\mathrm{Lip}}\le 1}
\Big|\sum_{j=0}^{T_{\max}} g(j)\,(p(j)-q(j))\Big|.
\]
For functions on $\mathbb{Z}$, define the forward difference
$\Delta g(k):=g(k)-g(k-1)$ (with $g(-1)$ arbitrary). If $\|g\|_{\mathrm{Lip}}\le 1$
then $|\Delta g(k)|\le 1$ for all $k$.

We can rewrite the expectation difference by discrete summation by parts:
\[
\sum_{j=0}^{T_{\max}} g(j)\,(p(j)-q(j))
\;=\; \sum_{k=1}^{T_{\max}} \Delta g(k)\,\big(S_P(k)-S_Q(k)\big).
\]
Hence
\[
W_1(P,Q)
=\sup_{|\Delta g(k)|\le 1}\Big|\sum_{k=1}^{T_{\max}}\Delta g(k)\big(S_P(k)-S_Q(k)\big)\Big|
\le \sum_{k=1}^{T_{\max}}\big|S_P(k)-S_Q(k)\big|.
\]

where the last inequality follows by choosing the signs of $\Delta g(k)$ optimally.

For each $k$, expand the tail difference and use the triangle inequality:
\[
\big|S_P(k)-S_Q(k)\big|
= \Big|\sum_{j=k}^{T_{\max}}\!(p(j)-q(j))\Big|
\;\le\; \sum_{j=k}^{T_{\max}}\! |p(j)-q(j)|.
\]
Summing over $k=1,\dots,T_{\max}$ and swapping the order of summation gives
\[
\sum_{k=1}^{T_{\max}}\big|S_P(k)-S_Q(k)\big|
\le \sum_{k=1}^{T_{\max}}\sum_{j=k}^{T_{\max}}|p(j)-q(j)|
= \sum_{j=1}^{T_{\max}} j\,|p(j)-q(j)|.
\]

Since $j \le T_{\max}$ for every $j$, we have
\[
j\,|p(j)-q(j)| \;\le\; T_{\max}\,|p(j)-q(j)|.
\]
Summing over $j=1,\dots,T_{\max}$ gives
\[
\sum_{j=1}^{T_{\max}} j\,|p(j)-q(j)| \;\le\;
T_{\max} \sum_{j=1}^{T_{\max}} |p(j)-q(j)|.
\]
Recall that for discrete distributions $P$ and $Q$ on $\{0,\dots,T_{\max}\}$,
\[
\mathrm{TV}(P,Q)
:=\max_{A\subseteq\{0,\ldots,T_{\max}\}}|P(A)-Q(A)|
= \sum_{j:\,p(j)>q(j)}(p(j)-q(j))
= \tfrac12\sum_{j=0}^{T_{\max}}|p(j)-q(j)|.
\]

The second equality follows because $\sum_j [p(j)-q(j)] = 0$, so the total positive and total negative differences are equal in magnitude, and the subset $A$ that attains the maximum is $\{j : p(j) > q(j)\}$.
Dropping the nonnegative $j=0$ term in the sum only decreases its value, hence
\[
\sum_{j=1}^{T_{\max}} |p(j)-q(j)|
\;\le\; \sum_{j=0}^{T_{\max}} |p(j)-q(j)|
= 2\,\mathrm{TV}(P,Q).
\]
Combining these gives
\[
\sum_{j=1}^{T_{\max}} j\,|p(j)-q(j)|
\;\le\; 2T_{\max}\,\mathrm{TV}(P,Q).
\]

Putting everything together,
\[
W_1(P,Q) \;\le\; 2T_{\max}\,\mathrm{TV}(P,Q).
\]
Applying this with $P=p_{\mathrm{data}}(T)$ and $Q=p_G(T)$ gives the stated bound.

\end{proof}

\subsection{Effect of Length-Aware Sampling (LAS)}
\label{subsec:las_app}

\paragraph{Definition (LAS).}
Partition the set of possible lengths \(\{0,1,\dots,T_{\max}\}\) into disjoint buckets \(\mathcal{B}_1,\dots,\mathcal{B}_K\).
Let \(w_k:=\mathbb{P}_{p_{\mathrm{data}}}(T\in\mathcal{B}_k)\) and \(\hat{w}_k:=\mathbb{P}_{p_G}(T\in\mathcal{B}_k)\) denote the marginal probabilities under data and generator, respectively.
LAS draws training mini-batches by first sampling a bucket \(k\) with probability \(w_k\) (or an empirical estimate \(\tilde{w}_k \approx w_k\)), then sampling examples within that bucket from both data and generator.
Thus, during training, the discriminator receives a mixture whose \emph{bucket weights} closely match the data histogram.

\paragraph{An IPM/Wasserstein view of LAS.}
Let \(d_{\mathrm{traj}}\) be the trajectory semi-metric defined above.
Define \(K(x)\in\{1,\dots,K\}\) as the bucket index such that \(T(x)\in\mathcal{B}_{K(x)}\).

For each bucket \(k\), let \(\mathcal{X}_k:=\{x:\,T(x)\in\mathcal{B}_k\}\) and let \(d_{\mathrm{traj}}^{(k)}\) denote the restriction of \(d_{\mathrm{traj}}\) to \(\mathcal{X}_k\times \mathcal{X}_k\).
Define the within-bucket Wasserstein-1 distance
\[
\begin{aligned}
W_{1,k}\big(p_{\mathrm{data},k},p_{G,k}\big)
&:= \sup_{\phi_k}\ \Big(\E_{p_{\mathrm{data},k}}[\phi_k]-\E_{p_{G,k}}[\phi_k]\Big), \\
&\text{s.t.}\ \phi_k\in \mathrm{Lip}_1(\mathcal{X}_k).
\end{aligned}
\]

where \(\mathrm{Lip}_1(\mathcal{X}_k)\) denotes 1-Lipschitz functions with respect to \(d_{\mathrm{traj}}^{(k)}\).
We also define the \emph{LAS discrepancy}

\[
W_{\mathrm{LAS}}\big(p_{\mathrm{data}},p_G\big)
\;:=\;
\sum_{k=1}^K w_k\, W_{1,k}\big(p_{\mathrm{data},k},p_{G,k}\big).
\]

\begin{proposition}[LAS-aligned objective equals a weighted within-bucket IPM (matched weights)]
\label{prop:wlas_ipm}
If the generator matches the data bucket weights, i.e., \(\hat w_k=w_k\) for all \(k\),
then
\[
W_{\mathrm{LAS}}\big(p_{\mathrm{data}},p_G\big)
\;=\;
\sup_{\substack{\phi(x)=\phi_{K(x)}(x)\\ \phi_k\in\mathrm{Lip}_1(\mathcal{X}_k)}}
\Big(\E_{p_{\mathrm{data}}}[\phi]-\E_{p_G}[\phi]\Big).
\]
\end{proposition}

\begin{proof}
Write \(p_{\mathrm{data}}=\sum_k w_k p_{\mathrm{data},k}\) and \(p_G=\sum_k w_k p_{G,k}\) under the matched-weight assumption.
For any bucket-separable \(\phi(x)=\phi_{K(x)}(x)\),
\[
\E_{p_{\mathrm{data}}}[\phi]-\E_{p_G}[\phi]
=
\sum_{k=1}^K w_k\Big(\E_{p_{\mathrm{data},k}}[\phi_k]-\E_{p_{G,k}}[\phi_k]\Big).
\]
Taking the supremum over \(\phi\) is equivalent to independently maximizing over each \(\phi_k\in\mathrm{Lip}_1(\mathcal{X}_k)\),
yielding \(\sum_k w_k W_{1,k}(p_{\mathrm{data},k},p_{G,k})=W_{\mathrm{LAS}}(p_{\mathrm{data}},p_G)\).
\end{proof}

\begin{lemma}[Bucket-only (length-only) critics are a null space under LAS]
\label{lem:las_nullspace}
Let \(a:\{1,\dots,K\}\to\mathbb{R}\) and define \(\psi(x):=a(K(x))\).
Then for every bucket \(k\),
\[
\E_{p_{\mathrm{data},k}}[\psi]-\E_{p_{G,k}}[\psi]=0.
\]
Equivalently, adding any bucket-only term \(a\circ K\) to a within-bucket critic does not change any \(W_{1,k}(p_{\mathrm{data},k},p_{G,k})\) and thus does not change \(W_{\mathrm{LAS}}(p_{\mathrm{data}},p_G)\).
\end{lemma}

\begin{proof}
Under \(x\sim p_{\mathrm{data},k}\) or \(x\sim p_{G,k}\), we have \(K(x)=k\) almost surely.
Thus \(\psi(x)=a(k)\) almost surely under both distributions, and the expectation difference is zero.
\end{proof}

\begin{lemma}[Global Wasserstein can be dominated by length-marginal mismatch]
\label{lem:rs_length_lb}
Let \(w,\hat w\) be the bucket weights of \(p_{\mathrm{data}},p_G\).
Then the global Wasserstein-1 distance on trajectories (with cost \(d_{\mathrm{traj}}\)) satisfies
\[
W_1\big(p_{\mathrm{data}},p_G\big)\;\ge\; B\,\mathrm{TV}(w,\hat w).
\]
\end{lemma}

\begin{proof}
For any coupling \(\pi\) of \(p_{\mathrm{data}}\) and \(p_G\), let \((X,\hat X)\sim\pi\).
Since \(d_{\mathrm{traj}}(X,\hat X)\ge B\,|T(X)-T(\hat X)|\ge B\,\mathbf{1}\{K(X)\neq K(\hat X)\}\),
\[
\E_{\pi}[d_{\mathrm{traj}}(X,\hat X)]\;\ge\; B\,\mathbb{P}_{\pi}\big(K(X)\neq K(\hat X)\big).
\]
Minimizing over couplings gives
\[
W_1(p_{\mathrm{data}},p_G)\;\ge\; B\,\inf_{\pi}\mathbb{P}_{\pi}(K(X)\neq K(\hat X)).
\]
The minimum mismatch probability between two discrete distributions equals their total variation distance,
so \(\inf_{\pi}\mathbb{P}_{\pi}(K(X)\neq K(\hat X))=\mathrm{TV}(w,\hat w)\),
which proves the claim.
\end{proof}

\begin{corollary}[Within-bucket matching implies derived-variable distribution matching]
\label{cor:las_derived_match}
Let \(f\in\{\mathrm{Tot},\mathrm{Avg},\mathrm{Vis}\}\).
Then \(f\) is 1-Lipschitz with respect to \(d_{\mathrm{traj}}\) (Lemma~\ref{lem:lipschitz-w1}), and
\[
W_1\big(f_{\#}p_{\mathrm{data}},\, f_{\#}p_G\big)
\le \sum_{k=1}^K w_k\, W_{1,k}\big(p_{\mathrm{data},k},p_{G,k}\big)
+ C_f\,\mathrm{TV}(w,\hat w).
\]

where one may take \(C_{\mathrm{Tot}}=B T_{\max}\), \(C_{\mathrm{Avg}}=B\), and \(C_{\mathrm{Vis}}=T_{\max}\).
\end{corollary}

\begin{proof}
The 1-Lipschitz property implies \(W_1(f_{\#}p_{\mathrm{data},k}, f_{\#}p_{G,k})\le W_{1,k}(p_{\mathrm{data},k},p_{G,k})\) for each \(k\).
A standard mixture bound on \(W_1\) then gives
\[
W_1\!\big(f_{\#}p_{\mathrm{data}},\, f_{\#}p_G\big)
\le \sum_k w_k\, W_1\!\big(f_{\#}p_{\mathrm{data},k},\, f_{\#}p_{G,k}\big)
+ \mathrm{diam}\!\big(f(\mathcal{X})\big)\,\mathrm{TV}(w,\hat w).
\]

and \(\mathrm{diam}(f(\mathcal{X}))\le C_f\) under Assumption~\ref{ass:bounds}.
Combining the inequalities yields the result.
\end{proof}

\paragraph{Consequences and mechanism.}
Lemma~\ref{lem:las_nullspace} formalizes that LAS \emph{projects out} bucket-only (length-only) shortcut features within each update, so the critic must rely on within-bucket structure.
Proposition~\ref{prop:wlas_ipm} shows that, once bucket weights are aligned, LAS corresponds to optimizing a weighted sum of within-bucket Wasserstein/IPM objectives.
Corollary~\ref{cor:las_derived_match} then connects within-bucket matching to the \emph{derived-variable distribution matching} reported in our experiments.
In contrast, Lemma~\ref{lem:rs_length_lb} highlights that the global Wasserstein objective optimized under random sampling can be dominated by bucket-marginal mismatch, encouraging length-driven discrimination rather than improving within-bucket structure.

\section{Experimental Evaluation (Full)}
\label{app:full_eval}
\label{app:exp_full}
\label{sec:experiments_app}
\noindent This appendix complements the main-text experimental protocol with additional plots, dataset details, and ablation results.

\subsection{Additional mall plots}
\label{app:extra_mall_plots}

\noindent Figure~\ref{fig:mall_length_all} provides per-mall trajectory-length (\#visits) overlays under random sampling (RS) and LAS.

\begin{figure*}[t]
\centering
\textbf{RS (top row)}\par\vspace{0.3em}
\begin{subfigure}{0.24\textwidth}\centering
\includegraphics[width=\textwidth]{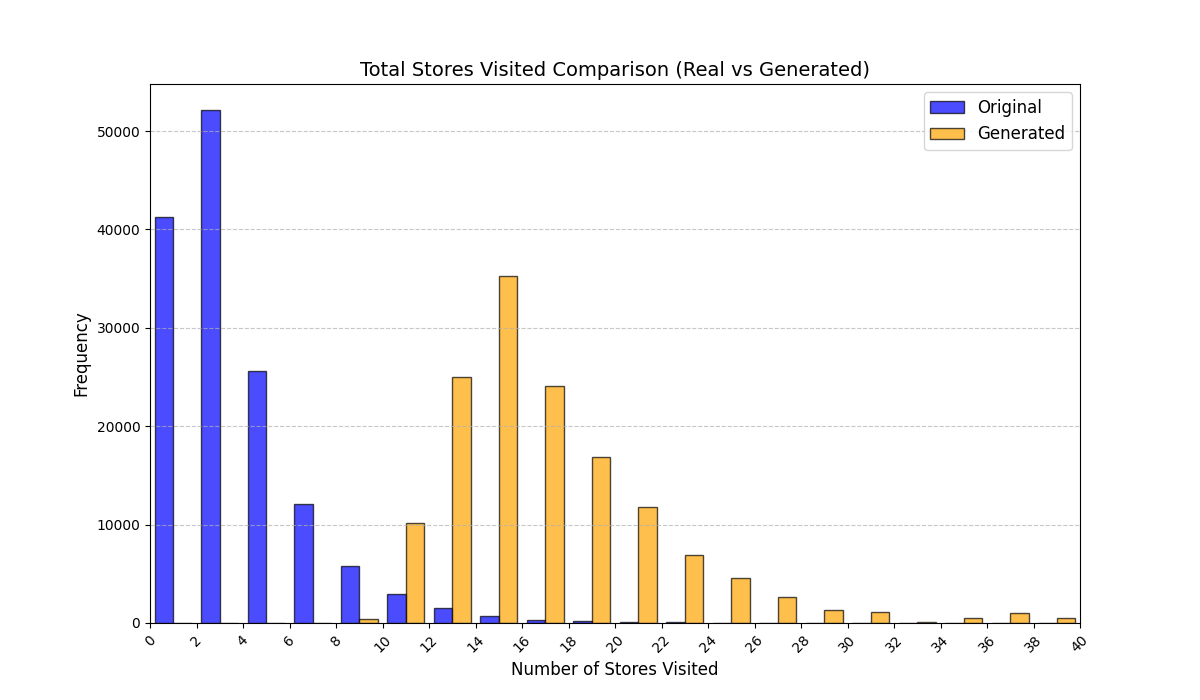}\caption{Mall A}\end{subfigure}
\begin{subfigure}{0.24\textwidth}\centering
\includegraphics[width=\textwidth]{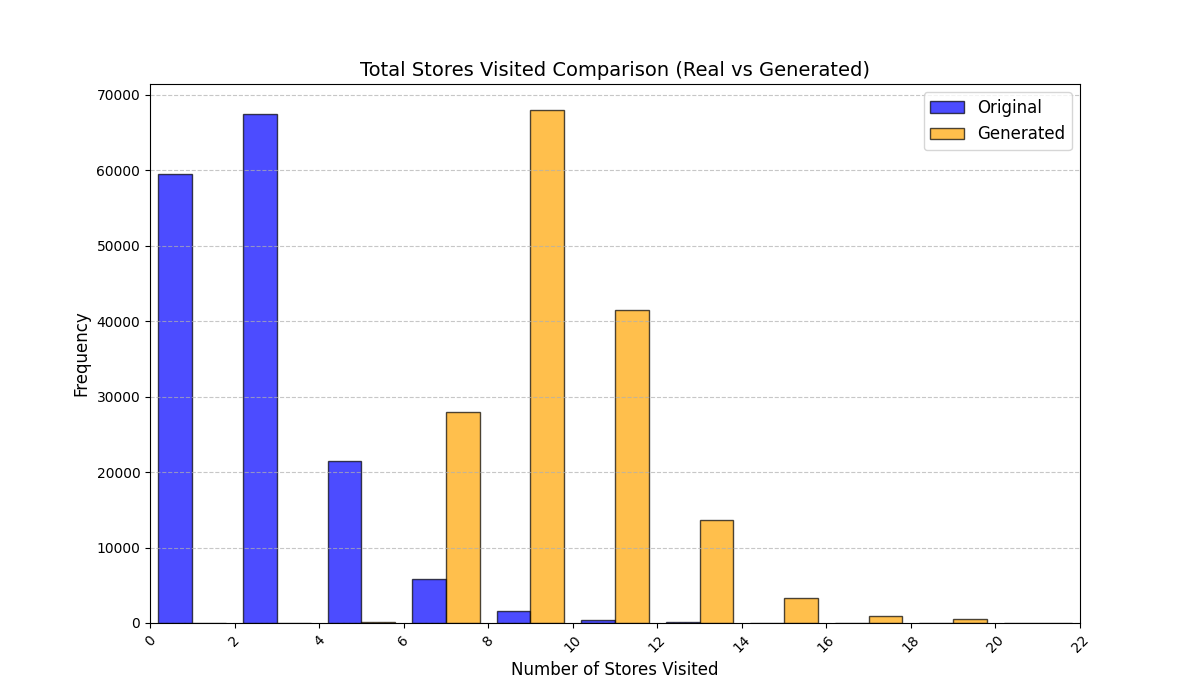}\caption{Mall B}\end{subfigure}
\begin{subfigure}{0.24\textwidth}\centering
\includegraphics[width=\textwidth]{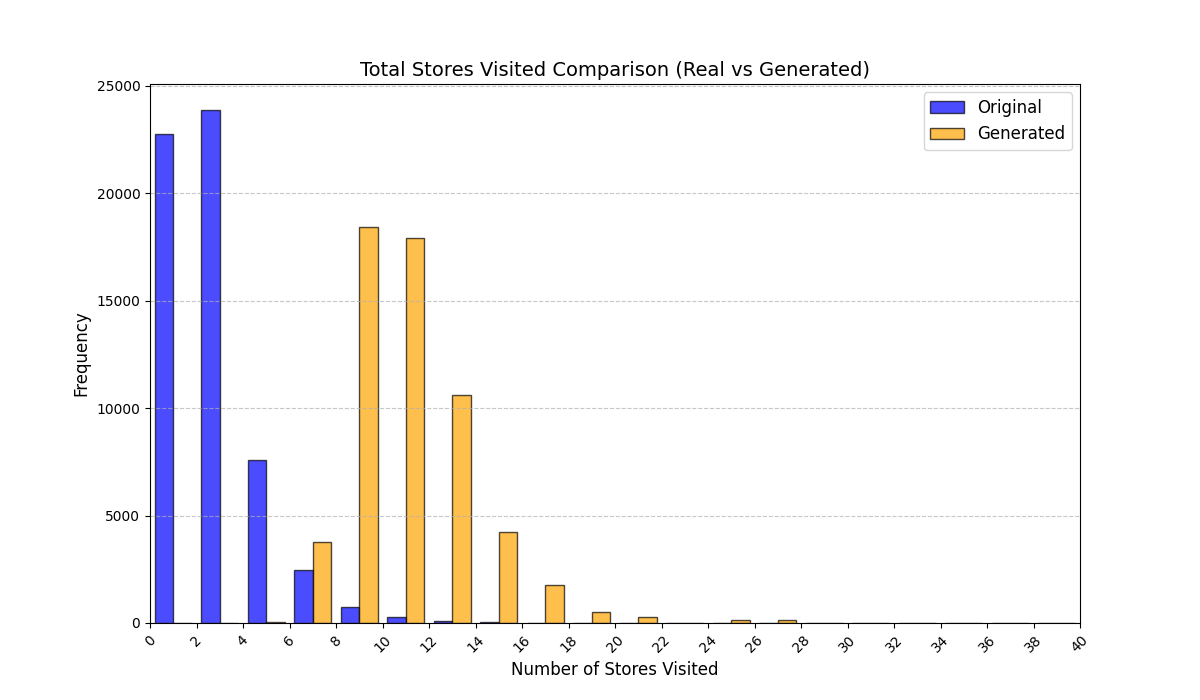}\caption{Mall C}\end{subfigure}
\begin{subfigure}{0.24\textwidth}\centering
\includegraphics[width=\textwidth]{figures/mall_bjfk_num_visits_RS.png}\caption{Mall D}\end{subfigure}

\vspace{0.6em}
\textbf{LAS (bottom row)}\par\vspace{0.3em}
\begin{subfigure}{0.24\textwidth}\centering
\includegraphics[width=\textwidth]{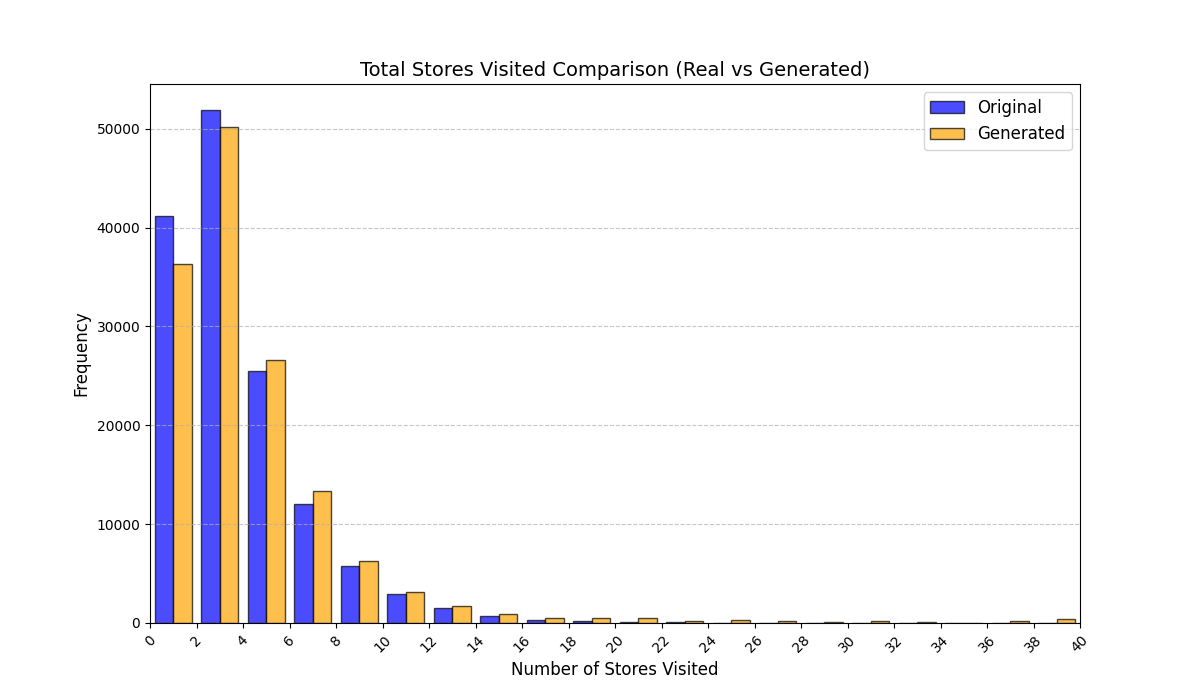}\caption{Mall A}\end{subfigure}
\begin{subfigure}{0.24\textwidth}\centering
\includegraphics[width=\textwidth]{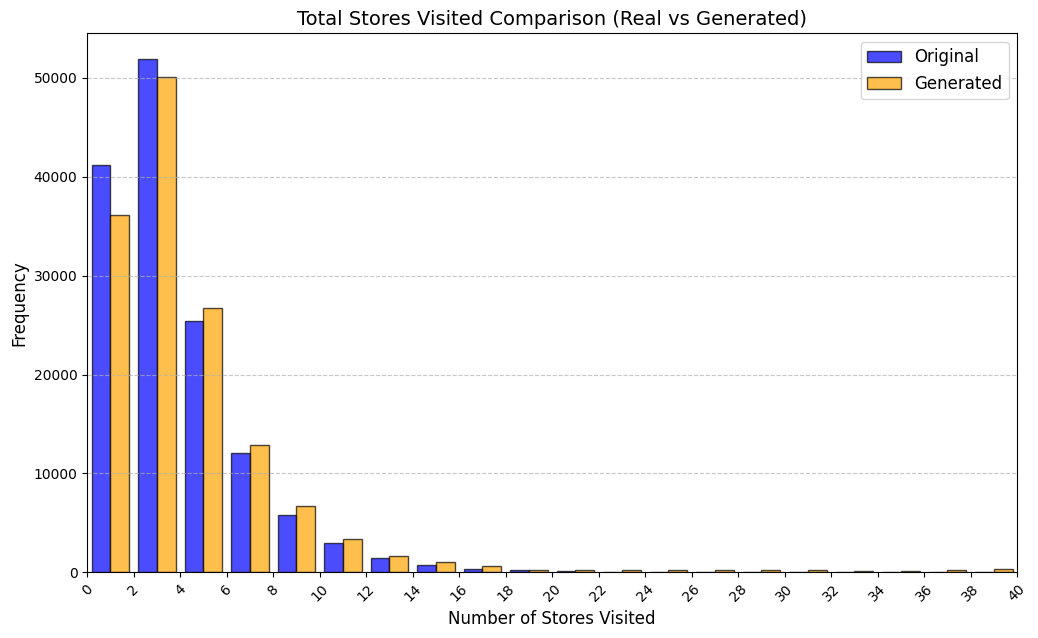}\caption{Mall B}\end{subfigure}
\begin{subfigure}{0.24\textwidth}\centering
\includegraphics[width=\textwidth]{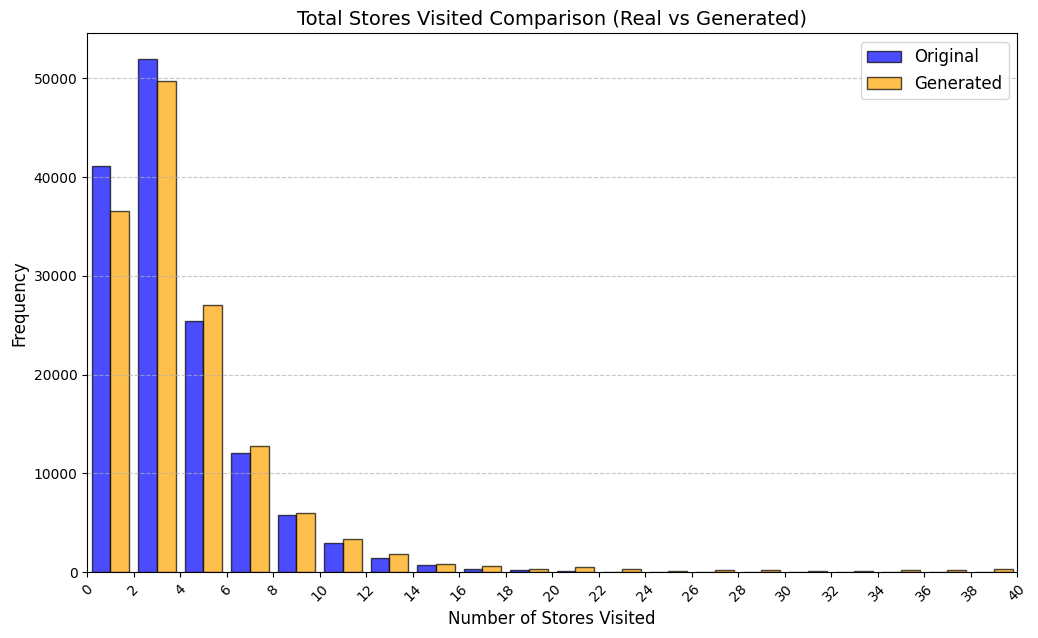}\caption{Mall C}\end{subfigure}
\begin{subfigure}{0.24\textwidth}\centering
\includegraphics[width=\textwidth]{figures/mall_bjfk_num_visits_LAS.png}\caption{Mall D}\end{subfigure}
\caption{Trajectory-length (\#visits) distributions across four malls. LAS matches the ground-truth length marginal substantially better than RS.}
\label{fig:mall_length_all}
\end{figure*}

\paragraph{Data.}
We use anonymized mall visit trajectories on held-out calendar days. Each trajectory
$\pi=\{(j_t,\tau_t^{(\text{intra})},\tau_t^{(\text{inter})})\}_{t=1}^{T}$
records the visited store $j_t$, the intra–store dwell time $\tau_t^{(\text{intra})}$, and the inter–store (walking) time $\tau_t^{(\text{inter})}$ at step $t$.
The corpus covers a single multi-floor mall with
$|\mathcal{S}|=202$ stores across $F=3$ floors and $C=19$ categories, spanning a broad mix of weekdays/weekends and event days.

\paragraph{Train/test split.}
To prevent temporal leakage, we split by \emph{unique days} rather than by individual trajectories.
We use an $80\%$/$20\%$ day-level split with a fixed seed and no overlap between sets.
Unless otherwise noted, all figures compare real vs.\ generated distributions on the held-out test days only.

\paragraph{Model configuration (notation $\rightarrow$ value).}
Model architecture and embedding dimensions are shared across experiments; dataset-specific constants (e.g., the number of stores/floors/categories) are set from each dataset.
For a representative mall, we use:
\begin{center}
\small
\setlength{\tabcolsep}{4pt}
\renewcommand{\arraystretch}{0.95}
\begin{tabular}{l l l}
\toprule
Symbol & Description & Value \\
\midrule
$|\mathcal{S}|$ & number of stores & $202$ \\
$F$ & number of floors & $3$ \\
$C$ & number of store categories & $19$ \\
\midrule
$d_e$ & store embedding dimension & $32$ \\
$h$ & LSTM hidden size & $128$ \\
$z$ & latent dimension (generator) & $16$ \\
$d_{\text{type}}$ & store--type embedding dimension & $16$ \\
$d_{\text{floor}}$ & floor embedding dimension & $8$ \\
\bottomrule
\end{tabular}
\end{center}

\paragraph{Training protocol.}
Training follows the procedure described in the algorithmic section, with the same loss notation and objectives:
the adversarial loss for realism and $\ell_1$ losses for time heads (intra/inter) weighted as in the loss section.
We use Adam optimizers ($\beta_1{=}0.5,\ \beta_2{=}0.999$) with learning rate $10^{-4}$ for both generator and discriminator,
batch size $128$, spectral normalization on linear layers, and Gumbel–Softmax sampling for store selection with an annealed temperature
from $1.5$ down to $0.1$.
Training runs for up to $18$ epochs with early stopping (patience $=3$) based on generator loss.

\paragraph{Evaluation protocol.}
Our evaluation is both \emph{quantitative} and \emph{visual}. For each dataset, we define a set of trajectory-derived variables (e.g., total time, trajectory length/\#visits, intra-/inter-event times, and categorical summaries such as store-type or floor distributions).
We report scalar goodness-of-fit via the Kolmogorov--Smirnov (KS) statistic between the empirical distributions of real and generated trajectories (lower is better), and we additionally overlay the corresponding distributions using shared binning and axis ranges for visualization.
Unless noted otherwise, the reference is the empirical distribution from real trajectories on the held-out test split, and comparisons are made against trajectories generated under the same day-level context and conditioning variables.
The subsequent subsections (Unconditional, Conditional ON/OFF, Swapping by Gate Distance, Swapping by Anchor Distance) apply this protocol under their respective conditions.

\subsection*{Notation and metrics}
A trajectory is $\pi=\{(j_t,\tau_t^{(\text{intra})},\tau_t^{(\text{inter})})\}_{t=1}^{T}$ with visited store $j_t$, intra-store time $\tau_t^{(\text{intra})}$, inter-store (walking) time $\tau_t^{(\text{inter})}$ at step $t$, and $T$ total store visits (trajectory length). We visualize overlays for:
\begin{itemize}
    \item Total time in mall:
    \( M_{\text{total}}=\sum_{t=1}^{T}\tau_t^{(\text{intra})}+\sum_{t=1}^{T}\tau_t^{(\text{inter})} \)
    \item Total intra time:
    \( M_{\text{intra}}^{\text{tot}}=\sum_{t=1}^{T}\tau_t^{(\text{intra})} \)
    \item Total inter time:
    \( M_{\text{inter}}^{\text{tot}}=\sum_{t=1}^{T}\tau_t^{(\text{inter})} \)
    \item Avg.\ intra time per store:
    \( M_{\text{avg-intra}}=\frac{1}{T}\sum_{t=1}^{T}\tau_t^{(\text{intra})} \)
    \item Avg.\ inter time per hop:
    \( M_{\text{avg-inter}}=\frac{1}{\max(T-1,1)}\sum_{t=1}^{T}\tau_t^{(\text{inter})} \)
    \item Trajectory length:
    \( M_{\text{len}}=T \)
\end{itemize}
For category/floor summaries, with $c(j_t)$ the category and $f(j_t)$ the floor of $j_t$, we visualize:
\begin{itemize}
    \item Diversity of categories per trajectory:
    \( M_{\text{div}}=\big|\{c(j_t)\}_{t=1}^{T}\big| \)
    \item Visit counts by category:
    \( N_c=\sum_{t=1}^{T}\mathbf{1}[c(j_t)=c] \)
    \item Intra-store time by category:
    \( T_c=\sum_{t=1}^{T}\mathbf{1}[c(j_t)=c]\cdot \tau_t^{(\text{intra})} \)
    \item Floor-level visit counts:
    \( N_f=\sum_{t=1}^{T}\mathbf{1}[f(j_t)=f] \)
\end{itemize}

\subsection{Unconditional Distribution Matching}
\label{sec:uncond}
We pool all held-out test days—without conditioning on store status—and compare real vs.\ generated trajectories at the population level.

\begin{figure}[htb]
  \centering
  \begin{minipage}[t]{0.48\linewidth}
    \centering
    \includegraphics[width=\linewidth]{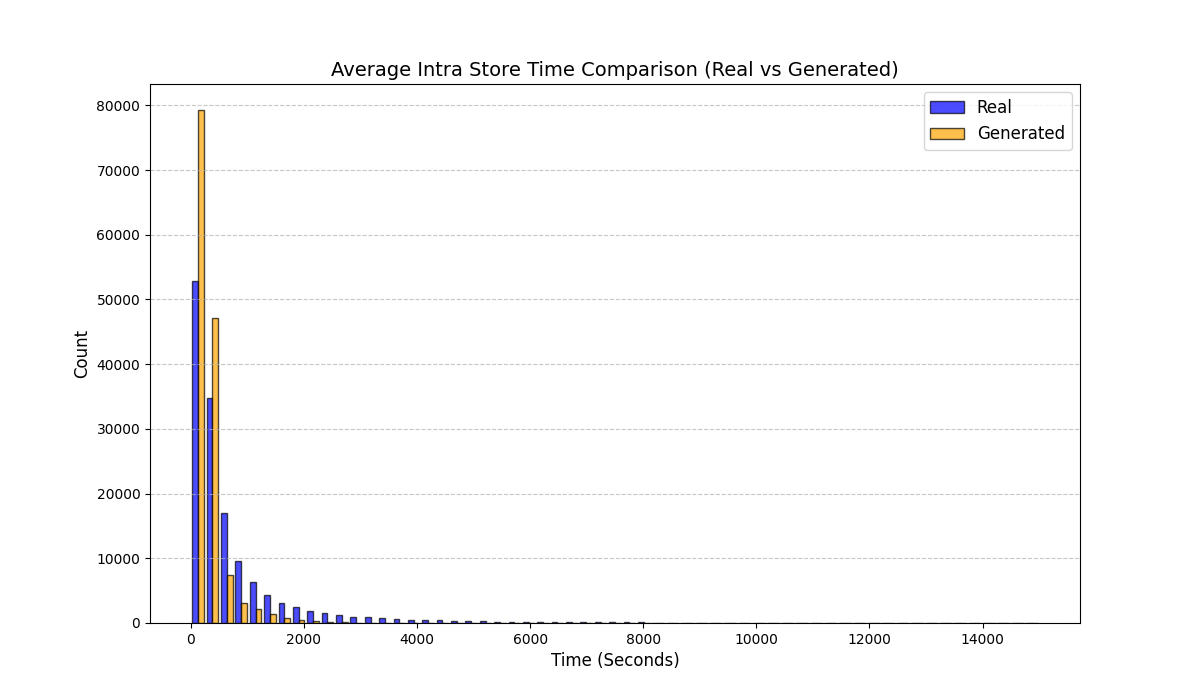}
    \vspace{2pt}
    {\small Average \emph{intra}-store time (real vs.\ generated)}
  \end{minipage}\hfill
  \begin{minipage}[t]{0.48\linewidth}
    \centering
    \includegraphics[width=\linewidth]{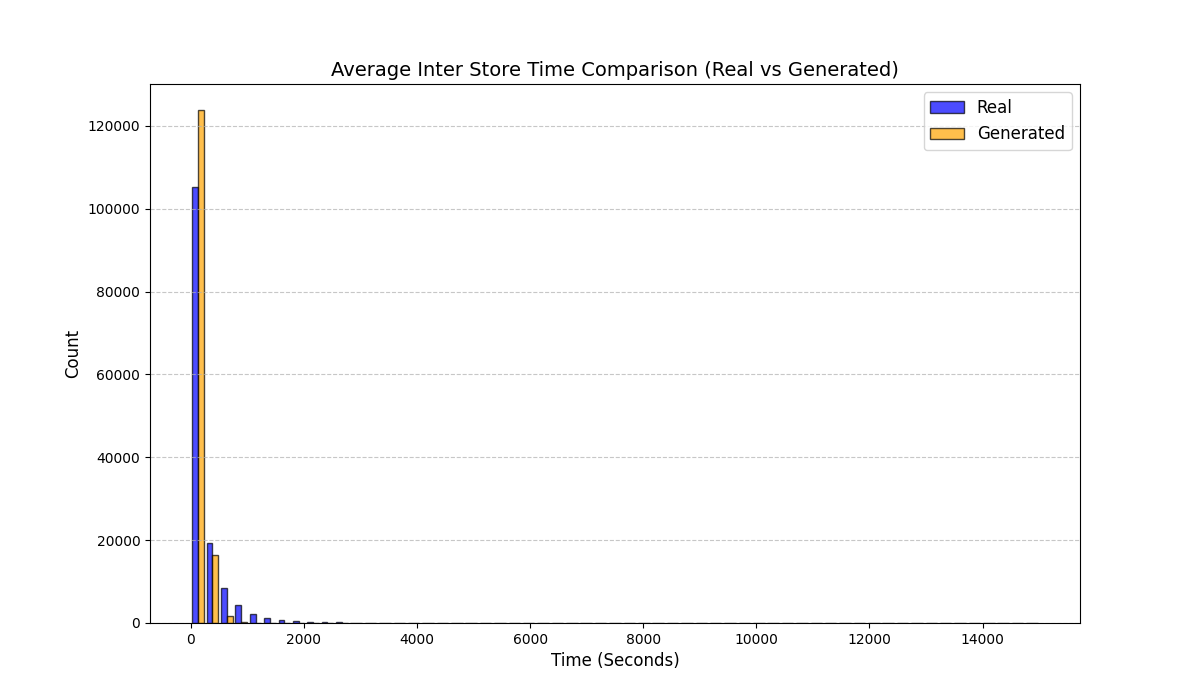}
    \vspace{2pt}
    {\small Average \emph{inter}-store time (real vs.\ generated)}
  \end{minipage}
  \caption{Unconditional overlays for average intra/inter time.}
  \label{fig:uncond_time1}
\end{figure}

\begin{figure}[htb]
  \centering
  \begin{minipage}[t]{0.48\linewidth}
    \centering
    \includegraphics[width=\linewidth]{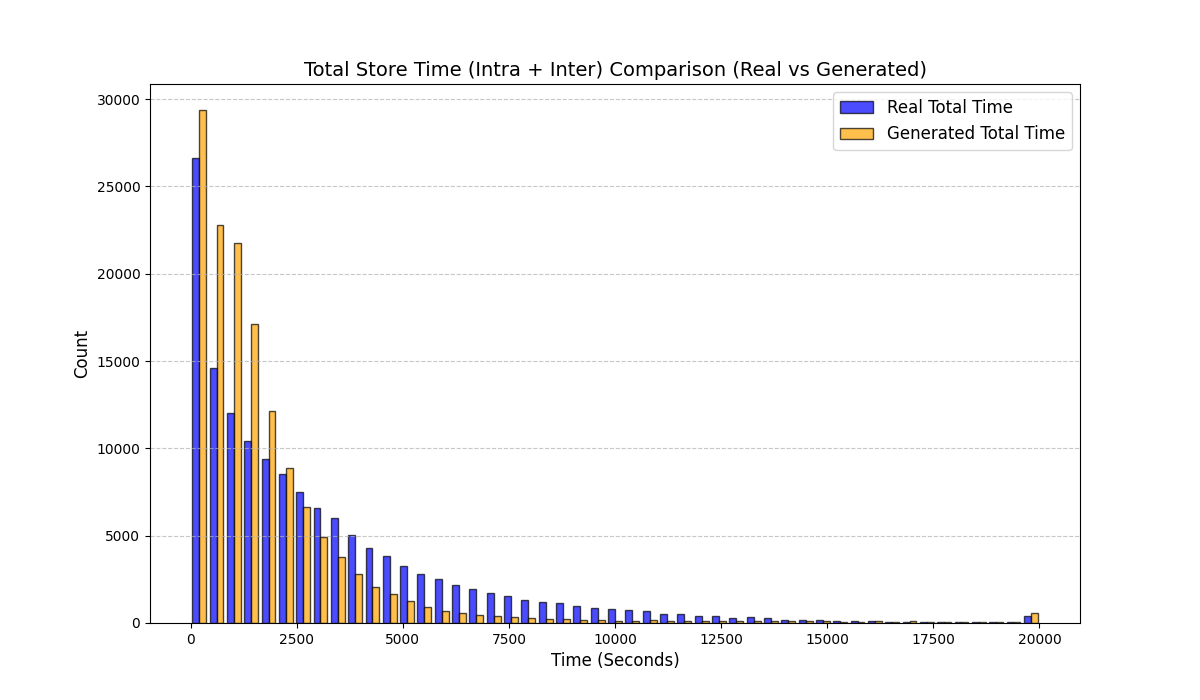}
    \vspace{2pt}
    {\small Total time in mall ($\text{intra}+\text{inter}$)}
  \end{minipage}\hfill
  \begin{minipage}[t]{0.48\linewidth}
    \centering
    \includegraphics[width=\linewidth]{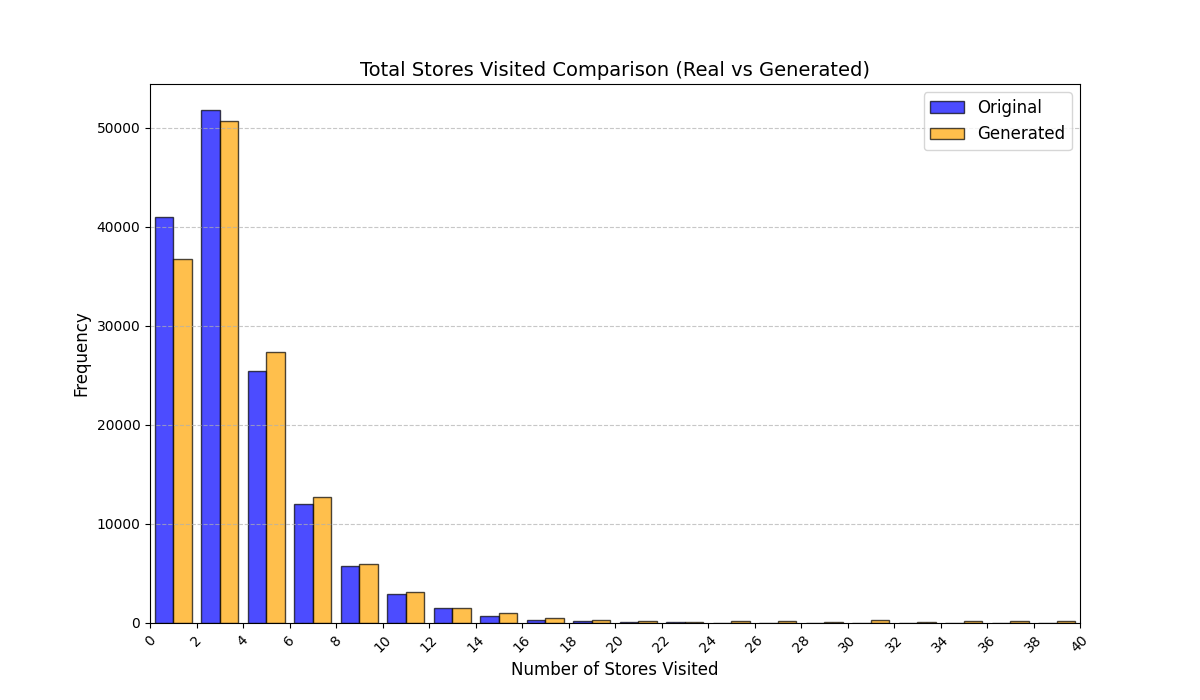}
    \vspace{2pt}
    {\small Trajectory length ($T$)}
  \end{minipage}
  \caption{Unconditional overlays for total time in mall and trajectory length.}
  \label{fig:uncond_time2}
\end{figure}

\begin{figure}[htb]
  \centering
  \begin{minipage}[t]{0.48\linewidth}
    \centering
    \includegraphics[width=\linewidth]{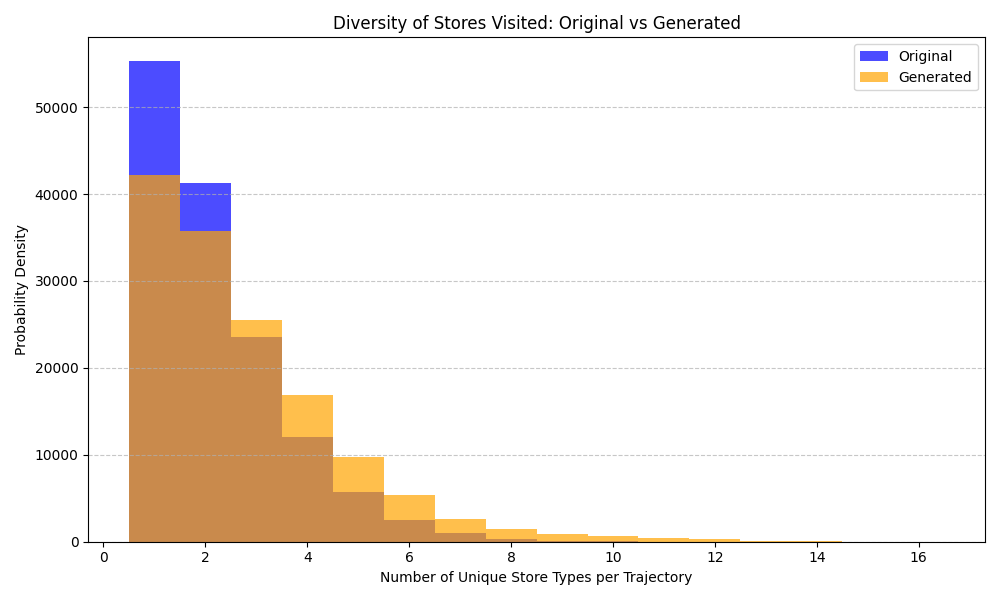}
    \vspace{2pt}
    {\small Diversity: number of unique categories per trajectory}
  \end{minipage}\hfill
  \begin{minipage}[t]{0.48\linewidth}
    \centering
    \includegraphics[width=\linewidth]{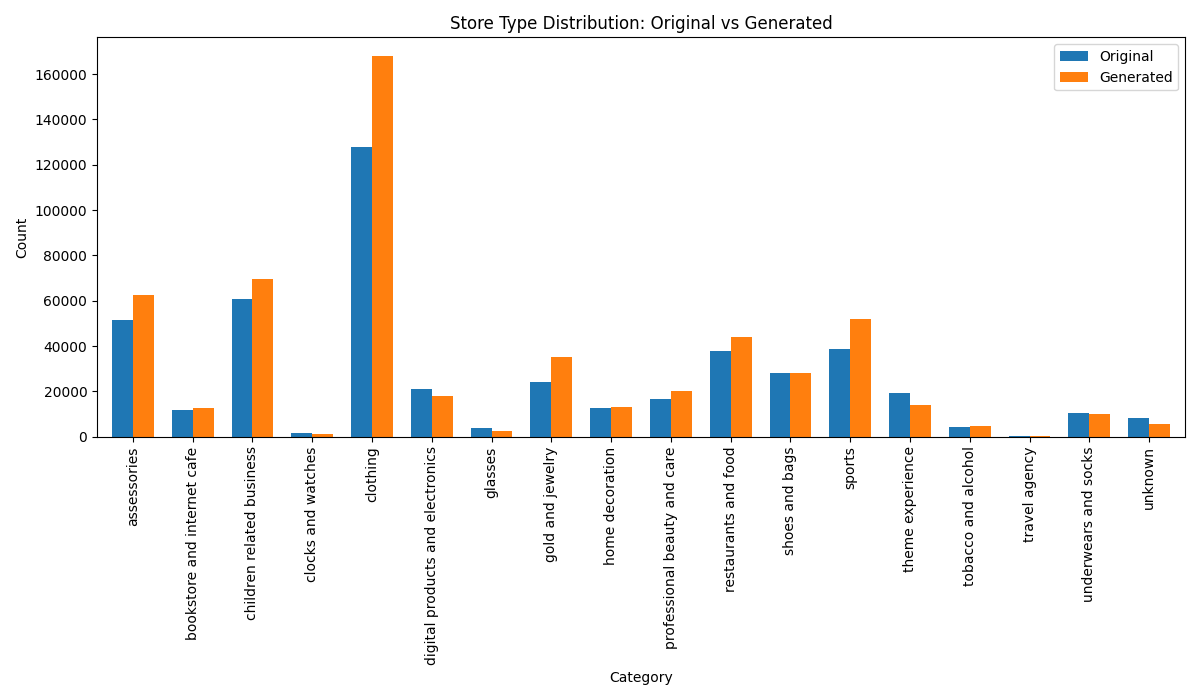}
    \vspace{2pt}
    {\small Store-type visit distribution (counts by category)}
  \end{minipage}
  \caption{Unconditional category/diversity overlays (time-per-category and floor distributions omitted for brevity).}
  \label{fig:uncond_cat}
\end{figure}

\begin{figure}[htb]
  \centering
  \begin{minipage}[t]{0.48\linewidth}
    \centering
    \includegraphics[width=\linewidth]{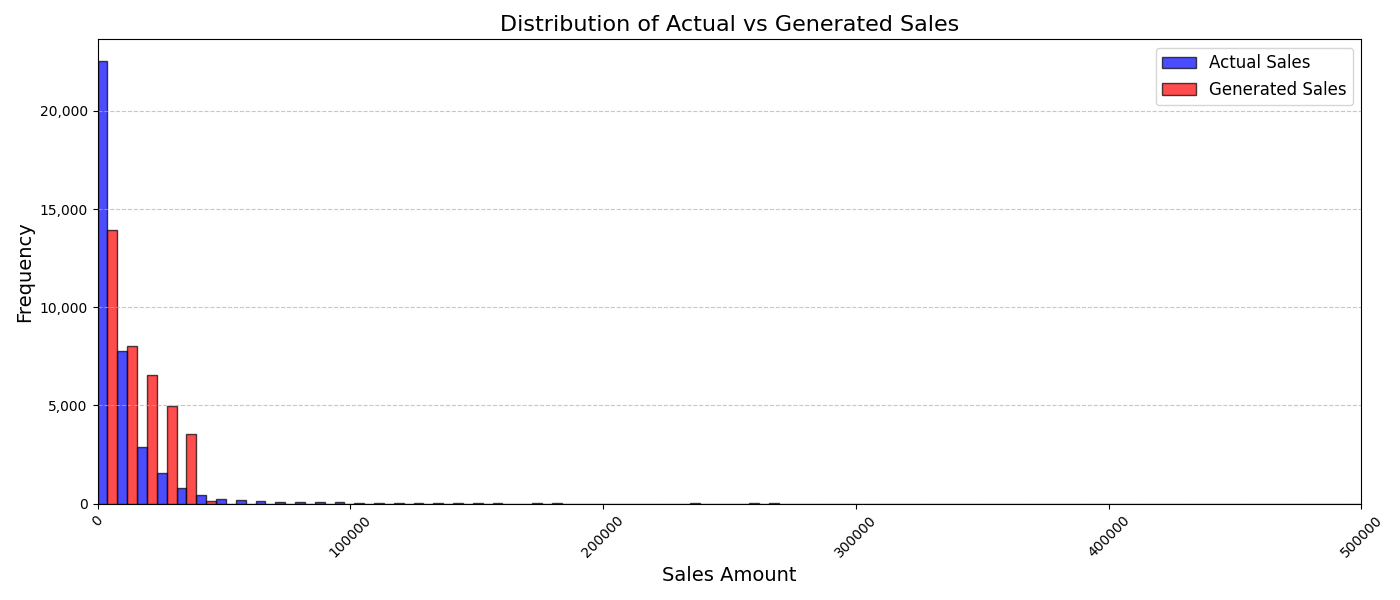}
    \vspace{2pt}
    {\small Store-level sales distribution}
  \end{minipage}\hfill

  \caption{Unconditional overlays for sales marginals (real vs.\ generated).}
  \label{fig:uncond_sales}
\end{figure}

\noindent\emph{Observations.}
Across Figs.~\ref{fig:uncond_time1}–\ref{fig:uncond_time2}, the generator places more mass at shorter dwell times and under-represents the longest tails relative to real trajectories. In Fig.~\ref{fig:uncond_cat}, clothing dominates, with sports and restaurants also prominent; generated trajectories slightly over-index on these high-traffic categories and under-index on smaller experiential types. These patterns indicate that long-stay cohorts (e.g., event days) are harder to reproduce without explicit conditioning, while category shares follow observed traffic but may need rebalancing for niche segments. Sales marginals (Fig.~\ref{fig:uncond_sales}) track the shape of real distributions qualitatively; extremes are less frequent in the generated set.

\subsection{Conditional Store Influence (ON/OFF)}
\label{sec:store-influence}
We study behavioral shifts when a specific store $s^*$ is open (ON) versus closed (OFF). We partition real and generated trajectories by the observed status of $s^*$ and overlay the distributions of the metrics defined above. Representative results for ZARA and MLB are in Figs.~\ref{fig:zara_on_off}–\ref{fig:mlb_on_off}. This analysis is conditional (not counterfactual).

\begin{figure}[htb]
    \centering
    \includegraphics[width=0.45\linewidth]{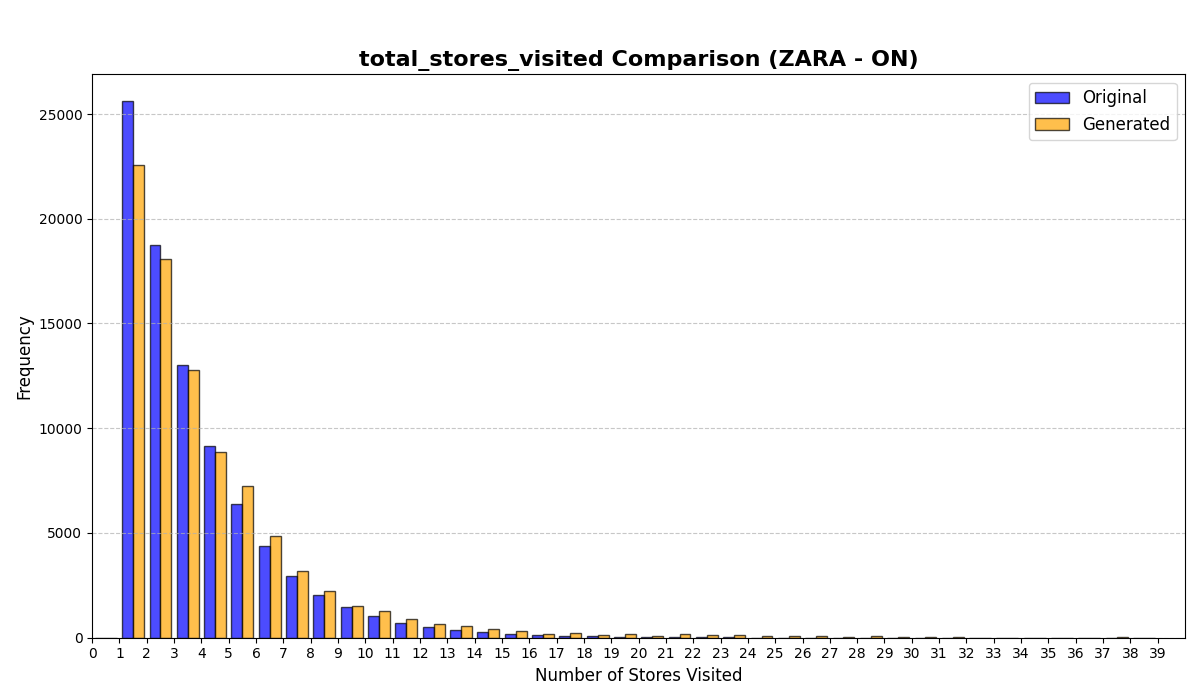}
    \includegraphics[width=0.45\linewidth]{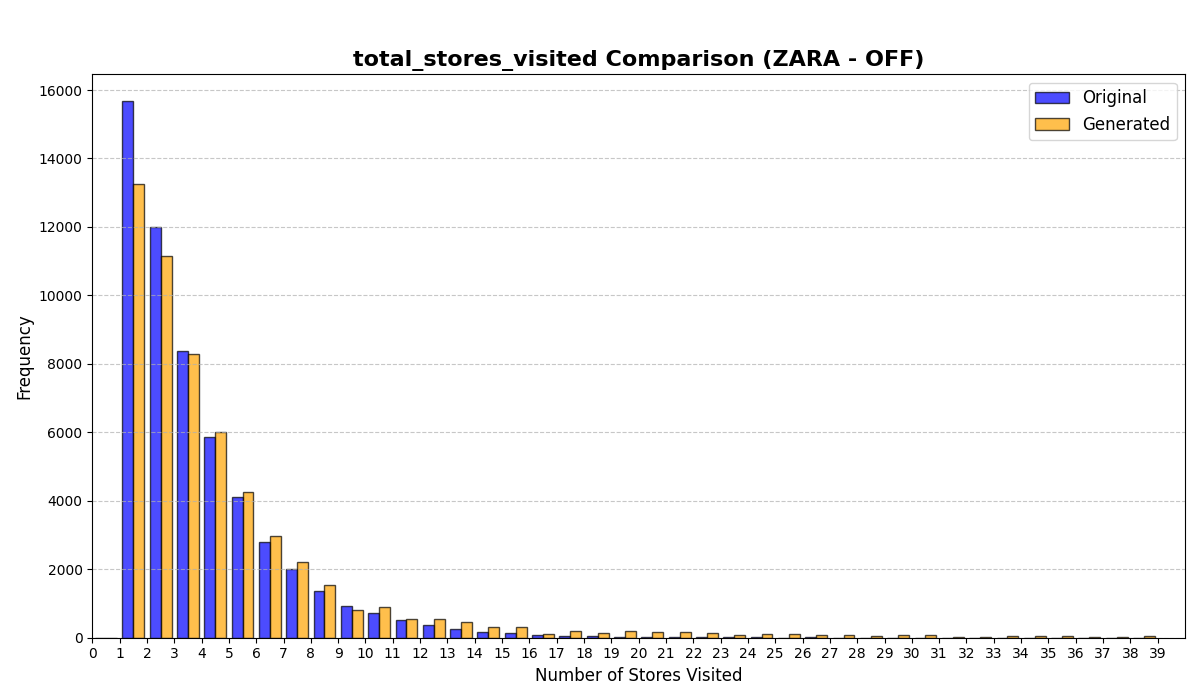}
    \caption{Trajectory length overlays for ZARA under ON (left) and OFF (right). Blue = real, Orange = generated.}
    \label{fig:zara_on_off}
\end{figure}
\begin{figure}[htb]
    \centering
    \includegraphics[width=0.45\linewidth]{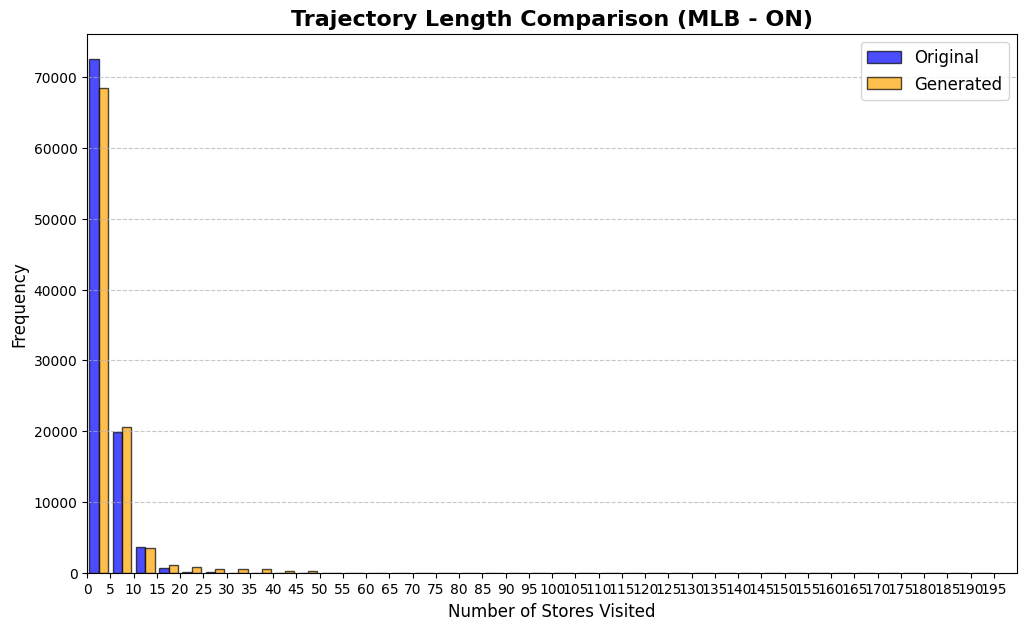}
    \includegraphics[width=0.45\linewidth]{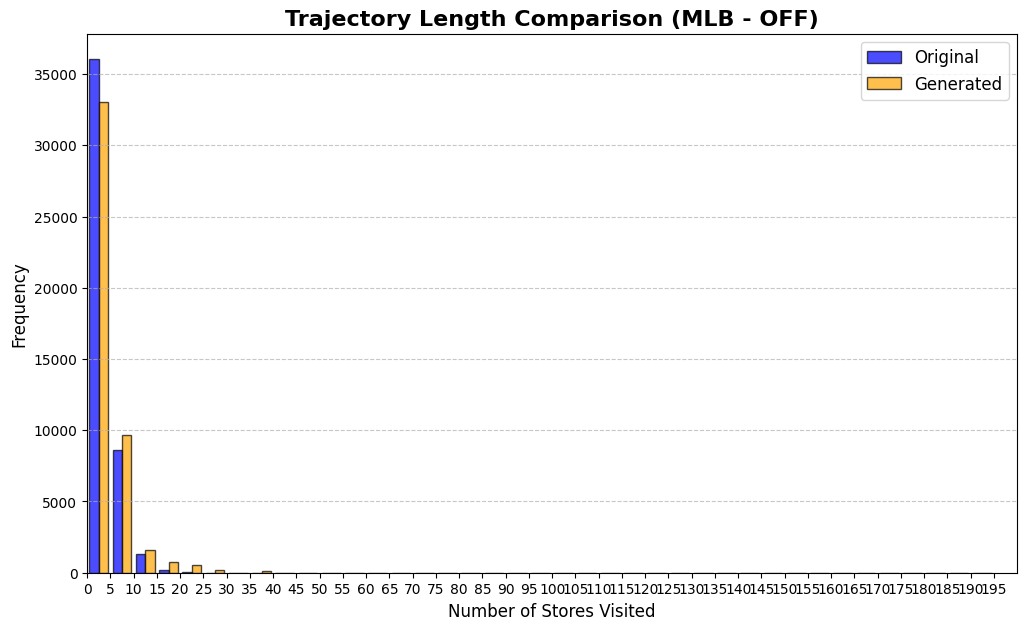}
    \caption{Trajectory length overlays for MLB under ON (left) and OFF (right).}
    \label{fig:mlb_on_off}
\end{figure}

\noindent\emph{Observations.}
ON days exhibit a heavier mid/long tail in trajectory length than OFF days, indicating more multi-store tours when the focal store is available. This suggests co-promotion or cross-windowing with nearby tenants on ON days, while OFF days behave more like quick, targeted trips.

\subsection{Swapping Experiments: Gate Distance}
\label{sec:swap-gate}
We test sensitivity to placement by swapping a target brand (e.g., ZARA) with alternative stores grouped by their distance to the nearest gate. Stores are binned by $(f,h)$: floor $f$ and hop-from-gate group $h$. For each bin we regenerate trajectories under the same day context and visualize the metrics as means with dispersion; comparisons are qualitative.

\begin{figure}[htb]
  \centering
  \includegraphics[width=0.9\linewidth]{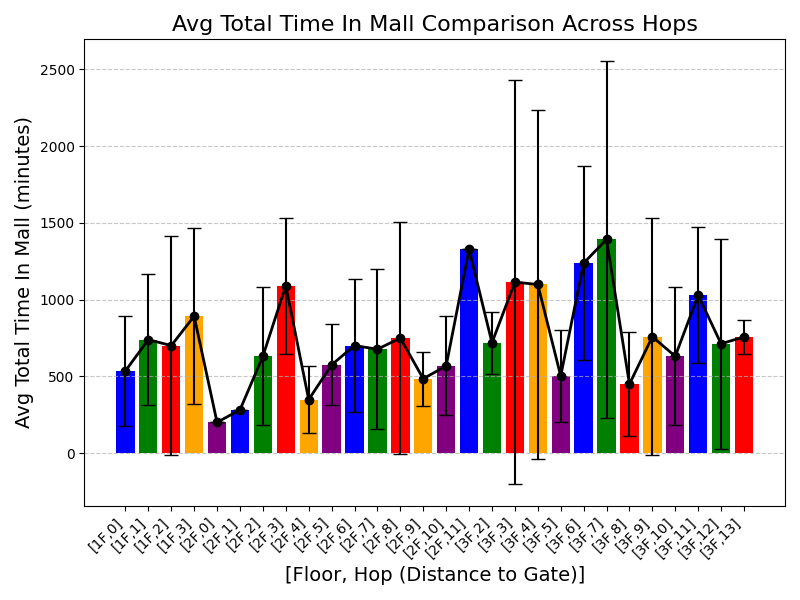}
  \caption{Average total time in mall ($\text{intra}+\text{inter}$) across gate-distance bins $(f,h)$ after swapping. Bars show group means with error bars; the line overlays the trend.}
  \label{fig:swap_gate_total_time}
\end{figure}
\begin{figure}[htb]
  \centering
  \includegraphics[width=0.9\linewidth]{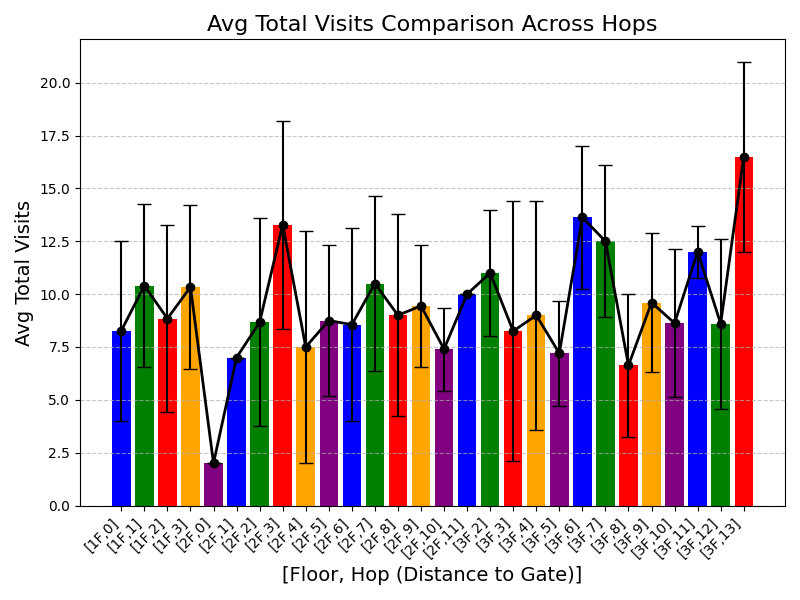}
  \caption{Average total visits ($T$) across gate-distance bins $(f,h)$ after swapping.}
  \label{fig:swap_gate_visits}
\end{figure}

\noindent\emph{Observations.}
Placements a few hops from primary gates tend to exhibit higher mean total time in mall and higher visit counts than gate-adjacent or distant placements (Figs.~\ref{fig:swap_gate_total_time}–\ref{fig:swap_gate_visits}); variance remains substantial and floor effects are visible. For dwell-oriented concepts, positions just beyond the entrances are associated with longer tours.

\subsection{Swapping Experiments: Anchor-Store Distance}
\label{sec:swap-anchor}
We analyze sensitivity to an anchor store $s_c$ (e.g., ZARA) by swapping a target brand (e.g., Uniqlo) with candidates binned by $(f,h)$, where $f$ is the floor and $h$ is the hop distance to $s_c$. For each $(f,h)$ bin we regenerate trajectories under the same day context and visualize qualitative summaries.

\noindent\emph{Observations.}
Bins with small $h$ (roughly $h\in\{1,2,3\}$) show higher average visit counts (Fig.~\ref{fig:swap_anchor_visits}), while the average intra-store time per stop is largely flat across $(f,h)$ (Fig.~\ref{fig:swap_anchor_avg_intra}). Thus, proximity primarily affects circulation rather than dwell; changing dwell per stop likely requires adjustments to in-store experience or messaging rather than small relocations (see Figs.~\ref{fig:swap_anchor_avg_intra}–\ref{fig:swap_anchor_visits}).

\begin{figure}[htb]
  \centering
  \includegraphics[width=0.95\linewidth]{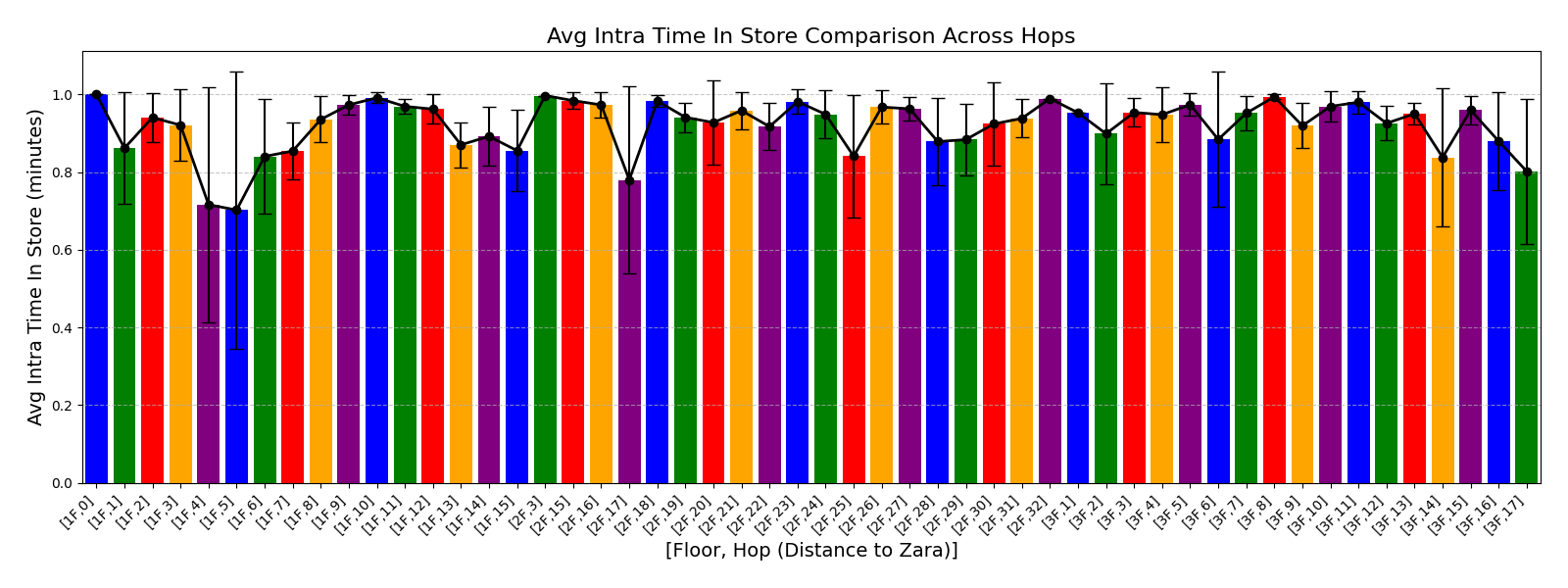}
  \caption{Average intra time per store across \([f,h]\) bins relative to anchor ZARA after swapping Uniqlo.}
  \label{fig:swap_anchor_avg_intra}
\end{figure}
\begin{figure}[htb]
  \centering
  \includegraphics[width=0.95\linewidth]{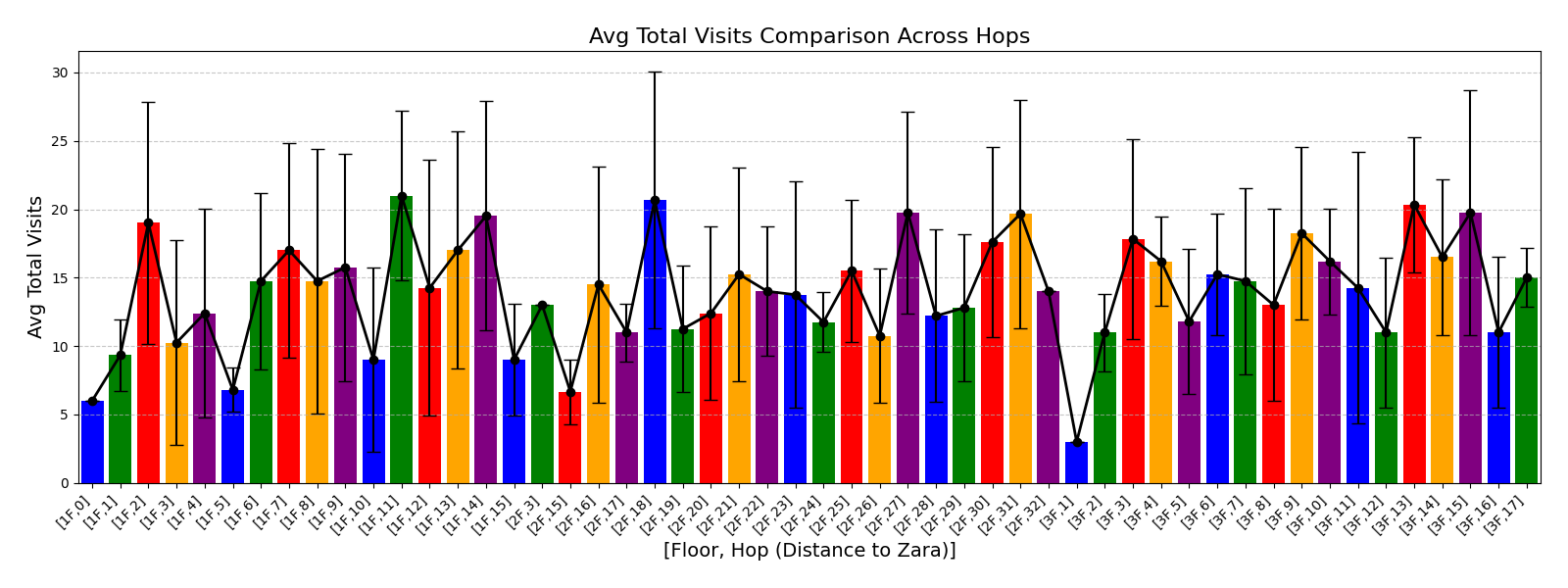}
  \caption{Average total visits (\(T\)) across \([f,h]\) bins relative to anchor ZARA after swapping Uniqlo.}
  \label{fig:swap_anchor_visits}
\end{figure}

\subsection{Four-mall evaluation: full results}
\label{app:full_mall_results}
Table~\ref{tab:mall_full_results} reports the KS statistic for all derived metrics we computed from mall trajectories.
Figure~\ref{fig:mall_total_time_all} visualizes the total-mall-time marginals across all four malls for RS and LAS.

\begin{table*}[t]
\centering
\caption{Full four-mall results: KS statistic for all derived metrics (lower is better).}
\label{tab:mall_full_results}
\setlength{\tabcolsep}{3pt}
\renewcommand{\arraystretch}{1.03}
\scriptsize
\begin{tabular}{lcccccccc}
\toprule
\multirow{2}{*}{Derived metric}
& \multicolumn{2}{c}{Mall A}
& \multicolumn{2}{c}{Mall B}
& \multicolumn{2}{c}{Mall C}
& \multicolumn{2}{c}{Mall D}\\
\cmidrule(lr){2-3}\cmidrule(lr){4-5}\cmidrule(lr){6-7}\cmidrule(lr){8-9}
& RS & LAS & RS & LAS & RS & LAS & RS & LAS\\
\midrule
Average inter-store time & 0.622 & 0.289 & 0.645 & 0.380 & 0.684 & 0.404 & 0.767 & 0.456 \\
Average intra-store time & 0.975 & 0.005 & 0.978 & 0.066 & 0.975 & 0.382 & 0.959 & 0.034 \\
Floor distribution & 1.000 & 0.667 & 1.000 & 0.333 & 1.000 & 0.667 & 0.200 & 0.400 \\
Store diversity & 0.641 & 0.074 & 0.611 & 0.045 & 0.523 & 0.108 & 0.450 & 0.039 \\
Store category mix & 0.278 & 0.333 & 0.506 & 0.287 & 0.467 & 0.333 & 0.477 & 0.303 \\
Time spent per category & 0.419 & 0.432 & 0.468 & 0.383 & 0.394 & 0.367 & 0.425 & 0.379 \\
Total inter-store time & 0.726 & 0.538 & 0.777 & 0.413 & 0.784 & 0.219 & 0.764 & 0.352 \\
Total intra-store time & 0.854 & 0.060 & 0.801 & 0.116 & 0.730 & 0.168 & 0.799 & 0.134 \\
Total time in mall & 0.528 & 0.056 & 0.538 & 0.152 & 0.630 & 0.269 & 0.661 & 0.072 \\
Trajectory length / \#visits & 0.955 & 0.047 & 0.947 & 0.048 & 0.953 & 0.048 & 0.951 & 0.044 \\
\midrule
Mean across metrics & 0.700 & 0.250 & 0.727 & 0.222 & 0.714 & 0.297 & 0.645 & 0.221 \\
\bottomrule
\end{tabular}
\end{table*}

\begin{figure*}[t]
\centering
\textbf{RS}\par\medskip
\begin{subfigure}[t]{0.24\linewidth}
\centering
\includegraphics[width=\linewidth]{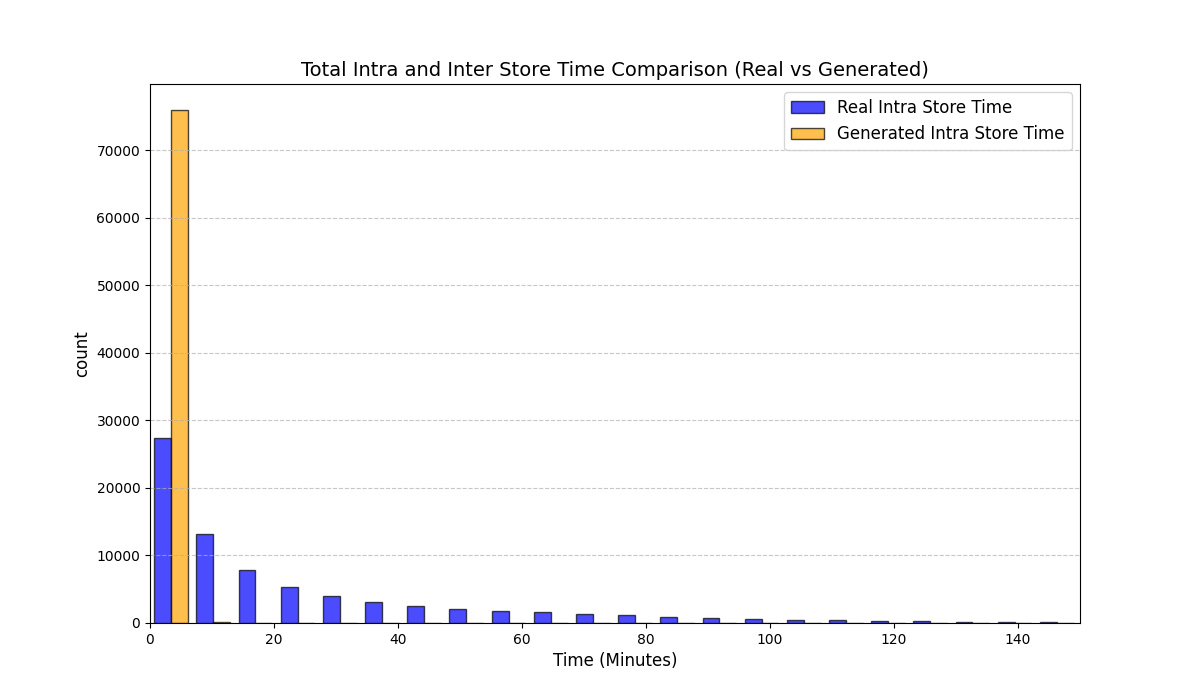}
\caption{Mall A}
\end{subfigure}\hfill
\begin{subfigure}[t]{0.24\linewidth}
\centering
\includegraphics[width=\linewidth]{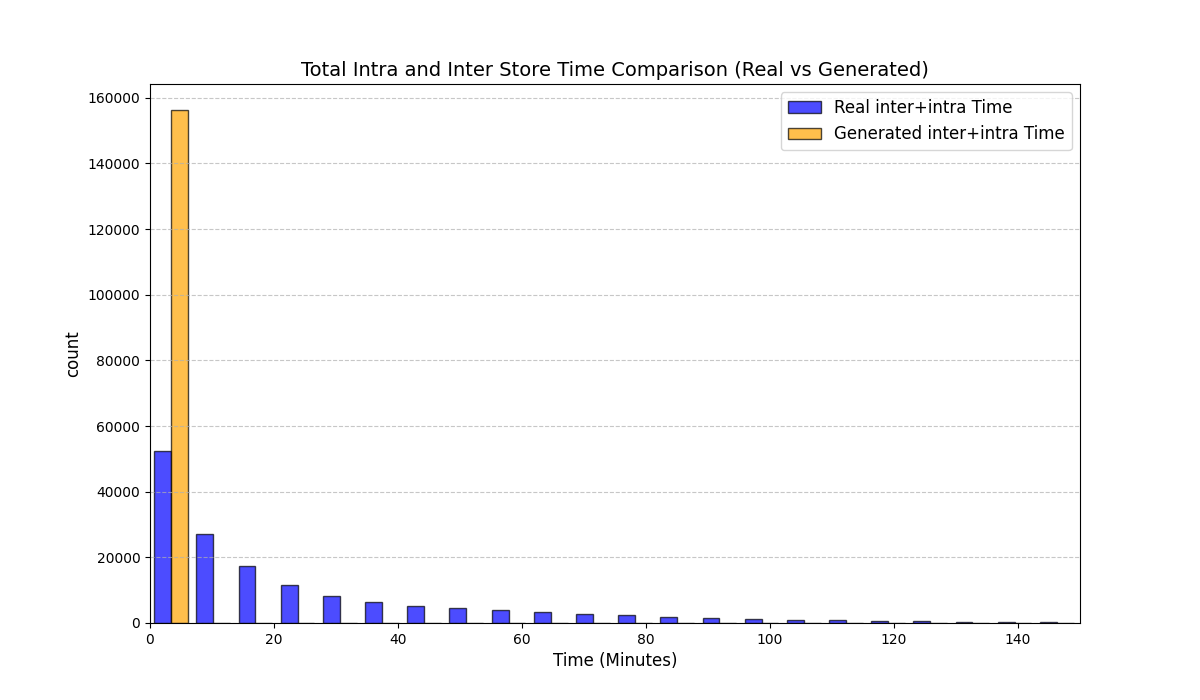}
\caption{Mall B}
\end{subfigure}\hfill
\begin{subfigure}[t]{0.24\linewidth}
\centering
\includegraphics[width=\linewidth]{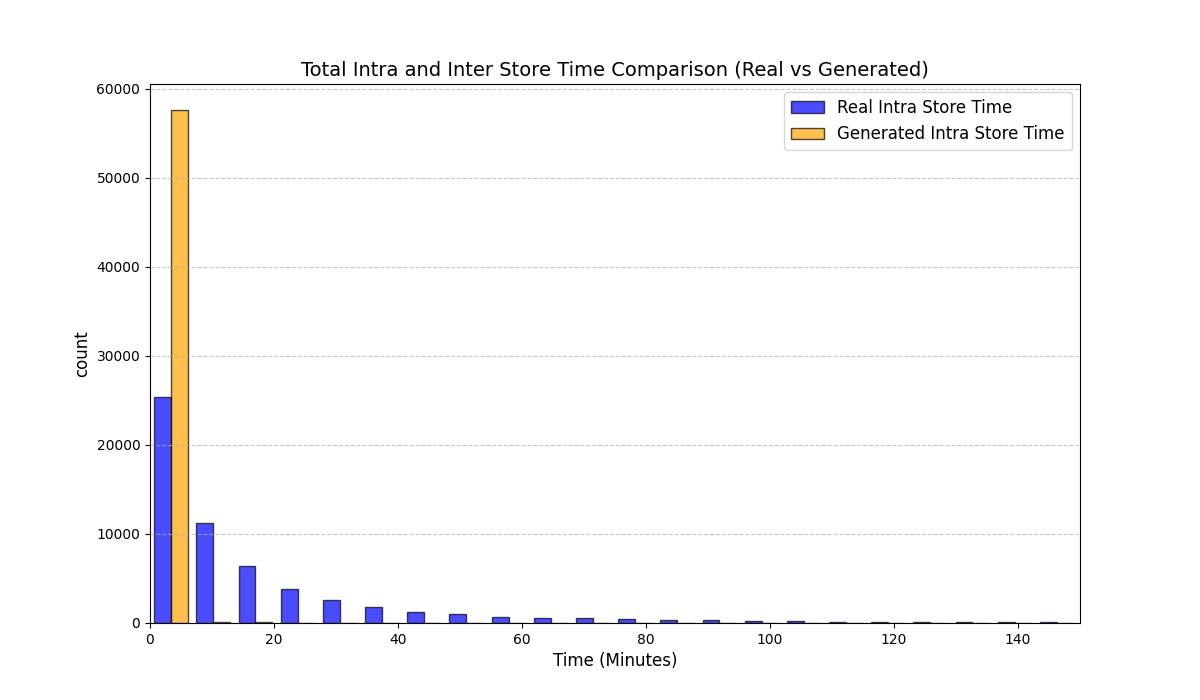}
\caption{Mall C}
\end{subfigure}\hfill
\begin{subfigure}[t]{0.24\linewidth}
\centering
\includegraphics[width=\linewidth]{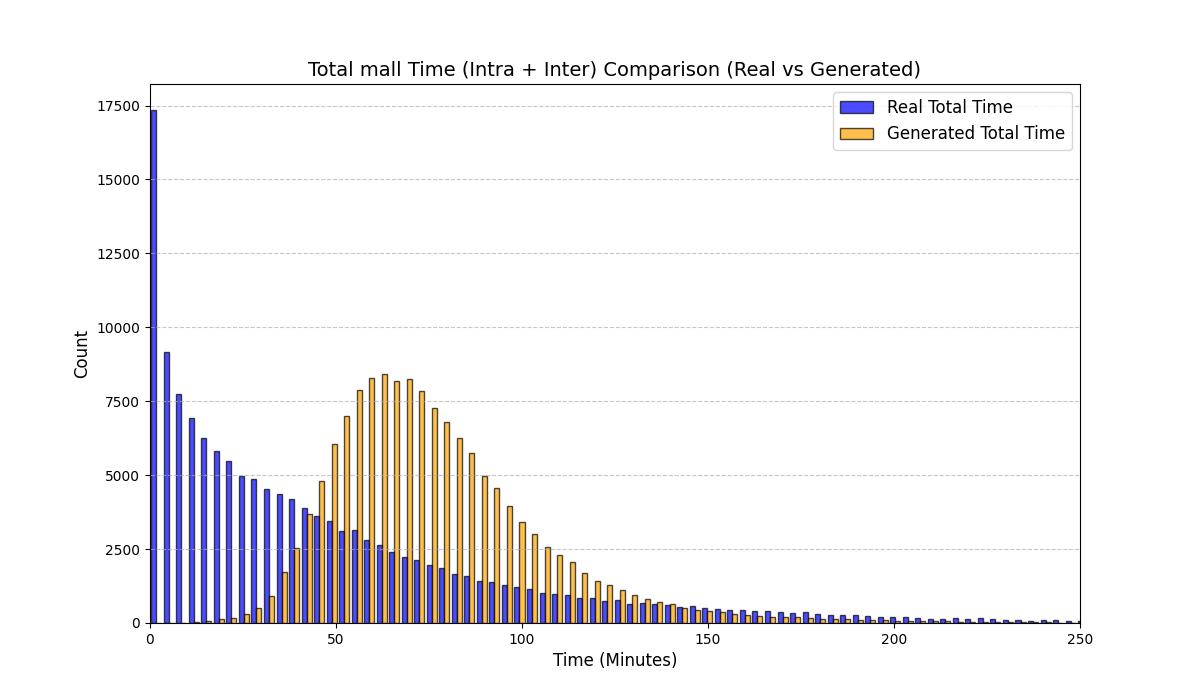}
\caption{Mall D}
\end{subfigure}

\medskip
\textbf{LAS}\par\medskip
\begin{subfigure}[t]{0.24\linewidth}
\centering
\includegraphics[width=\linewidth]{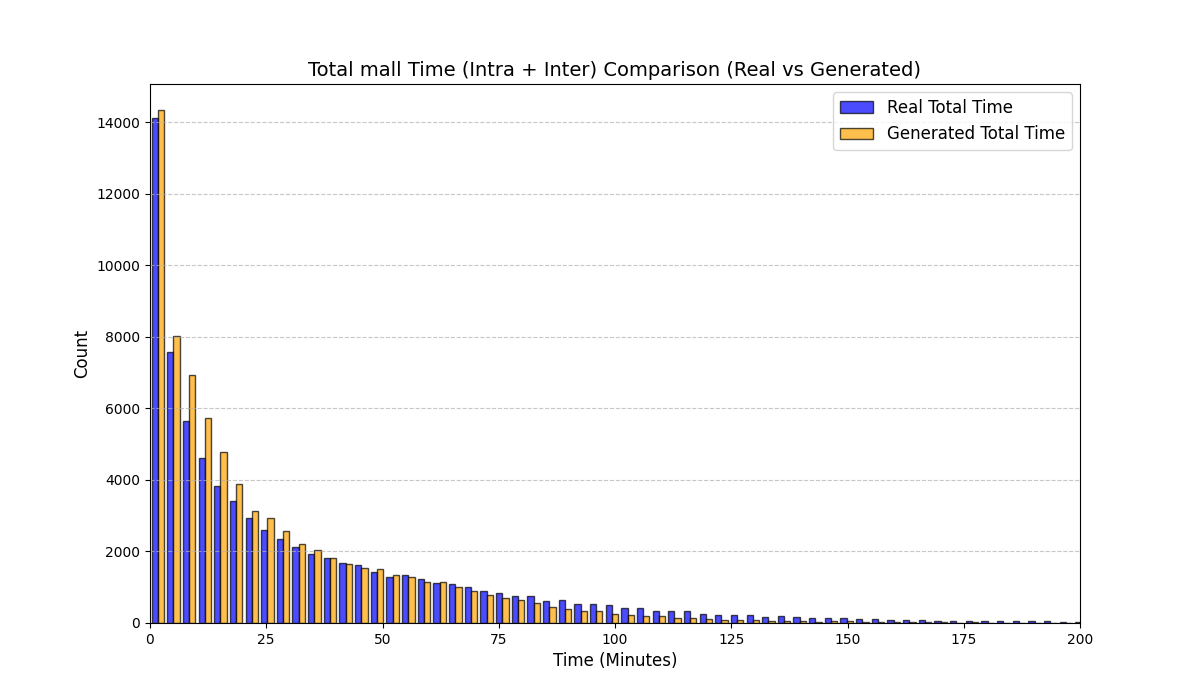}
\caption{Mall A}
\end{subfigure}\hfill
\begin{subfigure}[t]{0.24\linewidth}
\centering
\includegraphics[width=\linewidth]{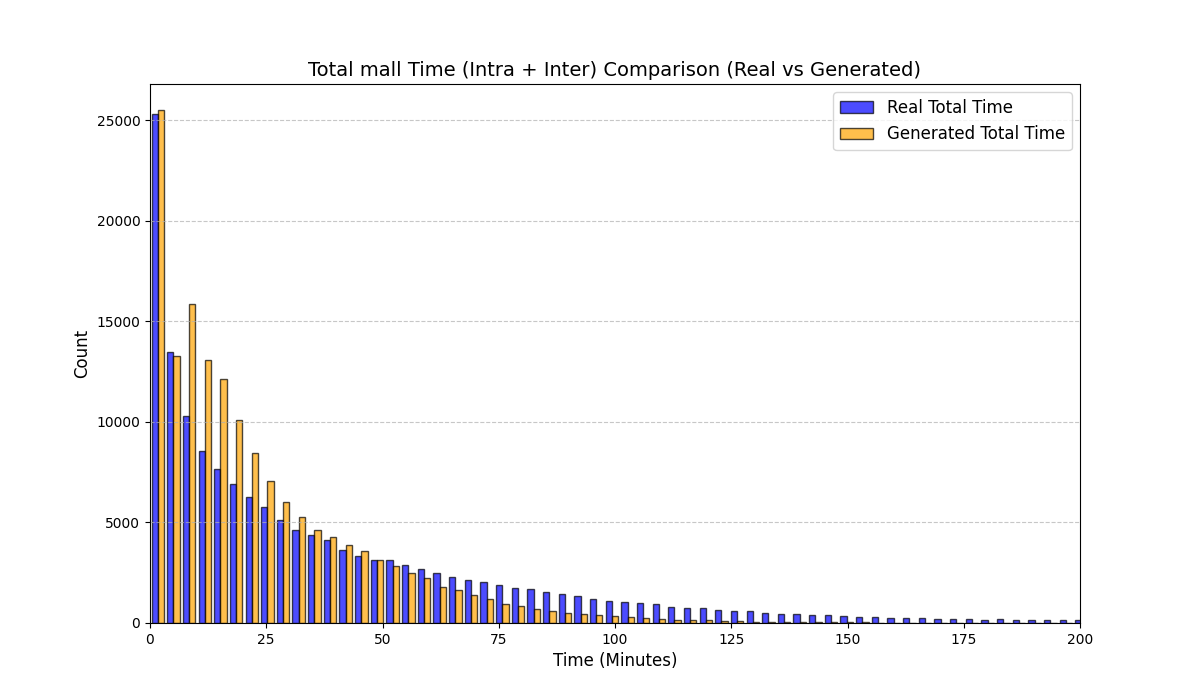}
\caption{Mall B}
\end{subfigure}\hfill
\begin{subfigure}[t]{0.24\linewidth}
\centering
\includegraphics[width=\linewidth]{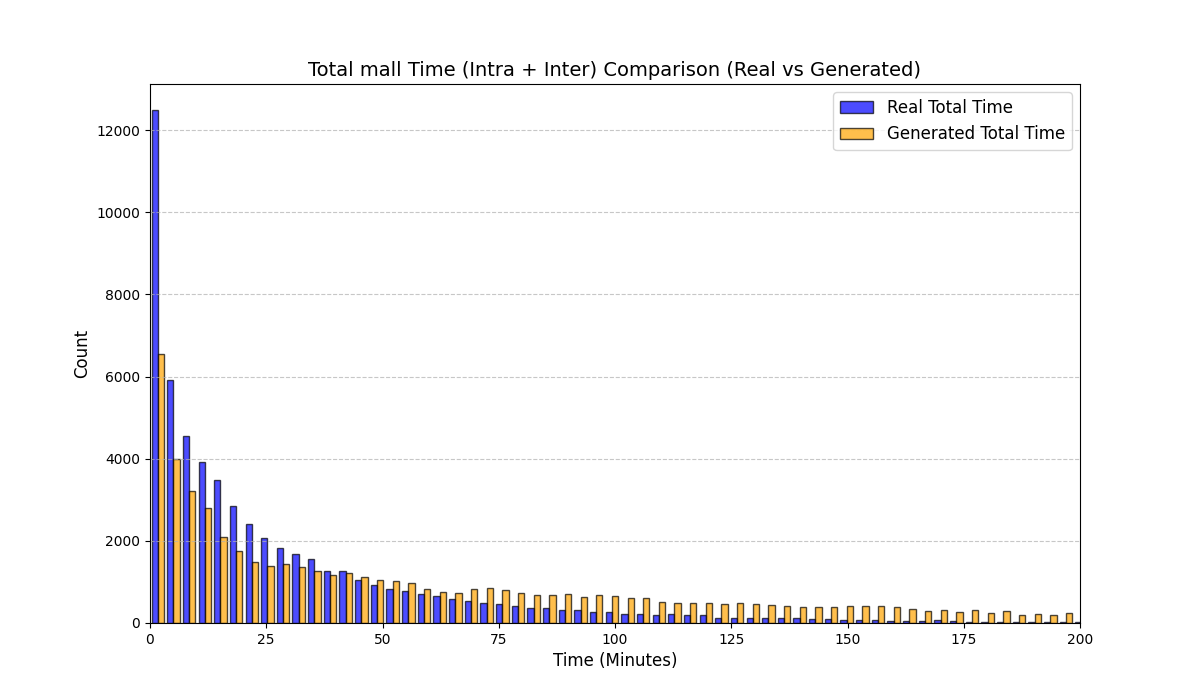}
\caption{Mall C}
\end{subfigure}\hfill
\begin{subfigure}[t]{0.24\linewidth}
\centering
\includegraphics[width=\linewidth]{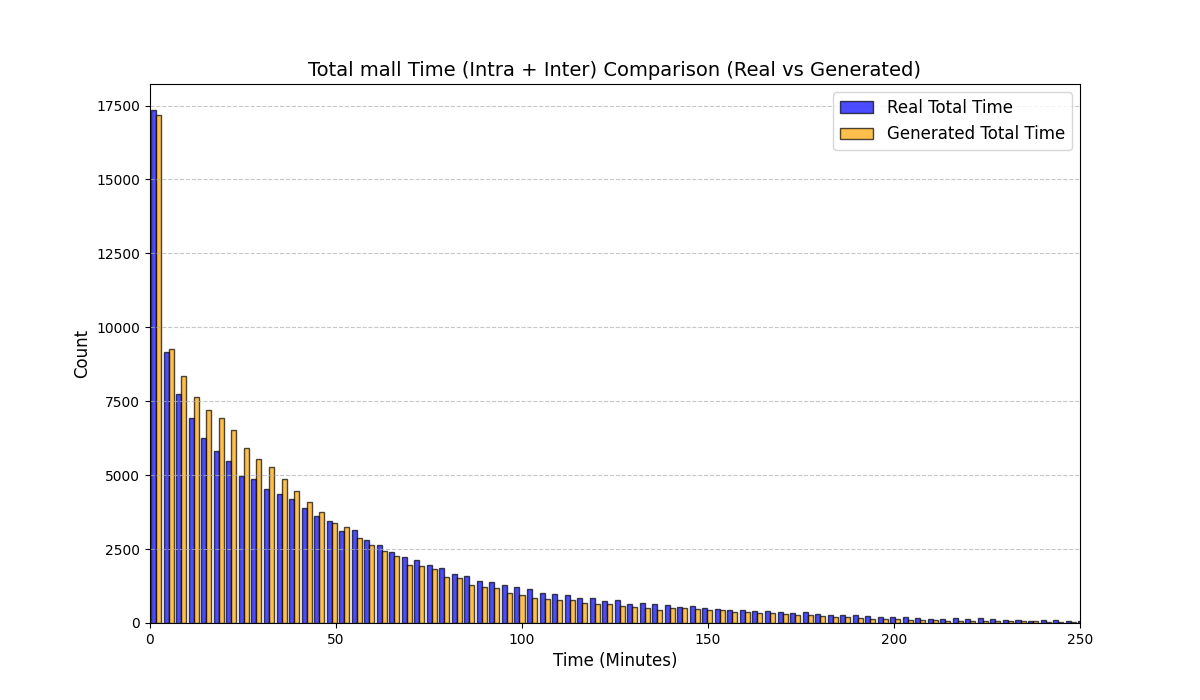}
\caption{Mall D}
\end{subfigure}

\caption{Total time spent in the mall: real vs.\ generated marginal distributions for all four malls. Top: RS. Bottom: LAS.}
\label{fig:mall_total_time_all}
\end{figure*}

\subsection{Public datasets: full results and extra plots}
\label{app:public_full_results}
Table~\ref{tab:public_full_results} reports a full set of derived-metric KS errors on Amazon, Movie, Education, and GPS.
Figures~\ref{fig:public_extra} and~\ref{fig:public_extra_gps_edu} provide additional visualizations that complement the main-text plots in Figs.~\ref{fig:amazon_main}--\ref{fig:gps_main}.

\begin{table*}[t]
\centering
\caption{Full public-dataset results: KS statistic for each derived metric (lower is better).}
\label{tab:public_full_results}
\setlength{\tabcolsep}{4pt}
\renewcommand{\arraystretch}{1.05}
\small
\begin{tabular}{llcc}
\toprule
Dataset & Derived metric & RS & LAS\\
\midrule
Amazon & Sequence length & 0.002 & 0.002 \\
Amazon & Item diversity & 0.338 & 0.020 \\
Amazon & Inter-event days & 0.456 & 0.170 \\
Amazon & Duration (days) & 0.413 & 0.046 \\
Amazon & Mean rating & 0.632 & 0.590 \\
\addlinespace
Movie & Trajectory length & 0.120 & 0.067 \\
Movie & Inter-rating time (min) & 0.466 & 0.294 \\
Movie & Mean rating & 0.155 & 0.106 \\
Movie & Rating std & 0.754 & 0.669 \\
\addlinespace
Education & Trajectory length & 0.411 & 0.164 \\
Education & Mean correctness & 0.9997 & 0.529 \\
Education & Std correctness & 0.9994 & 0.350 \\
\addlinespace
GPS & Trajectory length & 0.243 & 0.0287 \\
GPS & Total distance (km) & 0.284 & 0.142 \\
GPS & Average speed (km/h) & 0.312 & 0.108 \\
\midrule
Amazon & Mean across metrics & 0.368 & 0.166 \\
Movie & Mean across metrics & 0.373 & 0.284 \\
Education & Mean across metrics & 0.803 & 0.348 \\
GPS & Mean across metrics & 0.280 & 0.093 \\
\bottomrule
\end{tabular}
\end{table*}

\begin{figure*}[t]
\centering
\begin{subfigure}[t]{0.48\linewidth}
\centering
\includegraphics[width=\linewidth]{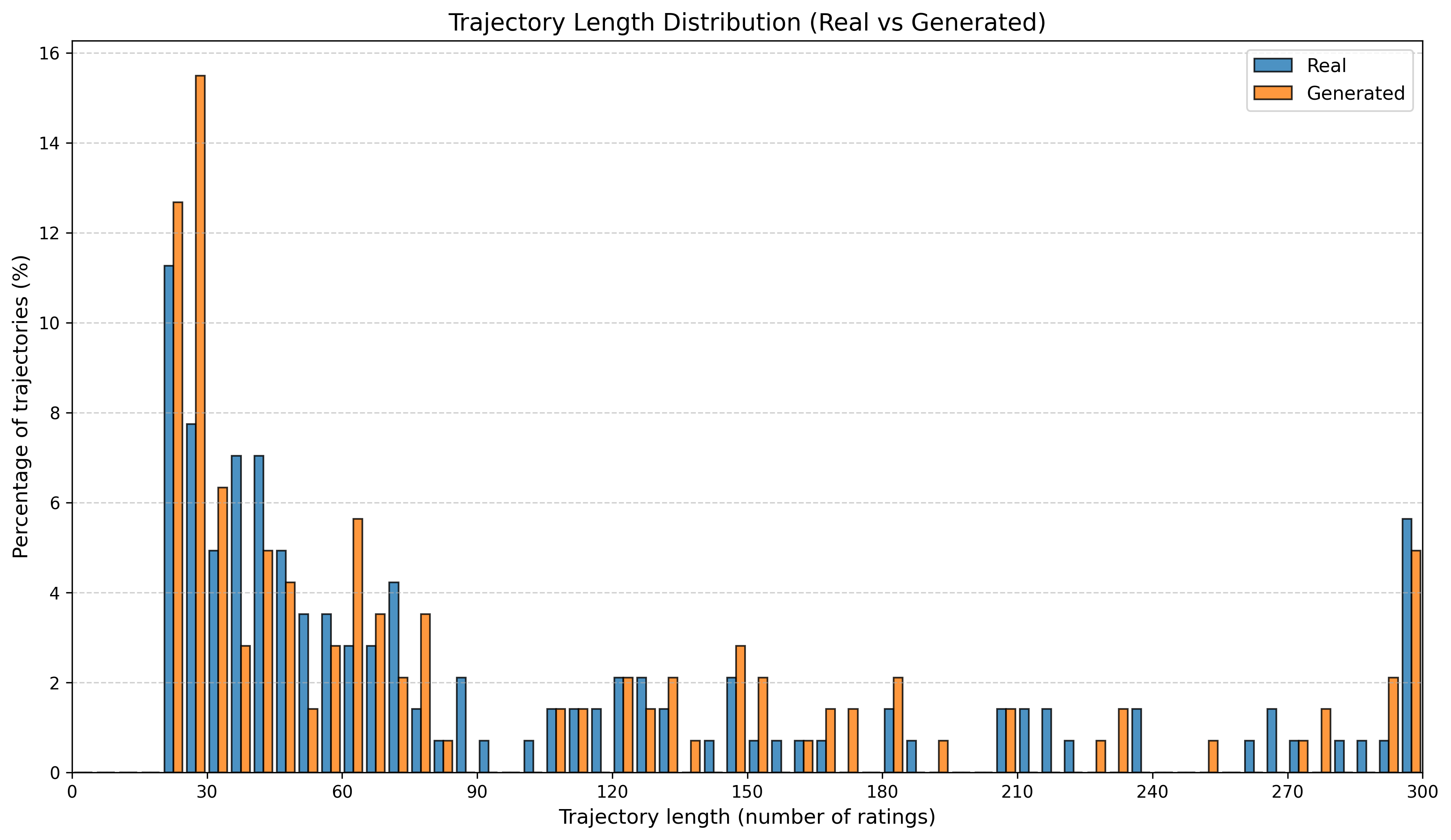}
\caption{RS: trajectory length (Movie)}
\end{subfigure}\hfill
\begin{subfigure}[t]{0.48\linewidth}
\centering
\includegraphics[width=\linewidth]{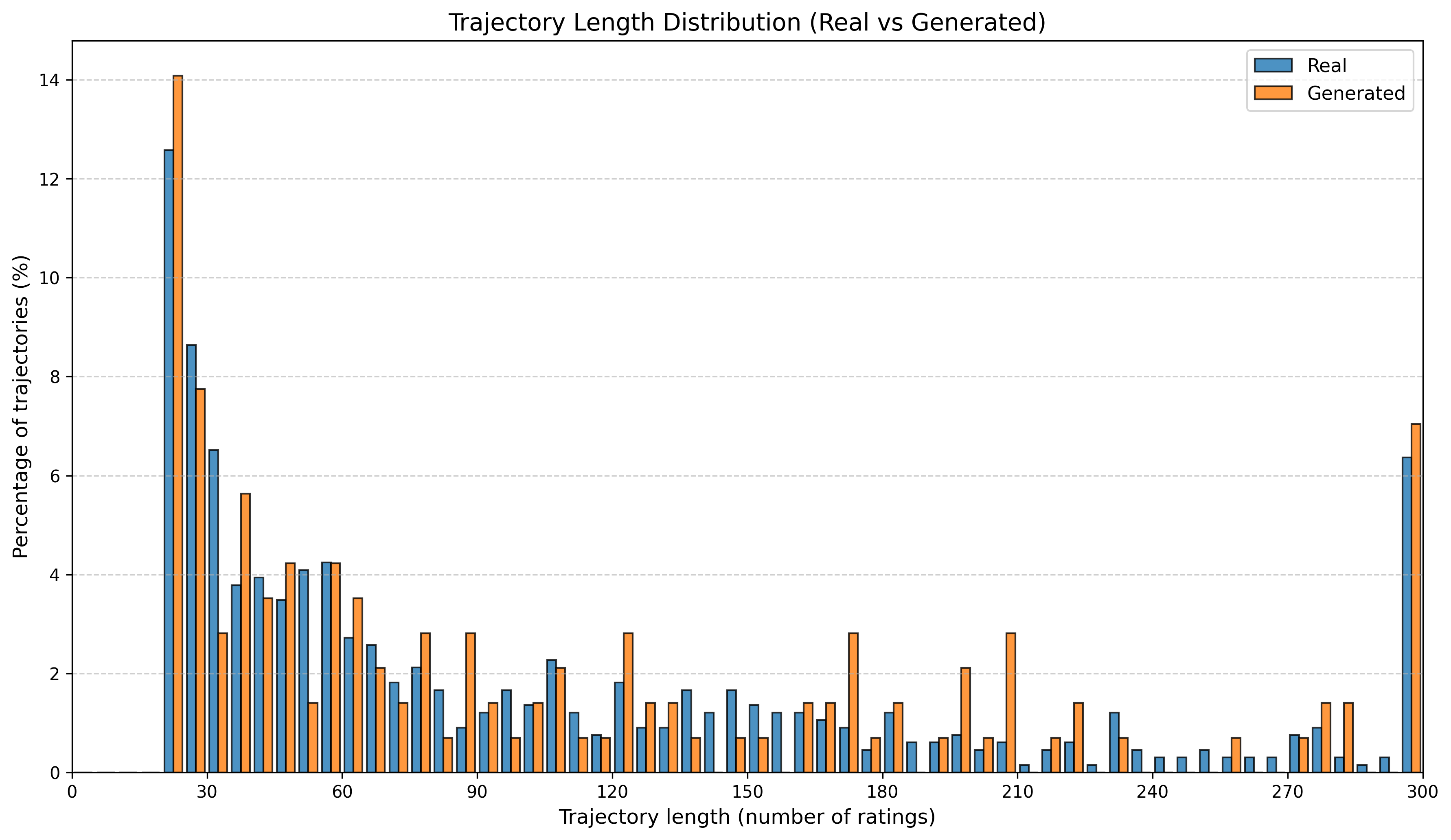}
\caption{LAS: trajectory length (Movie)}
\end{subfigure}
\caption{Movie: trajectory-length marginal distribution.}
\label{fig:public_extra}
\end{figure*}

\begin{figure*}[t]
\centering
\begin{subfigure}[b]{0.48\linewidth}
\centering
\includegraphics[width=\linewidth]{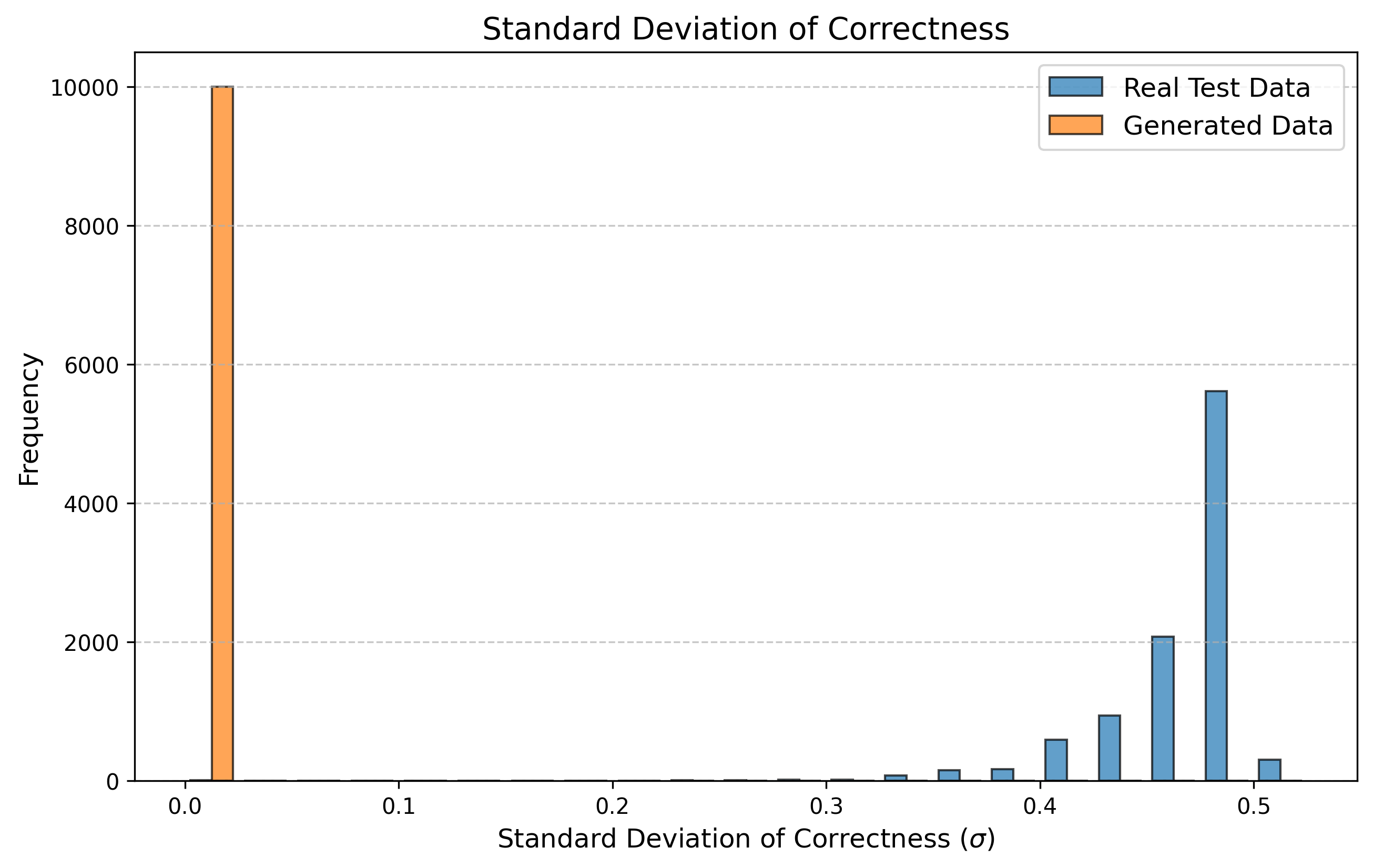}
\caption{Education (RS): std correctness}
\end{subfigure}
\hfill
\begin{subfigure}[b]{0.48\linewidth}
\centering
\includegraphics[width=\linewidth]{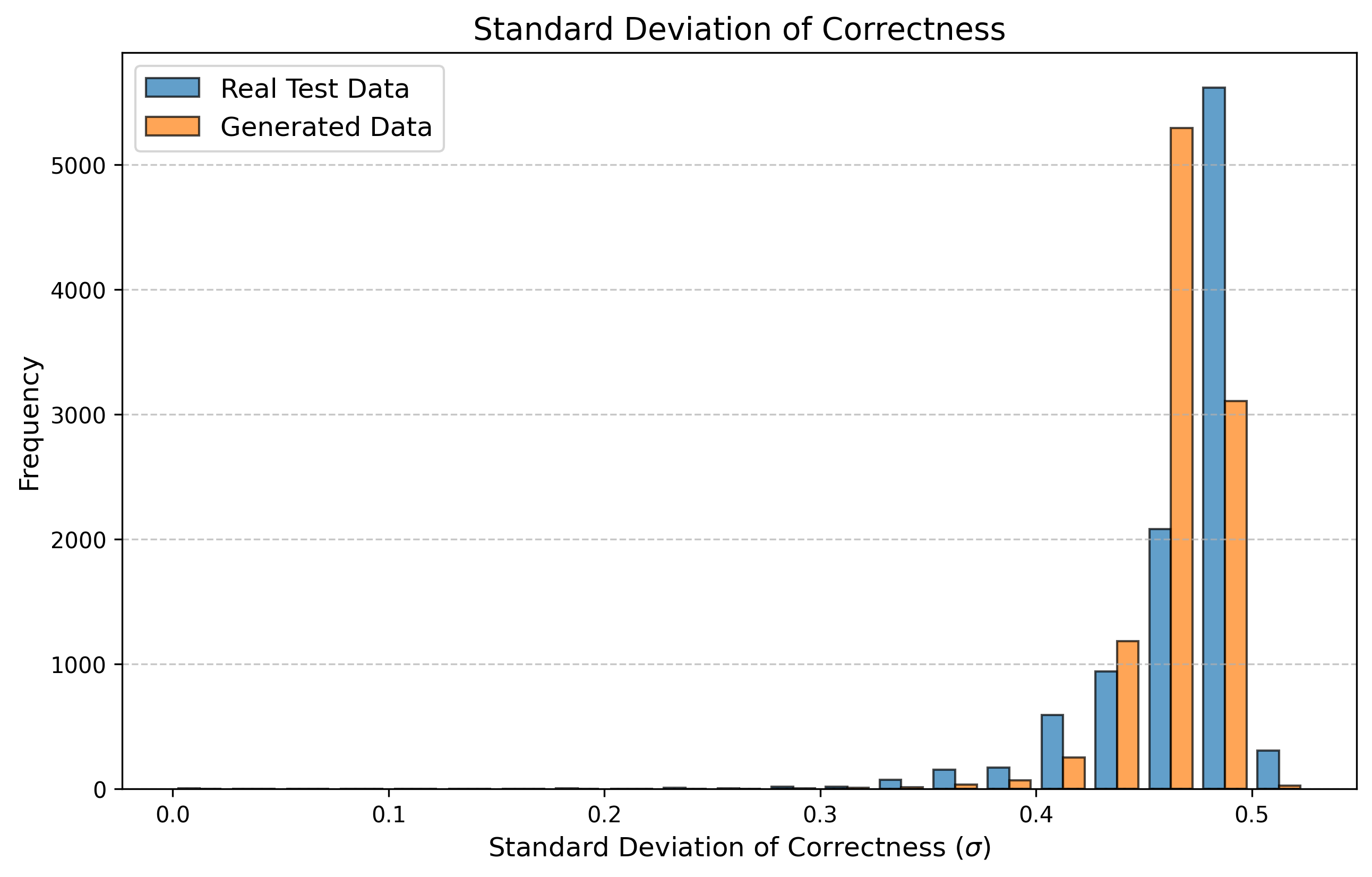}
\caption{Education (LAS): std correctness}
\end{subfigure}

\begin{subfigure}[b]{0.48\linewidth}
\centering
\includegraphics[width=\linewidth]{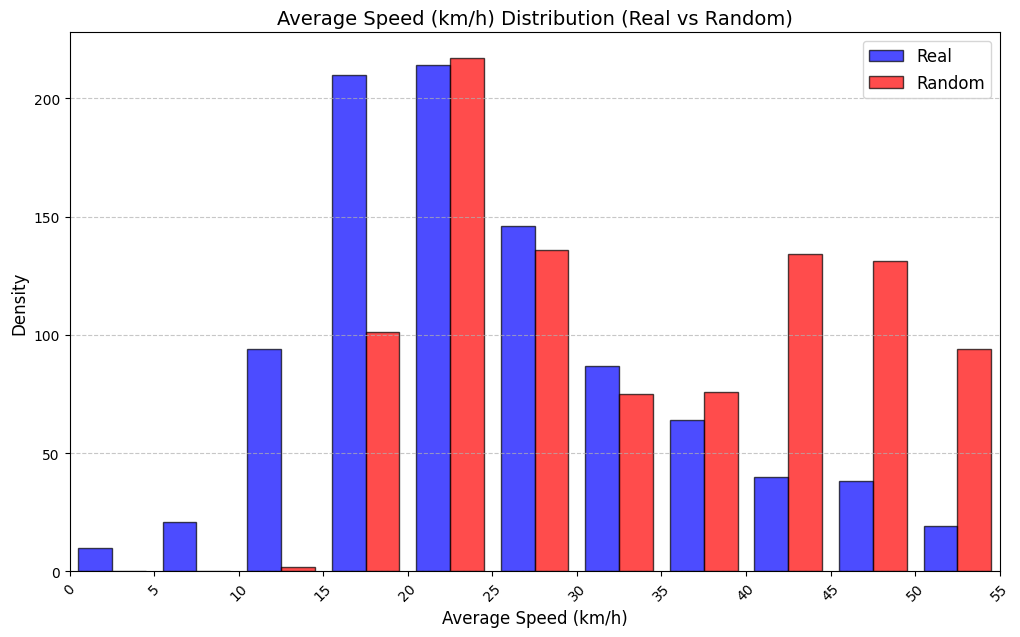}
\caption{GPS (RS): average speed}
\end{subfigure}
\hfill
\begin{subfigure}[b]{0.48\linewidth}
\centering
\includegraphics[width=\linewidth]{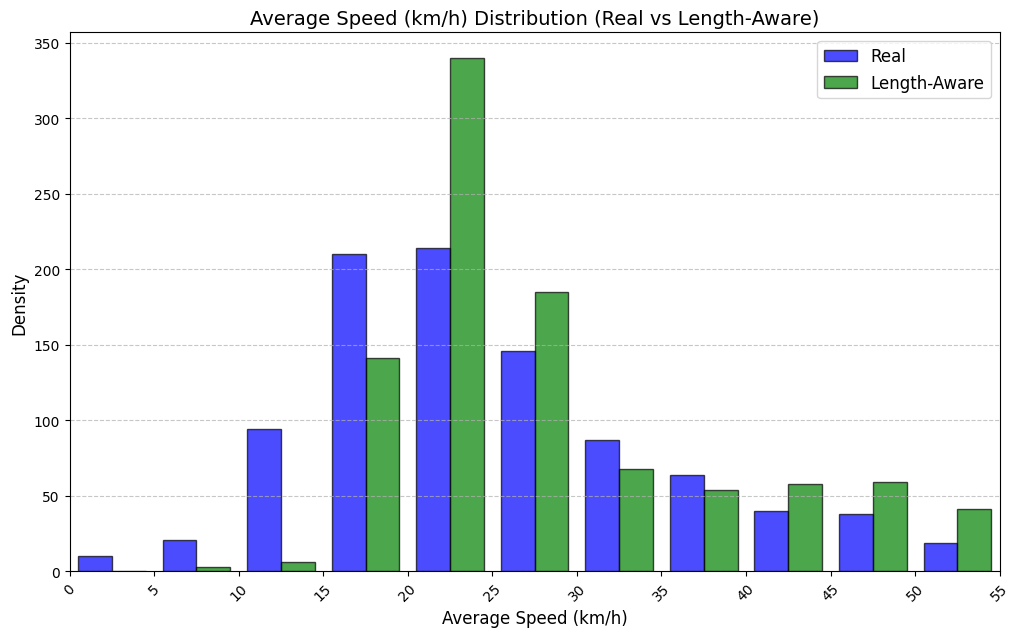}
\caption{GPS (LAS): average speed}
\end{subfigure}
\caption{Additional public-dataset marginals for Education and GPS.}
\label{fig:public_extra_gps_edu}
\end{figure*}

\end{document}